\documentclass{article}

\usepackage[preprint,nonatbib]{neurips_2026}

\usepackage[utf8]{inputenc} % allow utf-8 input
\usepackage[T1]{fontenc}    % use 8-bit T1 fonts
\usepackage{hyperref}       % hyperlinks
\usepackage{url}            % simple URL typesetting
\usepackage{booktabs}       % professional-quality tables
\usepackage{amsfonts}       % blackboard math symbols
\usepackage{nicefrac}       % compact symbols for 1/2, etc.
\usepackage{microtype}      % microtypography
\usepackage[table]{xcolor}         % colors
\usepackage{microtype}
\usepackage{graphicx}
\usepackage{subcaption}
\usepackage{booktabs} % for professional tables
\usepackage{url}
\usepackage{float}
\usepackage{graphicx}
\usepackage{multirow}
\usepackage{xspace}
\usepackage{natbib}
\usepackage{enumitem}
\usepackage{hyperref}
\usepackage{algorithm}
\usepackage{algorithmic}
\usepackage{makecell}
\usepackage{amsmath}
\allowdisplaybreaks
\usepackage{amssymb}
\usepackage{mathtools}
\usepackage{amsthm}
\usepackage{multirow}
\usepackage{subcaption}
\usepackage{todonotes}

\theoremstyle{plain}
\newtheorem{theorem}{Theorem}[section]

\newtheorem{lemma}[theorem]{Lemma}
\newtheorem{corollary}[theorem]{Corollary}
\theoremstyle{definition}

\newtheorem{assumption}[theorem]{Assumption}
\theoremstyle{remark}
\newtheorem{remark}[theorem]{Remark}

\definecolor{lightgray}{rgb}{0.9, 0.9, 0.9}
\definecolor{lightgreen}{rgb}{0.92, 0.98, 0.92} % 浅绿色，用于标题/分组行
\definecolor{darkgreen}{rgb}{0.85, 0.95, 0.85} % 稍深一点的绿色

\newcommand{\e}[1]{\mathbb{E}\left[#1\right]}
\newcommand{\nm}[1]{\left\|#1\right\|}
\newcommand{\nmsq}[1]{\left\|#1\right\|^2}
\newcommand{\ip}[1]{\left\langle#1\right\rangle}
\newcommand*\tcircle[1]{%
  \raisebox{0pt}{%
    \textcircled{\fontsize{7pt}{0}\fontfamily{phv}\selectfont #1}%
  }%
}

\title{DECA: Decentralizing Block-Wise Adam for Efficient LLM Full-Parameter Fine-Tuning on Non-IID Data}

\author{
  \textbf{Yunsheng Yuan}\textsuperscript{1}, 
  \textbf{Shaowei Li}\textsuperscript{1}, 
  \textbf{Kai Wang}\textsuperscript{2}, 
  \textbf{Zhongyuan Sun}\textsuperscript{1} \\
  \textbf{Zheng Zhang}\textsuperscript{3}, 
  \textbf{Kai Han}\textsuperscript{4}, 
  \textbf{Jun Luo}\textsuperscript{5}, 
  \textbf{Feng Li}\textsuperscript{1} \\
  \vspace{0.2cm}
  \normalfont % Resets the font to normal for the affiliations
  \textsuperscript{1}School of Computer Science and Technology, Shandong University, Qingdao China \\
  \textsuperscript{2}School of Mathematical Science, Peking University, China \\
  \textsuperscript{3}IEIT SYSTEM, China \\
  \textsuperscript{4}School of Computer Science and Artificial Intelligence, \\Shanghai University of Finance and Economics, Shanghai, China \\
  \textsuperscript{5}College of Computing and Data Science, Nanyang Technological University, Singapore \\
  \vspace{0.1cm}
  Email: \texttt{\{ysyuan1028, 202420868\}@mail.sdu.edu.cn, wangkaisd@stu.pku.edu.cn,} \\
  \texttt{202415172@mail.sdu.edu.cn, zhangzheng05@ieisystem.com,}\\
  \texttt{hankai@mail.shufe.edu.cn}
  \texttt{junluo@ntu.edu.sg, fli@sdu.edu.cn}
}

\begin{document}

\maketitle

\begin{abstract}
  Fine-tuning \emph{large language models} (LLMs) in privacy-sensitive and resource-constrained environments remains challenging. Since training data are often distributed across multiple clients, decentralized fine-tuning offers a natural paradigm for collaborative adaptation without a central server. However, enabling \emph{full-parameter fine-tuning} (FPFT) in this decentralized setting is difficult: FPFT provides strong adaptation capacity but incurs prohibitive resource consumption for billion-scale models. Existing decentralized LLM fine-tuning methods therefore mainly rely on parameter-efficient updates, which improve efficiency but may restrict downstream performance. Moreover, client data are typically non-IID, making decentralized optimization more vulnerable to client drift and unstable convergence. To address these challenges, we propose DECA, a resource-efficient decentralized FPFT framework for LLMs on non-IID data. DECA partitions model parameters into disjoint blocks and performs sequential block-wise Adam optimization, reducing resource consumption while preserving decentralized full-parameter adaptation. To stabilize training, DECA further introduces first- and second-order block-wise moment estimates with fresh local gradient statistics and consensus-derived discrepancy signals. We provide rigorous theoretical analysis and extensive experiments, showing that DECA achieves fast convergence, strong downstream performance, and significant resource efficiency.
\end{abstract}

\section{Introduction} \label{sec:intro}
  \emph{Large language models} (LLMs), such as LLaMA~\cite{DubeyJPAALMSF-ARXIV24}, Qwen~\cite{QwenAY-arXiv25}, and GPT~\cite{AchiamAAA-arXiv23}, have demonstrated remarkable capabilities in language understanding, reasoning, and generation. These models, typically built on the Transformer architecture~\cite{VaswaniSPNUJGKP-NIPS17}, serve as versatile foundations for a wide array of tasks. To effectively deploy these foundation models for domain-specific applications, fine-tuning remains an essential step for aligning model behavior with specialized downstream tasks.

  Most state-of-the-art studies related to LLM fine-tuning, e.g., \emph{parameter-efficient fine-tuning} (PEFT) methods~\cite{HoulsbyGJMDGAG-ICML19,LesterAC-EMNLP21,HuSSAWLWC-ICLR22} and \emph{full-parameter fine-tuning} (FPFT) methods~\cite{KingmaB-ICLR2015,LuX-MP15,MalladiGNDLCA-NIPS23,ZhaoZCWAT-ICML24,LuoYLi-NIPS24}, often assume that LLMs are tuned in a centralized manner, i.e., data are stored in a central infrastructure with sufficient resources. However, in real-world scenarios, data are distributed across different devices often in a non-IID fashion, while collecting the large volume of data from the devices to a single infrastructure may result in significant communication overhead and risk of privacy leakage. To this end, \emph{federated learning} (FL)~\cite{McMahanMRHA-AISTATS17} is applied, such that each client device calculates a local model update according to its private data, and a global model is trained by aggregating these local updates at a central server. Although FL has been integrated with both PEFT~\cite{GuoZWFQ-ICLR25,ZhaoDLLL-ICASSP23} and FPFT methods~\cite{QinCQDLD-ICML24,ShuHNLY-ICML25}, they generally trade performance for reducing the intensive demand on communication resources.

  %Unfortunately, in FL architecture, all of the clients are coordinated by the central server; hence, the usability of the FL framework may suffer severe limitations, e.g., when the central server suffers cyber-attacks or it has to persistently communicate with the clients and thus becomes a significant bottleneck. 

  Alternatively, the \emph{decentralized} paradigm offers a purely distributed solution that further eliminates the need for the central server, as each client only interacts with its direct neighbors. Though many state-of-the-art studies have been performed to investigate the application of the decentralized learning paradigm in training small-scale \emph{deep neural network} (DNN) models~\cite{LinKSJ-ICML21,EsfandiariTJBHHS-ICML21,VosFGKPS-NIPS23,GhiasvandRAP-TON25,WangYWL-ICDCS25,WangYLD-TON26}, only very few proposals exist to decentralize LLM fine-tuning, e.g., decentralized LoRA~\cite{GhiasvandAP-REALM25,Saadati-NeuralNet26}. Unfortunately, since PEFT techniques tune a subset of model parameters to reduce resource consumption, they potentially limit performance on downstream tasks~\cite{ZhangLCF-ICLR24,LuoYLi-NIPS24,PuJYK-arXiv23}.

  In this paper, we propose \textbf{DECA}, a lightweight decentralized framework for FPFT of LLMs on non-IID data. DECA partitions model parameters into disjoint blocks and decentralizes block-wise Adam optimization, enabling clients to collaboratively fine-tune the blocks in a sequential manner with significantly reduced resource consumption. 
  Furthermore, DECA introduces first- and second-order \emph{block-wise moment approximations} (BMAs), which extend Adam-style moment estimation to decentralized block-wise optimization by incorporating fresh local gradient statistics and consensus-derived discrepancy signals. This allows each client to form a more globally consistent update direction within each active block, thereby mitigating client drift and stabilizing optimization under non-IID data.
  Our theoretical analysis shows that DECA converges under non-IID data while matching the asymptotic convergence behavior of state-of-the-art decentralized optimization methods. Our experiments on classification and generation tasks with models ranging from 1.5B to 8B parameters demonstrate that DECA consistently outperforms decentralized baselines and enables practical FPFT on resource-constrained devices. Our main contributions are summarized as follows:
  \begin{itemize}
    \item We propose DECA, a decentralized FPFT framework enabling block-wise, BMA-enhanced Adam optimization for collaborative fine-tuning of LLMs on non-IID data.
    \item We provide a rigorous theoretical analysis establishing state-of-the-art convergence rates under non-IID data, together with substantial improvements in memory, computation, and communication efficiency.
    \item We conduct extensive experiments showing that DECA consistently outperforms existing baselines across diverse downstream tasks and model scales. 
  \end{itemize}

  The remainder of this paper is organized as follows. Sec.~\ref{sec:pre} introduces the system model. Sec.~\ref{sec:alg} presents the DECA algorithm, followed by the theoretical analysis in Sec.~\ref{sec:theory}. Sec.~\ref{sec:exp} reports our experiment results, and Sec.~\ref{sec:relwork} reviews related work. Finally, Sec.~\ref{sec:conclusion} concludes the paper.

\section{System Model and Preliminaries} \label{sec:pre}
  Consider a decentralized learning system represented by an undirected graph with weights $\mathcal{G}=\{\mathcal{N}, \mathbf{W}\}$, where $\mathcal{N}=\{1,2,\cdots,N\}$ denotes a set of $N$ clients, and $\mathbf{W}\in[0,1]^{N \times N}$ denotes the adjacency matrix of graph $G$. Let $\mathcal{D}_i$ represent the local data distribution of client $i \in \mathcal{N}$, and $D_i=|\mathcal{D}_i|$ denote the number of data samples in $\mathcal{D}_i$. It is worth noting that the data are non-IID across different clients. The communication topology is characterized by $\mathbf{W}$. We assume $\mathbf{W}$ is a symmetric doubly stochastic matrix~\cite{LianZZHZL-NIPS17, EsfandiariTJBHHS-ICML21, YuJY-ICML19, ZhangCHWY-ICML22}, such that $\sum_{i = 1}^N w_{i,j} = \sum_{j = 1}^N w_{i,j} = 1$ and $w_{i,j} = w_{j,i}$ for $\forall i,j \in \mathcal{N}$, where $w_{i,j}\in[0,1]$ denotes the element in the $i$-th row and $j$-th column of the matrix $\mathbf{W}$. We use $w_{i,j}$ to encode how much client $i$ can affect client $j$, while $w_{i,j}=0$ implies clients $i$ and $j$ are disconnected. Let $\mathcal{N}_i = \{ j \in \mathcal{N} \mid w_{i,j} > 0 \}$ denote the set of neighbors for client $i$, with cardinality $N_i = |\mathcal{N}_i|$. 
  %
  %We assume self-loops exist (i.e., $w_{i,i} > 0$), ensuring that $i \in \mathcal{N}_i$.
  %
  The objective can be formulated as follows:
  \begin{equation} \label{eq:decentobj}
    \min_{{x} \in \mathbb{R}^d} \mathcal{F}({x}) = \frac{1}{N} \sum_{i=1}^{N} \underbrace{\mathbb{E}_{\zeta_i \sim \mathcal{D}_i} \left[F_i({x}; \zeta_i)\right]}_{\triangleq f_i({x})} 
  \end{equation}
  where ${x} \in \mathbb{R}^d$ represents the parameters of the target global model, $f_i({x}) \triangleq \mathbb{E}_{\zeta_i \sim \mathcal{D}_i} [F_i({x}; \zeta_i)]$ denotes the local loss function of client $i$, and $\zeta_i \sim \mathcal{D}_i$ represents the randomly sampled data from $\mathcal{D}_i$. Without loss of generality, we assume $F_i (\cdot; \cdot) = F (\cdot; \cdot)$, $\forall i\in\mathcal{N}$.
  Unfortunately, enabling FPFT in a decentralized setting is rather challenging.
  \begin{itemize}
    \item \textbf{Communication Overhead:} In traditional decentralized optimization, each client $i$ has to exchange its full local model $x_i \in \mathbb{R}^d$ with its neighbors. Since LLMs typically contain billions of parameters, exchanging full model updates incurs substantial communication overhead, i.e., $\text{bit-width}\times d\times (N_i-1)$, which exceeds the bandwidth capacity of clients.
    \item \textbf{Memory Constraints:} FPFT mandates the storage of full-parameter local updates. Furthermore, the aggregation step necessitates accommodating extremely high-dimensional parameters received from neighbors, a requirement that frequently exceeds the physical GPU memory capacity of individual clients.
    \item \textbf{Statistical Divergence:} The non-IID nature of $\mathcal{D}_i$ creates a mismatch between each local objective $f_i(x)$ and the global objective $\mathcal{F}(x)$. As a result, local updates may bias client models, causing client drift. Without centralized aggregation, such drift can outpace the consensus process induced by the communication topology $\mathbf{W}$, increasing the risk of unstable training or divergence.
    
  \end{itemize}

\section{Our Algorithm} \label{sec:alg}
%
  %We first provide an overview of DECA in Sec.\ref{ssec:overview}, and then present its details in Sec.\ref{ssec:algorithm}. We further introduce a communication-efficient extension of DECA in Sec.~\ref{ssec:extension}.  

  \subsection{Overview} \label{ssec:overview}
    We propose \textbf{DECA}, a novel decentralized optimization framework that integrates \emph{block-wise Adam optimization} with first- and second-order \emph{block-wise moment approximations} (BMAs). Focusing on the resource-intensive Transformer core (while treating embedding and LM head layers as static), we partition the trainable parameter space into $B$ disjoint blocks, denoted as $x = \{x_1, x_2, \dots, x_B\}$. Let $d_k=|x_k|$ denote the number of parameters in the $k$-th block. The algorithm operates over $T$ global rounds. In each round $t$, we sequentially optimize the blocks from $k=1$ to $B$. Specifically, for the $k$-th block, clients perform $R$ steps of collaborative decentralized optimization, while holding the other blocks fixed. Formally, the sub-problem for optimizing block $k$ is defined as:
    \begin{equation} \label{eq:block_obj}
      \min_{x_{k}\in\mathbb{R}^{d_k}}~ \frac{1}{N} \sum^N_{i=1} \mathbb{E}_{\zeta \sim \mathcal{D}_i} \left[ F_i(x_1, \dots, x_k, \dots, x_B; \zeta) \right]
    \end{equation}
    Standard \emph{block coordinate descent} (BCD)-based methods typically use simple gradient descent or basic Adam updates~\cite{BengioLPL-NIPS06,Wright-MP15,BelilovskyEO-ICML19,LuoYLi-NIPS24}. However, directly applying block-wise Adam to decentralized non-IID settings is challenging, as local adaptive updates may amplify client drift, while intermittently updated blocks suffer from stale moment estimates. To address this issue, BMA constructs hybrid first- and second-order estimates by combining fresh local gradient statistics with a consensus-derived discrepancy signal, thereby correcting local updates toward network-wide agreement and stabilizing decentralized block-wise optimization.

    %
    %, which struggles with the complex loss landscapes of LLMs; conversely, decentralized Adam is notoriously susceptible to model divergence when data is non-IID. We bridge this gap by enhancing the block-wise Adam optimizer with BMAs. 
    %
    %

  \subsection{Algorithm Design} \label{ssec:algorithm}
    \begin{algorithm}[htb!]
    \caption{DECA (for each client $i\in\mathcal{N}$)}
    \label{alg:deca}
    \begin{algorithmic}[1]
    % Use \item[] to create an unnumbered line
    \item[] {\bfseries Input:} Graph $\mathcal{G}=\{\mathcal{N}, \mathbf{W}\}$, number of communication rounds $T$, number of local updates $R$, granularity of model partitioning $B$, learning rate $\gamma$, and hyperparameters $\alpha_1$, $\alpha_2$, $\beta_1$ and $\beta_2$.
    \item[] {\bfseries Output:} Final global model $x^{[T]} = \frac{1}{N}\sum^N_{i=1}x^{T}_i$
    %
    % Numbering starts here (Line 1)
    %
    \STATE Randomly initialize local model $x^{[0]}_i=\{x^{[0]}_{i,k}\}^B_{k=1}$
    \FOR{$t=0$ {\bfseries to} $T-1$}
      \FOR{$k=1$ {\bfseries to} $B$}
        \STATE $m^{[t,0]}_{i,k} = 0$, $v^{[t,0]}_{i,k} = 0$ \label{ln:init_mom}
        \FOR{$r=0$ {\bfseries to} $R-1$}
          \STATE $g^{[t,r]}_{i,k} = \nabla_k F_i (x^{[t,r]}_{i,k}, \zeta)$ where $\zeta \sim \mathcal{D}_i$;
          \STATE $m^{[t,r+\frac{1}{2}]}_{i,k} = \alpha_1 m^{[t,r]}_{i,k} + (1-\alpha_1) g^{[t,r]}_{i,k}$ and $v^{[t,r+\frac{1}{2}]}_{i,k} = \alpha_2 v^{[t,r]}_{i,k} + (1-\alpha_2) g^{[t,r]}_{i,k} \odot g^{[t,r]}_{i,k}$;
          \STATE $\widehat{m}^{[t,r]}_{i,k} = m^{[t,r+\frac{1}{2}]}_{i,k} / (1-\alpha^{r+1}_1)$ and $\widehat{v}^{[t,r]}_{i,k} = v^{[t,r+\frac{1}{2}]}_{i,k} / (1-\alpha^{r+1}_2)$;
          \STATE $x^{[t,r+\frac{1}{2}]}_{i,k} = x^{[t,r]}_{i,k} - \gamma\frac{\widehat{m}^{[t,r]}_{i,k}}{\sqrt{\widehat{v}^{[t,r]}_{i,k}} + \epsilon}$;
          \STATE Exchange $x^{[t,r+\frac{1}{2}]}_{i,k}$ with its neighbors;
          \STATE $x^{[t,r+1]}_{i,k} = \sum_{j\in\mathcal{N}_i} w_{i,j} x^{[t,r+\frac{1}{2}]}_{j,k}$;
          \STATE $h^{[t,r]}_{i,k} = \frac{x^{[t,r]}_{i,k} - x^{[t,r+1]}_{i,k}}{\left\| x^{[t,r]}_{i,k} - x^{[t,r+1]}_{i,k} \right \|}$;
          \STATE $m^{[t,r+1]}_{i,k} = \beta_1 m^{[t,r+\frac{1}{2}]}_{i,k} + (1-\beta_1)h^{[t,r]}_{i,k}$ and $v^{[t,r+1]}_{i,k} = \beta_2 v^{[t,r+\frac{1}{2}]}_{i,k} + (1-\beta_2)h^{[t,r]}_{i,k}\odot h^{[t,r]}_{i,k}$;
        \ENDFOR
        \STATE $x^{[t+1,0]}_{i,k} = x^{[t,R]}_{i,k}$;
      \ENDFOR
    \ENDFOR
    %\UNTIL{$noChange$ is $true$}
    \end{algorithmic}
    \end{algorithm}

    The pseudo-code of our algorithm is given in \textbf{Algorithm}~\ref{alg:deca}. The optimization of each block $k$ during round $t$ follows an iterative process managed through an inner loop of $R$ steps. At the beginning of the $k$-th block update in round $t$, each client $i$ initializes its BMAs: $m^{[t,0]}_{i,k}=0$ and $v^{[t,0]}_{i,k}=0$. At each step $r \in \{0, 1, \dots, R-1\}$, client $i$ calculates its local gradient based on its local data $\mathcal{D}_i$:
    \begin{equation} \label{eq:grad}
      g^{[t,r]}_{i,k} = \nabla_k F_i (x^{[t,r]}_{i,k}, \zeta), \quad \text{where } \zeta \sim \mathcal{D}_i.
    \end{equation}
    Subsequently, client $i$ updates its first and second BMAs using the hyperparameters $\alpha_1$ and $\alpha_2$:    
    \begin{equation} \label{eq:update_momentum}
    %\begin{cases}
      m^{[t,r+\frac{1}{2}]}_{i,k} = \alpha_1 m^{[t,r]}_{i,k} + (1-\alpha_1) g^{[t,r]}_{i,k}, ~\text{and}~
      v^{[t,r+\frac{1}{2}]}_{i,k} = \alpha_2 v^{[t,r]}_{i,k} + (1-\alpha_2) g^{[t,r]}_{i,k} \odot g^{[t,r]}_{i,k} 
    %\end{cases}
    \end{equation}
    Following this, we perform bias correction to obtain the estimates $\widehat{m}^{[t,r]}_{i,k}$ and $\widehat{v}^{[t,r]}_{i,k}$:
    \begin{equation} \label{eq:correct_momentum}
      \widehat{m}^{[t,r]}_{i,k} = {m^{[t,r+\frac{1}{2}]}_{i,k}}/{(1-\alpha^{r+1}_1)}, \quad \widehat{v}^{[t,r]}_{i,k} = {v^{[t,r+\frac{1}{2}]}_{i,k}}/{(1-\alpha^{r+1}_2)}.
    \end{equation}
    The local model parameters are then updated as follows:
    \begin{equation} \label{eq:update_local_model}
      x^{[t,r+\frac{1}{2}]}_{i,k} = x^{[t,r]}_{i,k} - \gamma \cdot {\widehat{m}^{[t,r]}_{i,k}} \Big/ {\left( \sqrt{\widehat{v}^{[t,r]}_{i,k}} + \epsilon \right)}.
    \end{equation}
    where $\epsilon$ is a small constant for numerical stability when the second moment estimate $\widehat{v}^{[t,r]}_{i,k}$ vanishes.

    Client $i$ then shares its parameters with its neighbors and performs weighted aggregation:
    \begin{equation} \label{eq:agg_model}
      x^{[t,r+1]}_{i,k} = \sum_{j\in\mathcal{N}_i} w_{i,j} x^{[t,r+\frac{1}{2}]}_{j,k}.
    \end{equation}
    To mitigate the challenges of data heterogeneity, we introduce a block-wise consensus-derived discrepancy signal to correct local optimization directions and reduce client drift. Specifically, we calculate the normalized difference between the local model and the aggregated model:
    \begin{equation} \label{eq:qgm_difference}
      h^{t,r}_{i,k} = \left( {x^{[t,r]}_{i,k} - x^{[t,r+1]}_{i,k}} \right) \Big/ {\left\| x^{[t,r]}_{i,k} - x^{[t,r+1]}_{i,k} \right \|},
    \end{equation}
    which is used to update both the first- and second-order BMAs for the subsequent step using the hyperparameters $\beta_1$ and $\beta_2$:
    \begin{align} \label{eq:qgm_update_m}
      m^{[t,r+1]}_{i,k} = \beta_1 m^{[t,r+\frac{1}{2}]}_{i,k} + (1-\beta_1)h^{t,r}_{i,k} = \alpha_1 \beta_1 m^{[t,r]}_{i,k} + \beta_1(1-\alpha_1) g^{[t,r]}_{i,k} + (1-\beta_1) h^{[t,r]}_{i,k}
    \end{align}
    and 
    \begin{align} \label{eq:qgm_update_v}
      v^{[t,r+1]}_{i,k} =& \beta_2 v^{[t,r+\frac{1}{2}]}_{i,k} + (1-\beta_2) h^{t,r}_{i,k} \odot h^{t,r}_{i,k} \nonumber\\
      =& \alpha_2 \beta_2 v^{[t,r]}_{i,k} + \beta_2(1-\alpha_2)\, g^{[t,r]}_{i,k}\odot g^{[t,r]}_{i,k} + (1-\beta_2)\, h^{[t,r]}_{i,k}\odot h^{[t,r]}_{i,k}.
    \end{align}
    %
    %Unlike standard Adam optimization~\cite{KingmaB-ICLR2015} or its decentralized variants~\cite{LinKSJ-ICML21}, which relies on continuously accumulated moment estimates derived from local gradients, our design leverages the intermediate moments $m^{[t,r+\frac{1}{2}]}_{i,k}$ and $v^{[t,r+\frac{1}{2}]}_{i,k}$ to update the BMAs. This choice is critical under block-wise optimization, where moment estimates associated with each parameter block become stale due to intermittent updates. Specifically, we first incorporates the fresh gradient of the active block to reconstruct an objective-aligned moment estimate, and then injects a consensus-derived local discrepancy (i.e., $h^{t,r}_{i,k}$) within the same parameter subspace, ensuring that the update direction remains primarily aligned with the (local) optimization objective while being adaptively steered toward network-wide agreement.
    %
    Unlike standard Adam optimization~\cite{KingmaB-ICLR2015} or its decentralized momentum variants~\cite{LinKSJ-ICML21}, which continuously accumulates moment estimates from local stochastic gradients, our design updates the BMAs using the intermediate moments $m^{[t,r+\frac{1}{2}]}_{i,k}$ and $v^{[t,r+\frac{1}{2}]}_{i,k}$. This design is critical for block-wise optimization, where moment estimates associated with inactive parameter blocks can become stale due to intermittent updates. Specifically, for each active block, we first incorporate the fresh gradient to reconstruct objective-aligned first- and second-order moment estimates, and then inject the consensus-derived local discrepancy $h^{[t,r]}_{i,k}$ within the same parameter subspace. As a result, the update direction remains primarily aligned with the local optimization objective while being adaptively corrected toward network-wide consensus, thereby mitigating client drift and stabilizing block-wise optimization under non-IID data.

\section{Analysis} \label{sec:theory}
  Due to space limitations, this section focuses on the theoretical convergence analysis of DECA. Additional analysis of communication complexity, memory consumption, and computational complexity is provided in Appendix~\ref{sec:add_theory}.
  %
  %
  %\subsection{Convergence Analysis} \label{ssec:convergence}
  %
  Before presenting our main results, we first give the requisite assumptions which have been widely adopted in theoretical analysis of decentralized optimization~\cite{YuJY-ICML19,EsfandiariTJBHHS-ICML21,LinKSJ-ICML21}. Specifically, \textbf{Assumption}~\ref{assum:smooth} imposes Lipschitz smoothness on the objective functions. \textbf{Assumption}~\ref{assum:variance} bounds the stochastic gradient variance and the magnitude of the block-wise gradients. \textbf{Assumption}~\ref{assum:dsm} characterizes the spectral properties of the communication topology. \textbf{Assumption}~\ref{assum:error_control} implies that switching from block $k$ to block $k+1$ may increase the objective value, but such increase is bounded, and the block transitions cannot arbitrarily destabilize the overall optimization.
    \begin{assumption}[\textbf{Lipschitz smoothness}] \label{assum:smooth}
      Each local objective function $f_i(\cdot)$ is $L$-smooth for any $i\in\mathcal{N}$. Specifically, there exist a constant $L \geq 0$ such that
      \begin{equation} \label{eq:smooth}
       f_i(x) \leq f_i(y)+\nabla f^\top_i(y)(x-y)+\frac{L}{2}\|x-y\|^2, ~\forall x,y\in\mathbb{R}^d
      \end{equation}
      and the overall function $f(\cdot)$ is also $L$-smooth such that:
      \begin{equation} \label{eq:smooth_overall}
        f(x) \leq f(y)+\nabla f^\top(y)(x-y)+\frac{L}{2}\|x-y\|^2, ~\forall x,y\in\mathbb{R}^d
      \end{equation}
    \end{assumption}
    \begin{assumption}[\textbf{Bounded Variance and Gradient}] \label{assum:variance}
      There exist non-negative constants $\xi$ and $\delta$ such that the variance of the stochastic gradients is bounded:
      \begin{align} \label{var_1}
        \mathbb{E}_{\zeta\sim\mathcal{D}_i} \big[ \left\| \nabla_k F_i(x; \zeta) - \nabla_k f_i(x) \right\|^2 \big] \leq \xi^2, \forall i\in\mathcal{N}, x\in\mathbb{R}^d, \forall k \in \{1, \dots, B\}&.
      \end{align}
      and the block-wise stochastic gradient is bounded:
      \begin{align} \label{G}
        \mathbb{E}_{\zeta\sim\mathcal{D}_i} \left[ \left\| \nabla_k F_i (x, \zeta) \right\|^2 \right] \leq \delta^2, \forall i\in\mathcal{N}, x\in\mathbb{R}^d, \forall k \in \{1, \dots, B\}.
      \end{align}
    \end{assumption}
    \begin{assumption}[\textbf{Doubly Stochastic Matrix}] \label{assum:dsm}
      The adjacent matrix $\mathbf{W}$ associated with the communication graphs $\mathcal{G}$ is a doubly stochastic matrix. Its spectral gap satisfies:
      \begin{equation} \label{eq:dsm}
        \lambda_1(\mathbf{W}) = 1 ~\text{and}~ \max \left\{  \left|\lambda_2(\mathbf{W})\right|, \left|\lambda_N(\mathbf{W}) \right| \right\} \leq \sqrt{\rho}
      \end{equation}
      where $\lambda_i(\mathbf{W})$ is the $i$-th largest eigenvalue of $\mathbf{W}$ and $\rho \in [0,1)$ is a constant measuring the connectivity of the graph.
    \end{assumption}
    We define a virtual sequence $\hat{x}_k^{[t,r]}$ which has SGD style but scaled by adaptive step sizes: 
    \begin{align}\label{eq:error_df}
      \hat{x}_k^{[t,r]} =
      \begin{cases}
        \bar{x}_{k-1}^{[t,R]} & \text{if } r = 0, \\
        \hat{x}_k^{[t,r-1]} - \eta \bar{g}_k^{[t,r-1]}/(1 - \alpha_1^{r+1})  & \text{if } r \geq 1.
      \end{cases}
    \end{align}
    where $\bar{g}_k^{[t,r]} = \frac{1}{N} \sum_{i=1}^N \nabla_k F_i(x_{i,k}^{[t,r]})$ and $\bar{x}^{[t,r]}_k$ represents the average of ${x}^{[t,r]}_{i,k}$ over all clients. We also define an error sequence $e_k^{[t,r+1]} = \hat{x}_k^{[t,r+1]} - \bar{x}_k^{[t,r+1]}$.
    %
    % \begin{align}
    % e_k^{[t,r+1]} = \hat{x}_k^{[t,r+1]} - \bar{x}_k^{[t,r+1]}
    % \end{align}
    % 
    \begin{assumption}[\textbf{Error-controlled Block Transition}] \label{assum:error_control}
      For any block $k$ and round $t$, there exists constant $C_{\mathrm{err}} > 0$ that
      \begin{equation} \label{eq:error_control}
        \mathbb{E}\big[\mathcal{F}(\hat{x}_{k+1}^{[t,0]})\big] - \mathbb{E}\big[\mathcal{F}(\hat{x}_k^{[t,R]})\big]
        \leq C_{\mathrm{err}} \, \mathbb{E}\big[\|e_k^{[t,R]}\|^2\big],
      \end{equation}
      It is implied that the increase of the objective value along the virtual sequence, from $\hat{x}_k^{[t,R]}$ to $\hat{x}_{k+1}^{[t,0]}$, is controlled by the squared norm of error $e_k^{[t,R]}$.
    \end{assumption}

    \textbf{Theorem}~\ref{thm:convergence} establishes the convergence guarantee of DECA under the above assumption. The average squared gradient norm is bounded by the initial optimality gap, together with additional error terms induced by stochastic gradient variance, decentralized consensus, and block-wise transitions.
    \begin{theorem}\label{thm:convergence}
      Recall $x^{[0]}_{i}$ and $x^{[t,r]}_{i,k}$ denote the initial local model parameters of client $i$, and the parameters of block $k$ at client $i$ during inner step $r$ of round $t$, respectively. Let $\bar{x}^{[t,r]}_k \triangleq \frac{1}{N}\sum_{i=1}^N x^{[t,r]}_{i,k}$ and $\bar{x}^{[0]} = \frac{1}{N} \sum_{i=1}^N x^{[0]}_{i}$. Let $\mathcal{F}^{\star}$ represent the optimal value of the global objective function. Define $0< \lambda < \frac{1-\alpha_1^2\beta_1^2}{\alpha_1^2\beta_1^2}$, and
      \begin{align}
        C_{0} = \frac{(1+\rho)\rho\left( (1+\tfrac{1}{\lambda})(16\delta^2 + 8\rho) + 4 + 2\delta^2\right)}{(1-\rho)(1-\alpha_1)^2 \epsilon^2 \left(1-(1+\lambda) \beta_1^2 \alpha_1^2\right)}, ~\text{and}~
        C_{1} = \frac{\alpha_1^2 + \delta^2(1-\alpha_1)^2}{(1-\alpha_1)^2 \epsilon^2}.
      \end{align}
      Under \textbf{Assumptions}~\ref{assum:smooth}--\ref{assum:error_control}, when 
      \[
        0< \gamma \leq \min\left\{ \frac{(1-\beta_1)}{3L\sqrt{TBR}}, \frac{2(1-\beta_1)}{L+6RC_{err} +\sqrt{(L+6RC_{err})^2 + 36L^2R^2}} \right\},
      \]
      we have
      \begin{align}
        \frac{1}{TBR} \sum_{t=0}^{T-1} \sum_{r=0}^{R-1} \sum_{k=1}^{B} \mathbb{E}\left[\left\|\nabla_k \mathcal{F}\left( \bar{x}_{k}^{[t,r]} \right)\right\|^2\right]
        \leq& \frac{8(1 -\beta_1^{R})}{TBR\gamma}\left( \e{\mathcal{F}(\bar{x}^{0}) - \mathcal{F}^{\star}} \right) + 8C_{1}C_{err}(1 -\beta_1^{R})\gamma  \nonumber \\
       & + \frac{4(1 -\beta_1^{R})L \gamma\xi^2}{N(1 -\beta_1)^2} + \frac{4L^2\gamma^2 (4C_{0}+3RC_{1})}{(1-\rho)(1-\beta_1)}
       % 
       % & + \frac{16(1 -\beta_1^{R}) L^2C_{0}\gamma^2}{(1-\rho)(1-\beta_1)} + \frac{12(1 -\beta_1^{R})R C_{1} L^2\gamma^2}{(1-\beta_1)} + 8C_{1}C_{err}(1 -\beta_1^{R})\gamma
      \end{align}
    \end{theorem}
    %
    % \todo[inline]{What does $\eta$ stand for?}

    In the following corollary, we characterize the convergence rate of DECA more explicitly.
    \begin{corollary} \label{cor:convergence}
      Let learning rate $\gamma = \mathcal{O}(1/\sqrt{TBR})$. When $T$ is sufficiently large such that $T\geq \frac{ \mathcal{O}\left(L^2 + R^2C_{err}^2 + L^2R^2\right)}{L^2BR}$, we have $\frac{1}{TBR} \sum_{t=0}^{T-1} \sum_{r=0}^{R-1} \sum_{k=1}^{B} \mathbb{E}\left[\left\|\nabla_k \mathcal{F}(\bar{x}_{k}^{[t,r]})\right\|^2\right] = \mathcal{O}\left(\frac{1}{\sqrt{TBR}}\right).$
      %
      % \begin{align}\label{eq:main_eq}
      %   \frac{1}{TBR} \sum_{t=0}^{T-1} \sum_{r=0}^{R-1} \sum_{k=1}^{B} \mathbb{E}\left[\left\|\nabla_k \mathcal{F}(\bar{x}_{k}^{[t,r]})\right\|^2\right] = \mathcal{O}\left(\frac{1}{\sqrt{TBR}}\right).
      % \end{align}
    \end{corollary}
    %
    %Furthermore, assuming $B$ are fixed, the algorithm converges at a rate of $\mathcal{O}\left(\frac{1}{\sqrt{TR}}\right)$, which matches the optimal asymptotic rates established for state-of-the-art decentralized learning algorithms~\cite{EsfandiariTJBHHS-ICML21,XuZW-TPAMI21,GhiasvandAP-REALM25}. The detailed proof is provided in Appendix~\ref{sec:convergenceproof}.
    %
    %\textbf{Corollary}~\ref{cor:convergence} formally establishes that DECA achieves a convergence rate of $\mathcal{O}(B/\sqrt{TR})$, which aligns with the optimal asymptotic rates found in state-of-the-art decentralized learning algorithms (e.g., \cite{LinKSJ-ICML21,LianZZHZL-NIPS17,EsfandiariTJBHHS-ICML21,XuZW-TPAMI21,GhiasvandAP-REALM25}) when the number of blocks $B$ is treated as a constant. Notably, this guarantee holds even under non-IID data distributions. Intuitively, this result quantifies the trade-off inherent in the block-wise strategy: while partitioning the model into $B$ blocks introduces a linear scaling factor to the convergence timeline (effectively requiring $B$ times more iteration to touch all parameters), the algorithm preserves the fundamental $1/\sqrt{TR}$ convergence speed of stochastic optimization. This confirms that DECA theoretically ensures convergence to a solution comparable to full-model synchronization methods, but with amortized resource costs over sequential block updates, as will be shown later.
    %
    The above theoretical results formally establish that DECA achieves a convergence rate of $\mathcal{O}(1/\sqrt{TBR})$, which matches the optimal asymptotic rates of state-of-the-art decentralized learning algorithms (e.g., \cite{LinKSJ-ICML21,LianZZHZL-NIPS17,EsfandiariTJBHHS-ICML21,XuZW-TPAMI21,GhiasvandAP-REALM25}) when the number of blocks $B$ is treated as a constant. Notably, this guarantee holds even under non-IID data distributions. 
    % 
    % Intuitively, this theoretical result captures the fundamental trade-off of block-wise stochastic optimization: partitioning the model into $B$ blocks introduces a linear factor in the convergence bound, reflecting the sequential traversal of all parameters, while preserving the $\mathcal{O}(1/\sqrt{TR})$ convergence rate of stochastic optimization. This demonstrates that DECA converges to solutions comparable to full-model decentralized methods.
    % 
    Intuitively, this theoretical result captures the fundamental trade-off of block-wise stochastic optimization: partitioning the model into $B$ blocks reflects the sequential traversal of all parameters. The convergence rate of $\mathcal{O}(1/\sqrt{TBR})$ matches the standard stochastic optimization rate $\mathcal{O}(1/\sqrt{TR})$ under a unified accounting of block-level updates, demonstrating that DECA converges to solutions comparable to full-model decentralized methods.

\section{Experiments}\label{sec:exp}
  In this section, we first introduce the experimental settings in Sec.\ref{ssec:expsetup}, and then report the results of our extensive experiments in Sec.\ref{sssec:bslcls} and Sec.\ref{ssec:asexp}. More details on the experimental setup and additional experiment results are provided in Appendices~\ref{sec:app_expsetting} and \ref{sec:app_expresults}, respectively.

  \subsection{Experiment Settings} \label{ssec:expsetup}
    We validate our framework across four widely adopted LLMs with varying parameter scales: {Llama-2-7B}~\cite{TouvronMSAABBBSBB-arXiv23}, {Llama-3.1-8B-Instruct}~\cite{DubeyJPAALMSF-ARXIV24}, {Qwen2-1.5B}~\cite{YangYHZYZLILH-arXiv24}, and {Qwen2.5-3B-Instruct}~\cite{QwenAY-arXiv25}. Our evaluation encompasses both classification and generation tasks. 
    For classification, we use four standard datasets (NWGI, AGNEWS, TFNS, and MNLI) with their pre-defined training and test splits, and report accuracy (Acc.) and $F_1$ score as evaluation metrics. To evaluate generation capability, we use the Alpaca dataset and assess model outputs with Vicuna (VIC.)~\cite{ChiangLLSWZZZZG-blog2023} and MT-Bench (MT.)~\cite{ZhengCSZWZLLLX-NIPS24}.

    % whereas for the generative dataset we evaluate model outputs using Rouge-L (ROU-L)~\cite{Lin-TSBO04}, Meteor (MET.)~\cite{BanerjeeL-ACL05}, CIDEr~\cite{VedantamLP-CVPR15} and the MT-Bench (MT.)~\cite{ZhengCSZWZLLLX-NIPS24}. 

    We consider an Erd\H{o}s--R\'enyi (ER) graph consisting of eight clients. To simulate non-IID data distributions, we partition the data using a Dirichlet distribution parameterized by $\mu$,~\cite{LinKSJ-ICML21, TakezawaBNSY-TMLR23, GhiasvandAP-REALM25, GuoZWFQ-ICLR25}. A smaller $\mu$ (e.g., $\mu=0.25$ in our case) induces more skewed client data distributions and thus stronger heterogeneity.
    We set the learning rate to $\gamma = 5 \times 10^{-5}$, the total number of communication rounds to $T=4$, and the number of local update steps per block to $R=48$. For the Adam optimizer, we adopt the following hyperparameters: $\alpha_1=0.9$, $\alpha_2=0.999$, $\epsilon=1 \times 10^{-8}$. For BMA, we adopt $\beta_1=0.9$, $\beta_2=0.999$ for classification tasks and $\beta_1=0.3$, $\beta_2=0.333$ for generation tasks. For model partitioning, we group every two consecutive Transformer layers into a trainable block, such that the Llama-2-7B, Llama-3.1-8B, Qwen2-1.5B, and Qwen2.5-3B models are partitioned into 32, 36, 28 and 32 blocks, respectively. 
    We select three state-of-the-art decentralized LLM fine-tuning methods as baselines:
    \textbf{Dec-Adapter}~\cite{HoulsbyGJMDGAG-ICML19} extends PEFT methods originally implemented via adapters to decentralized scenarios.
    \textbf{Dec-LoRA}~\cite{GhiasvandAP-REALM25} and \textbf{DeCAF}~\cite{Saadati-NeuralNet26} propose decentralized LoRA strategies.

  \subsection{Comparison with Baselines} \label{sssec:bslcls}
    We first compare DECA with existing baselines on classification tasks. As shown in Table~\ref{tab:exp_classification}, DECA achieves the best average accuracy and $F_1$ score on both Qwen2.5-3B and Llama-3.1-8B models, demonstrating its consistent effectiveness across model scales. On Qwen2.5-3B, DECA outperforms the strongest baseline by 2.95\% in average accuracy and 3.36\% in average $F_1$ score, with improvements observed on all four datasets. On Llama-3.1-8B, DECA also obtains the best average results. These results indicate that DECA effectively alleviates the limitations imposed by the low-rank approximations used in prior decentralized fine-tuning methods.
    %
    % On average, DECA surpasses the best DL baseline, Dec-LoRA, by substantial margins, specifically reaching gains of up to 1.6\% in accuracy and 1.3\% in $F_1$ score on the Llama-3.1-8B-Instruct model. Notably, DECA achieves these results without the need for an additional central server.
    %
    %
    \begin{table*}[t!]
    \centering
    \begin{small}
    \caption{Performance comparison in classification tasks. Values in \textbf{bold} indicate the best performance among the \emph{decentralized} methods, while \underline{underlined} values denote the second best results. The last column reports the average performance across all datasets.}
    \label{tab:exp_classification}
    \begin{tabular}{l *{5}{c} *{5}{c}}
    \toprule
    \multirow{2}{*}{\textbf{Methods}} 
    % & \multicolumn{12}{c}{\textbf{Datasets \& Evalutions}} \\
    % %
    % %
    % \cmidrule(lr){2-13}
    & \multicolumn{2}{c}{\textbf{NWGI}} & \multicolumn{2}{c}{\textbf{AGNEWS}} & \multicolumn{2}{c}{\textbf{TFNS}} & \multicolumn{2}{c}{\textbf{MNLI}}  & \multicolumn{2}{c}{\textbf{Average}} \\
    \cmidrule(lr){2-3} \cmidrule(lr){4-5} \cmidrule(lr){6-7} \cmidrule(lr){8-9} \cmidrule(lr){10-11} 
    & \textit{Acc.} & $F_1$ & \textit{Acc.} & $F_1$ & \textit{Acc.} & $F_1$ & \textit{Acc.} & $F_1$ & \textit{Acc.} & $F_1$ \\
    \midrule
    \rowcolor{lightgreen} \multicolumn{11}{c}{\textit{Qwen2.5-3B}} \\
    \midrule
    %
    % FFA-LoRA
    % &50.11 &28.68 &65.20 &65.75 &69.38 &58.44 &68.75 &66.90 &52.05 &45.41 & 61.10 & 53.04 \\
    %
    % FedSA-LoRA
    % & 56.39 & 37.03 & 74.66 & 74.29 & 71.73 & 61.57 & 78.32 & 77.78 & 69.05 & 64.57 & 70.03 & 63.05 \\
    %
    % FedPrompt
    % & 37.78 & 22.87 & 55.58 & 56.48 & 43.09 & 28.82 & 47.51 & 39.87 & 36.85 & 29.62 & 44.16 & 35.53 \\
    % \hdashline
    %
    Dec-Adapter
    &\underline{52.72} &\underline{35.55} &83.82 &83.36 &88.36 &\underline{85.47} &\underline{80.57} &\underline{78.65} &\underline{76.37} &\underline{70.76} \\
    Dec-LoRA
    % &51.44 &33.23 &81.30 &81.11 &85.51 &79.80 &77.62 &76.43 &73.97 &67.64 \\
    &52.44 &35.23 &\underline{87.30} &\underline{86.11} &84.69 & 78.31 &77.63 &76.01 &75.52 &68.92 \\
    DeCAF
    &42.93 &27.96 &67.45 &65.36 &66.12 &55.52 &72.09 &67.62 &62.15 &54.12 \\
    \textbf{DECA}
    &\textbf{53.51} &\textbf{36.09} &\textbf{88.84} &\textbf{88.79} &\textbf{89.25} &\textbf{86.24} &\textbf{85.67} &\textbf{85.34} &\textbf{79.32} &\textbf{74.12} \\
    % 
    % \textbf{DECA(w/o QM)}
    % &48.67 &31.53 &83.67 &83.46 &82.38 &76.62 &77.62 &75.98 &73.09 &66.90 \\
    % %
    % \textbf{$4$-DECA}
    % &50.34 &35.01 &83.82 &83.36 &85.32	&80.66 &78.20	&76.89 &74.42 &68.98 \\
    % % 
    % \textbf{$8$-DECA}
    % &49.10 &33.71 &83.05 &82.57 &84.34 &80.14 &77.70 &75.75 &73.55 &68.04 \\
    %
    \midrule
    \rowcolor{darkgreen} \multicolumn{11}{c}{\textit{Llama-3.1-8B}} \\
    \midrule
    %
    % FFA-LoRA 
    % & 53.51 & 34.59 & 82.36 & 80.97 & 82.66 & 79.45 & 69.86 & 66.07 & 58.50 & 49.37 & 70.48 & 62.69 \\
    %
    % FedSA-LoRA
    % & 57.10 & 40.51 & 83.39 & 82.23 & 72.78 & 56.32 & 78.07 & 77.48 & 67.33 & 60.25 & 71.73 & 63.36 \\
    %
    Dec-Adapter
    & 12.47 &4.71 &24.66 &12.24 &23.29 &15.24 &32.68 &20.10 &23.28 &13.07 \\
    Dec-LoRA
    & \underline{54.72} & \underline{38.35} &\underline{90.76} &\underline{90.74} &\textbf{91.03} &\textbf{88.42} &\underline{85.15} &\underline{84.68} &\underline{80.42} &\underline{75.55} \\
    DeCAF
    & 47.13 & 29.03 & 70.74 & 65.13 & 67.98 & 56.13 & 72.49 & 67.95 &64.59 &54.56 \\
    \textbf{DECA} 
    & \textbf{57.36} &\textbf{39.82} &\textbf{92.79} &\textbf{92.77} &\underline{90.42} &\underline{87.66} &\textbf{88.64} & \textbf{88.46} &\textbf{82.30} &\textbf{77.18} \\
    %
    % \textbf{$\tau$-DECA}
    % &52.61 &35.58 & \underline{85.45} & \underline{83.95} & \underline{86.46} & \underline{83.26} &73.20 & 70.46 & \underline{76.93} & \underline{71.62} \\
    %
    % \midrule
    % \rowcolor{darkgreen} \multicolumn{11}{c}{\textit{Qwen2.5-14B-Instruct}} \\
    % \midrule
    %
    % FFA-LoRA 
    % & 56.91 & 30.59 & 86.32 & 84.24 & 78.27 & 76.32 & 80.04 & 78.03 & 57.75 & 47.50  & 72.46 & 64.34 \\
    %
    % FedSA-LoRA
    % & \textbf{58.76} & 35.45 & 84.25 & 83.23 & 75.88 & 65.81 & \textbf{81.36} & \textbf{79.16} & 64.70 & 59.39  & 74.39 & 66.41 \\
    %
    % FedPrompt
    % & 53.48 & 29.93 &82.30 & 81.95 &83.84 &76.61 &50.47 & 43.07 &51.85 & 43.14 & 64.39 & 54.94 \\ \hdashline
    %
    % Dec-LoRA
    % & 54.51 &\textbf{40.16} &92.95 &92.95 &\textbf{90.56} &87.09 &87.76 & 87.29 &81.45 &76.87 \\
    %
    % DeCAF
    % & 47.69 & 27.35 & 72.65 & 68.46 & 69.49 & 57.37 &65.07 &60.22 & 62.40 & 51.68 \\
    %
    % \textbf{DECA} 
    % & \textbf{58.00} &38.30 &\textbf{93.41} &\textbf{93.41} &90.20 &\textbf{87.25} &\textbf{89.62} & \textbf{89.46} &\textbf{82.81} &\textbf{77.11} \\
    %
    % \textbf{$\tau$-DECA}
    % & 53.54 & \underline{36.87} & \underline{86.93} & \underline{85.20} & \underline{88.39} & \underline{85.45} & 73.56 & 70.56 & \underline{76.81} & \underline{71.33} \\
    %
    \bottomrule
    \end{tabular}
    \end{small}
    \end{table*}

    \begin{table*}[t!]
    \centering
    \caption{Performance comparison in generation tasks.}
    \label{tab:gen_bsl}
    \begin{small}
    \begin{tabular}{l *{4}{c} | *{4}{c} }
    \toprule
    \multirow{2}{*}{\textbf{Methods}} & \multicolumn{4}{c}{Qwen2-1.5B} & \multicolumn{4}{c}{Llama-2-7B} \\
    \cmidrule(lr){2-5} \cmidrule(lr){6-9} 
    & \textbf{VIC.} & \textbf{MT-1} & \textbf{MT-2} & \textbf{MT.} & \textbf{VIC.} & \textbf{MT-1} & \textbf{MT-2} & \textbf{MT.} \\
    \midrule
    Dec-Adapter
    & 5.38 & 4.18 & 2.65 & 3.41 & 4.86 & 3.92 & 2.62 & 3.27 \\
    Dec-LoRA 
    & 5.52 & 4.43 & 2.95 & 3.69 & 5.60 & 4.30 & 2.87 & 3.59 \\
    DeCAF
    &\textbf{5.71} &4.87 &3.26 &4.06 & \textbf{5.91} & \underline{4.31} & \textbf{3.00} & \underline{3.65}  \\
    \textbf{DECA}
    & \underline{5.68} &\textbf{4.88} & \textbf{3.30} &\textbf{4.10}  & \underline{5.86} & \textbf{4.50} & \underline{2.93} & \textbf{3.72} \\
    \bottomrule
    \end{tabular}
    \end{small}
    \end{table*}

    Table~\ref{tab:gen_bsl} compares DECA with the decentralized PEFT baselines on generation tasks. On Qwen2-1.5B, DECA achieves the best MT-1 (4.88), MT-2 (3.30), and overall MT (4.10) scores, while remaining close to DeCAF on VIC. On Llama-2-7B, DECA obtains the highest MT-1 (4.50) and overall MT (3.72) scores, although DeCAF performs slightly better on VIC. and MT-2. Overall, DECA consistently delivers the best or near-best performance across model scales, validating the effectiveness of the proposed method for generation tasks. It is worth noting that, although DeCAF can achieve slightly better results on a few metrics by employing TSVD decomposition, it incurs substantially higher computation overhead, up to $230.52\times$ that of DECA, as shown in Appendix~\ref{ssec:app_exp_resource}.    
    % %
    %

    To illustrate the convergence behaviors of the different algorithms on the classification and generation tasks, we plot the training loss trajectories over model update steps (which correspond to iterations of the innermost loop $r$), in Fig.~\ref{fig:loss_tfns} and Fig.~\ref{fig:loss_alpaca}, respectively. On the classification tasks, DECA achieves a substantially lower loss than Dec-Adapter and converges at a faster rate, demonstrating more efficient optimization under non-IID data. Compared with the LoRA-based decentralized PEFT baselines, Dec-LoRA and DeCAF, DECA exhibits comparable convergence rates and final loss values, indicating that it can maintain stable optimization performance while using full-parameter fine-tuning. On the generation tasks, DECA again clearly outperforms Dec-Adapter, which consistently suffers from the highest loss values and the slowest convergence. Meanwhile, DECA shows convergence behavior similar to Dec-LoRA, reaching comparable final losses with a similar convergence speed. Although DeCAF continues to reduce its loss throughout training (with much higher computation overhead), it exhibits larger fluctuations within the observed update steps. Overall, DECA achieves competitive convergence performance compared with LoRA-based decentralized algorithms.

    % \textcolor{green}{To illustrate the convergence behaviors of the decentralized methods (DECA, $\tau$-DECA, Dec-LoRA, and DeCAF) on classification and generation tasks, we plot the training loss trajectories over model update steps (which correspond to iterations of the innermost loop $r$), in Fig.~\ref{fig:loss_tfns} and Fig.~\ref{fig:loss_alpaca}, respectively. Across all model scales, DECA and $\tau$-DECA consistently converge faster and achieve lower loss values than decentralized PEFT baselines such as Dec-LoRA and DeCAF, demonstrating more efficient and stable optimization under non-IID data. DECA typically exhibits the steepest early loss reduction and the smoothest convergence, while $\tau$-DECA closely follows with slightly slower initial progress due to reduced synchronization frequency, yet reaches comparable final losses. The performance gap is particularly pronounced in the generation task, indicating that full-parameter fine-tuning is more effective than low-rank adaptations for capturing complex generative behaviors. Overall, these results confirm that DECA achieves superior convergence across tasks, and that $\tau$-DECA retains most of this advantage while improving communication efficiency. }

    %
    \begin{figure*}[t]
    \centering
      \begin{minipage}[t]{0.48\textwidth}
      \centering
      \includegraphics[width=\textwidth]{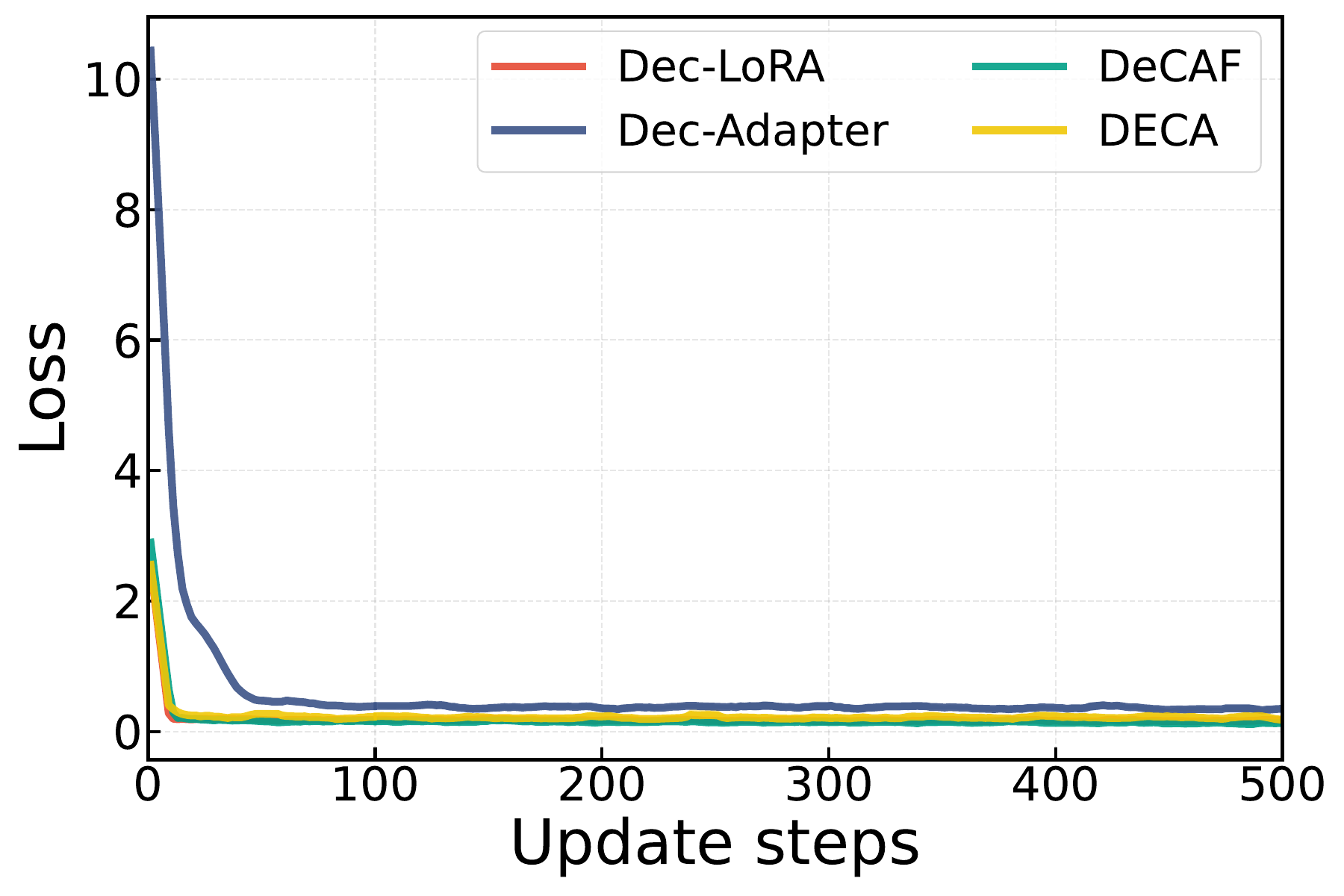}
      \captionof{figure}{Training loss of different algorithms on TFNS dataset using Llama-3.1-8B model.}
      \label{fig:loss_tfns}
      \end{minipage}
      \hfill
      \begin{minipage}[t]{0.48\textwidth}
      \centering
      \includegraphics[width=\textwidth]{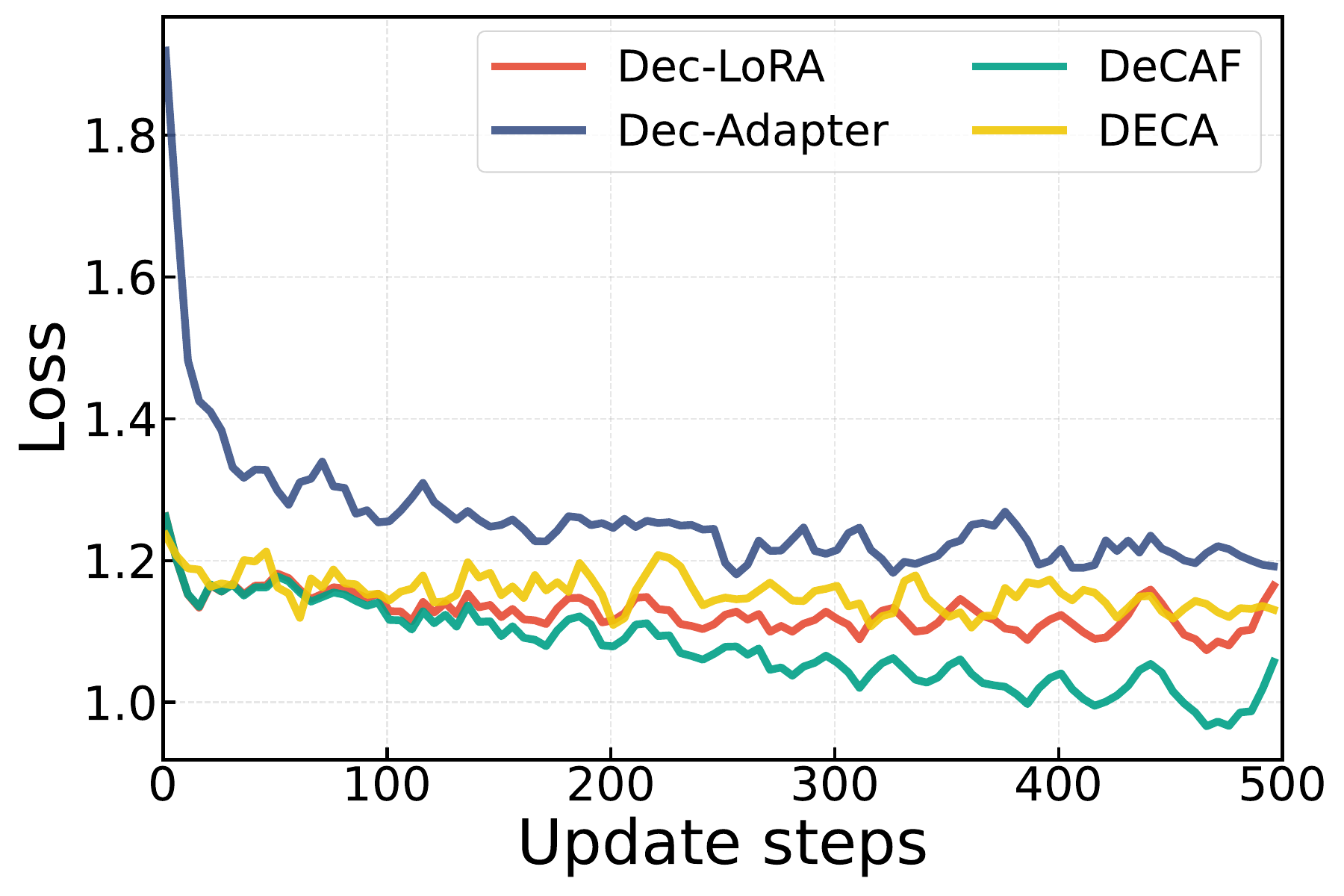}
      \captionof{figure}{Training loss of different algorithms on Alpaca dataset using Qwen2-1.5B model.}
      \label{fig:loss_alpaca}
      \end{minipage}
    \end{figure*}

    \begin{figure*}[t]
        \centering
        \begin{minipage}[t]{0.48\textwidth}
            \centering
            \includegraphics[width=\textwidth]{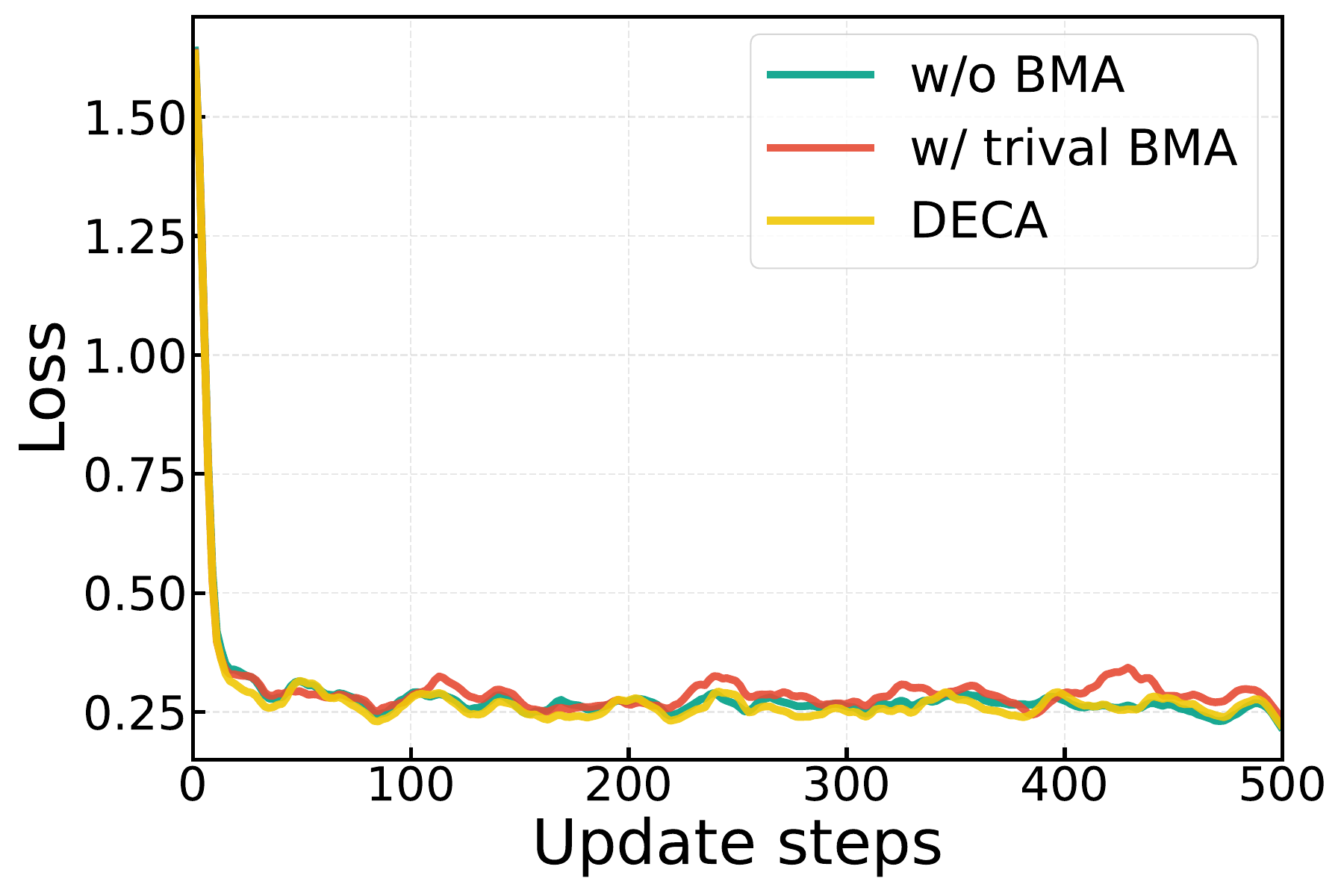}
            \captionof{figure}{Training loss of different algorithms on NWGI dataset using Llama-3.1-8B model.}
            \label{fig:loss_abl_llama}
        \end{minipage}
        \hfill
        \begin{minipage}[t]{0.48\textwidth}
            \centering
            \includegraphics[width=\textwidth]{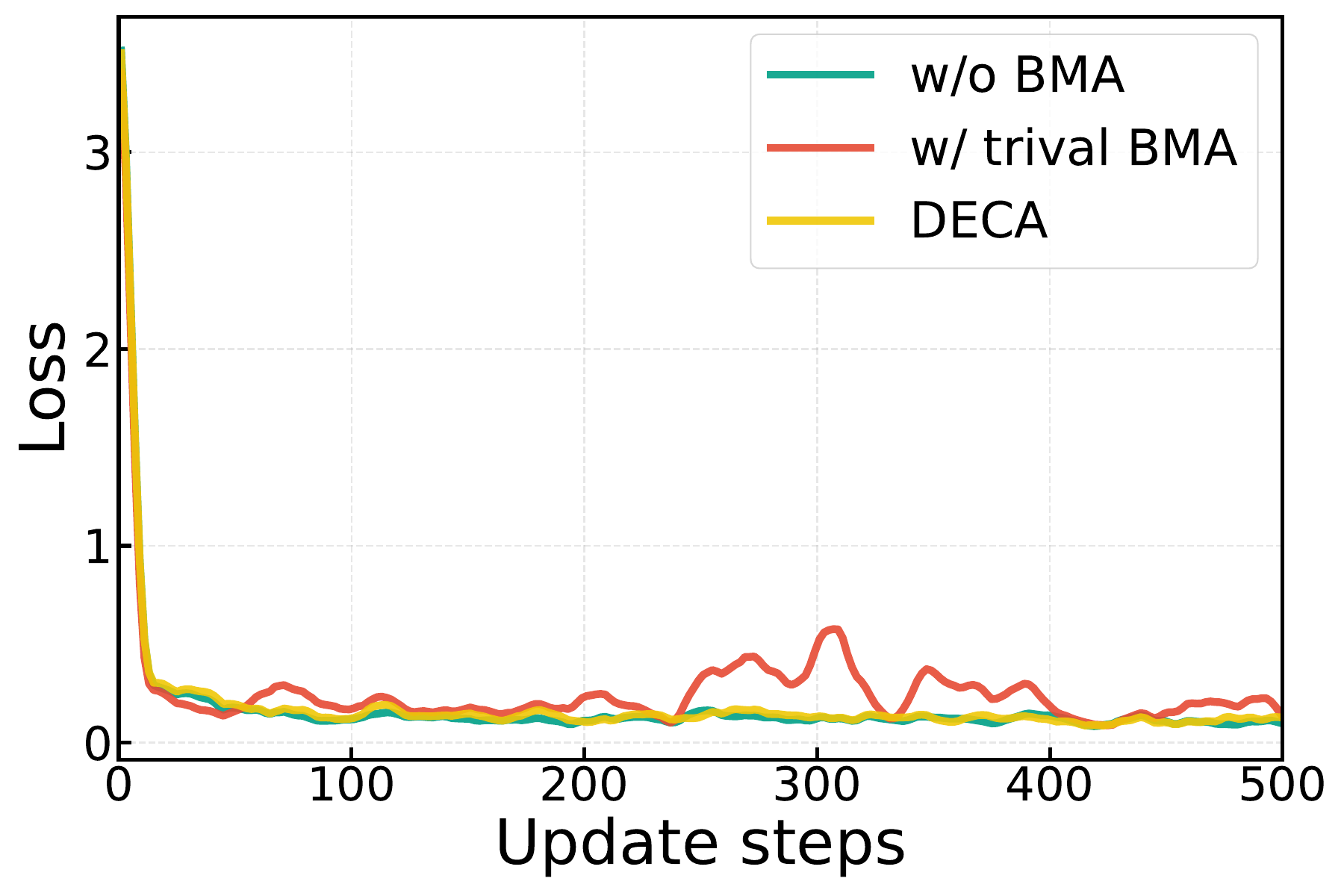}
            \captionof{figure}{Training loss of different algorithms on AGNEWS dataset using Qwen2.5-3B model.}
            \label{fig:loss_abl}
        \end{minipage}
    \end{figure*}

    \subsection{Ablation Study}\label{ssec:asexp}

    To investigate the importance of the BMA module in our algorithm, we design two ablation settings: \textbf{w/ trivial BMA} and \textbf{w/o BMA}. In the former setting, we let $m^{[t,r+1]}_{i,k} = \alpha_1 m^{[t,r]}_{i,k} + (1-\alpha_1) h^{[t,r]}_{i,k}$ and $v^{[t,r+1]}_{i,k} = \alpha_2 v^{[t,r]}_{i,k} + (1-\alpha_2)\, h^{[t,r]}_{i,k}\odot h^{[t,r]}_{i,k}$, which can be viewed as a block-wise extension of the quasi-global momentum principle in~\cite{LinKSJ-ICML21}, while in the latter one, we ablate the entire BMA module by removing Lines~12--13 in \textbf{Algorithm}~\ref{alg:deca}, which yields a trivial decentralized variant of the block-wise Adam optimization~\cite{LuoYLi-NIPS24}.
    Table~\ref{tab:exp_classification_abl} illustrates that DECA consistently achieves the best performance among ablation settings on classification tasks. Taking {Qwen2.5-3B} as an example, DECA reaches an average accuracy of 79.32\% and an average $F_1$ score of 74.12\%, while removing BMA leads to a clear degradation of 6.23\% and 7.22\%, respectively; replacing by w/ trivial BMA further reduces performance to 65.19\% and 61.69\%. A similar pattern is observed on Llama-3.1-8B, where DECA achieves 82.30\% accuracy and a 77.18\% $F_1$ score, consistently outperforming w/ trivial BMA and w/o BMA.
    Fig.~\ref{fig:loss_abl_llama} and Fig.~\ref{fig:loss_abl} further compare the training loss under different ablation settings across models. All methods exhibit similar convergence speeds in the early stage of fine-tuning. However, as training proceeds, w/ trivial BMA shows noticeable oscillations in the loss trajectory, especially on Qwen2.5-3B, suggesting that overemphasizing consensus-derived information may disturb local optimization. In contrast, w/o BMA remains relatively stable, but its performance is clearly inferior to DECA, as shown in Table~\ref{tab:exp_classification_abl}. These results indicate that DECA achieves a better balance between local gradient-driven optimization and consensus-guided correction, leading to both stable convergence and stronger downstream performance.

    \begin{table*}[t!]
    \centering
    \begin{small}
    \caption{Ablation results for classification tasks.}   
    \label{tab:exp_classification_abl}
    \begin{tabular}{l *{5}{c} *{5}{c}}
    \toprule
    \multirow{2}{*}{\textbf{Methods}} 
    % & \multicolumn{12}{c}{\textbf{Datasets \& Evalutions}} \\
    % %
    % %
    % \cmidrule(lr){2-13}
    & \multicolumn{2}{c}{\textbf{NWGI}} & \multicolumn{2}{c}{\textbf{AGNEWS}} & \multicolumn{2}{c}{\textbf{TFNS}} & \multicolumn{2}{c}{\textbf{MNLI}}  & \multicolumn{2}{c}{\textbf{Average}} \\
    \cmidrule(lr){2-3} \cmidrule(lr){4-5} \cmidrule(lr){6-7} \cmidrule(lr){8-9} \cmidrule(lr){10-11} 
    & \textit{Acc.} & $F_1$ & \textit{Acc.} & $F_1$ & \textit{Acc.} & $F_1$ & \textit{Acc.} & $F_1$ & \textit{Acc.} & $F_1$ \\
    \midrule
    \rowcolor{lightgreen} \multicolumn{11}{c}{{Qwen2.5-3B}} \\
    \midrule
    \textbf{DECA}
    &\textbf{53.51} &\textbf{36.09} &\textbf{88.84} &\textbf{88.79} &\textbf{89.25} &\textbf{86.24} &\textbf{85.67} &\textbf{85.34} &\textbf{79.32} &\textbf{74.12} \\
    w/ trivial BMA
    &30.62 &21.38 &\underline{85.36} &\underline{85.24} &76.66 &73.49 &68.10 &66.63 &65.19 &61.69 \\
    w/o BMA
    &\underline{48.67} &\underline{31.53} &83.67 &83.46 &\underline{82.38} &\underline{76.62} &\underline{77.62} &\underline{75.98} &\underline{73.09} &\underline{66.90} \\
    %
    % \textbf{$4$-DECA}
    % &50.34 &35.01 &83.82 &83.36 &85.32	&80.66 &78.20	&76.89 &74.42 &68.98 \\
    % % 
    % \textbf{$8$-DECA}
    % &49.10 &33.71 &83.05 &82.57 &84.34 &80.14 &77.70 &75.75 &73.55 &68.04 \\
    %
    \midrule
    \rowcolor{darkgreen} \multicolumn{11}{c}{{Llama-3.1-8B}} \\
    \midrule
    \textbf{DECA} 
    &\textbf{57.36} &\textbf{39.82} &\textbf{92.79} &\textbf{92.77} &\textbf{90.42} &\textbf{87.66} &\textbf{88.64} & \textbf{88.46} &\textbf{82.30} &\textbf{77.18} \\
    w/ trivial BMA
    &\underline{56.85} &37.48 &\underline{90.82} &\underline{90.75} &\underline{90.23} &\underline{87.35} &\underline{87.26} &\underline{87.01} &\underline{81.29} &\underline{75.65} \\
    w/o BMA
    &53.81 &\underline{37.54} &86.49 &85.95 &86.03 &83.21 &78.85 &76.07 &76.30 &70.69 \\
    \bottomrule
    \end{tabular}
    \end{small}
    \end{table*}

    % \begin{table*}[t!]
    % \centering
    % \caption{Ablation results for generation tasks.}
    % \label{tab:gen_bsl_abl}
    % \begin{small}
    % \begin{tabular}{l *{4}{c} | *{4}{c} }
    % \toprule
    % \multirow{2}{*}{\textbf{Methods}} & \multicolumn{4}{c}{Qwen2-1.5B} & \multicolumn{4}{c}{Llama-2-7B} \\
    % \cmidrule(lr){2-5} \cmidrule(lr){6-9} 
    % & \textbf{VIC.} & \textbf{MT-1} & \textbf{MT-2} & \textbf{MT.} & \textbf{VIC.} & \textbf{MT-1} & \textbf{MT-2} & \textbf{MT.} \\
    % \midrule
    % %
    % \textbf{DECA}
    % & \textbf{5.68} &\textbf{4.88} & \textbf{3.30} &\textbf{4.10} & \textbf{5.78} & \textbf{4.61} & \underline{2.84} & \textbf{3.73} \\
    % % 
    % w/ trivial BMA 
    % & 5.31 & 4.28 & 2.88 & 3.58 
    % & 5.60 & 4.43 & 2.81 & 3.62 \\
    % %
    % w/o BMA
    % & \underline{5.63} &\underline{4.81} &\underline{3.15} &\underline{3.98} & \underline{5.65} &\underline{4.51} & \textbf{2.86} & \underline{3.68} \\ 
    % %
    % \bottomrule
    % \end{tabular}
    % \end{small}
    % \end{table*}

\section{Related Work} \label{sec:relwork}
%
%Due to space limitations, we provide a brief literature survey below. 
  
%\subsection{Efficient LLM Fine-Tuning}

  To reduce the computational burden of LLM training, PEFT methods—such as Adapter~\cite{HoulsbyGJMDGAG-ICML19}, Prompt Tuning~\cite{LesterAC-EMNLP21,LiL-ARXIV221}, and the widely adopted LoRA~\cite{HuSSAWLWC-ICLR22,LialinSNA-ICLR24,XiaCH-arXiv24}—update only a fraction of parameters. While efficient, PEFT often restricts model capacity. Conversely, memory-efficient FPFT methods have emerged, leveraging techniques like on-the-fly gradient computation (LOMO)~\cite{LvYLGQ-ACL24}, zeroth-order approximations (MeZO)~\cite{MalladiGNDLCA-NIPS23}, gradient projections (GaLore)~\cite{ZhaoZCWAT-ICML24}, and CPU offloading~\cite{RenRARYZLH-ATC21,LiuZLLFWZS-EMNLP24}. Notably, \cite{LuoYLi-NIPS24} employs block coordinate descent to construct a block-wise Adam optimizer. However, this approach is strictly centralized and its theoretical convergence guarantee holds only in a deterministic setting. Extending it to decentralized stochastic optimization is highly non-trivial due to the mismatch between local block-wise updates and globally heterogeneous data distributions.

  %\subsection{Federated and Decentralized Learning}
  %
    To preserve privacy, FL aggregates updates via a central server. Due to communication constraints, most FL-LLM works rely on PEFT, including prompt-based FedPrompt~\cite{ZhaoDLLL-ICASSP23} and various LoRA adaptations like FDLoRA~\cite{QiLHFYQ-arXiv24}, FedSA-LoRA~\cite{GuoZWFQ-ICLR25}, and FFA-LoRA~\cite{SunLLD-ICLR24}. While recent works like FedKSeed~\cite{QinCQDLD-ICML24} and Ferret~\cite{ShuHNLY-ICML25} attempt federated FPFT, they remain dependent on a central coordinator.
    Decentralized learning eliminates the server bottleneck via peer-to-peer communication. Foundational algorithms like DPSGD~\cite{LianZZHZL-NIPS17} and DMSGD~\cite{YuJY-ICML19} established convergence for IID data. To address data heterogeneity, advanced techniques such as Gradient Tracking~\cite{PuN-MP21,TakezawaBNSY-TMLR23} and Cross-Gradient Aggregation~\cite{EsfandiariTJBHHS-ICML21} have been developed. 
    %
    %\textcolor{red}{\cite{LinKSJ-ICML21} introduces Quasi-Global (QG) momentum to stabilize decentralized DNN training under data heterogeneity. We build on this principle, but extend it to the block-wise optimization setting by designing both first- and second-order BMAs. These BMAs jointly incorporate local gradient information and consensus-derived discrepancy signals, thereby mitigating client drift and stabilizing block-wise optimization on non-IID data.} 
    %
    \cite{LinKSJ-ICML21} introduces Quasi-Global (QG) momentum for decentralized SGD under data heterogeneity, which locally approximates the global optimization direction. Building on this principle, we design first- and second-order BMAs that integrate fresh local gradient statistics with consensus-derived discrepancy signals in a block-wise manner, thereby mitigating client drift and stabilizing Adam-style optimization on non-IID data with theoretical guarantees.
    While decentralized adaptation for LLMs has been explored via PEFT (e.g., Dec-LoRA~\cite{GhiasvandAP-REALM25} and DeCAF~\cite{Saadati-NeuralNet26}), these methods are limited by updating only a small subset of model parameters. \emph{To the best of our knowledge, no existing work has established a decentralized framework for FPFT, representing an unexplored frontier for achieving superior task-specific adaptation.} A detailed literature survey is available in Appendix~\ref{sec:app_relwork}.

\section{Conclusion} \label{sec:conclusion}
  In this paper, we presented DECA, a resource-efficient decentralized FPFT framework for LLMs under non-IID data. DECA decentralizes block-wise Adam optimization by partitioning model parameters into disjoint blocks, enabling clients to collaboratively fine-tune the full model with reduced resource consumption. To mitigate client drift caused by data heterogeneity, DECA introduces first- and second-order BMAs that incorporate fresh local gradient statistics and consensus-derived discrepancy signals into Adam-style moment estimation, thereby stabilizing block-wise optimization without central coordination. We establish convergence guarantees for DECA under stochastic gradients and non-IID data, and extensive experiments on classification and generation tasks demonstrate its effectiveness across diverse LLMs.

\bibliography{reference}
\bibliographystyle{plain}

\newpage
\appendix
\onecolumn
\setcounter{tocdepth}{2}
\tableofcontents
\newpage

\section*{Appendix}

\section{Proof of Theorem~\ref{thm:convergence}} \label{sec:convergenceproof}

  \subsection{Preliminary Lemmas}
    We first reformulate $x^{[t,r]}_{i,k}$ into a block-wise sequence and define $\bar{x}^{[t,r]}_k$ as:
    \begin{align}
    x^{[t,r]}_{i,k} &= \{ \theta_{i,1}^{[t+1,0]}, \ldots, \theta_{i,k-1}^{[t+1,0]},  \theta_{i,k}^{[t,r]}, \theta_{i,k+1}^{[t,0]}, \ldots, \theta_{i,B}^{[t,0]}\} \\ \nonumber
    \bar{x}^{[t,r]}_k &= \{ \bar{\theta}_{1}^{[t+1,0]}, \ldots, \bar{\theta}_{k-1}^{[t+1,0]}, \bar{\theta}_{k}^{[t,r]}, \bar{\theta}_{k+1}^{[t,0]}, \ldots, \bar{\theta}_{B}^{[t,0]} \}
    \end{align}
    
    Before diving into the proof details, we introduce several preliminary lemmas.
    %
    %The proof of \textbf{Lemma}~\ref{lem:s_sup} can be found in \cite{LuoYLi-NIPS24}, while the proofs of \textbf{Lemmas}~\ref{lem:dsm}--\ref{lem:hadamard} are provided in \cite{YuJY-ICML19, WangYWL-ICDCS25}.
    %
    \begin{lemma}[Uniform bound for the adaptive reciprocal]\label{lem:s_sup}
    For every client $i\in\mathcal N$, block $k\in\{1,\dots,B\}$, outer round $t$, and inner step $r$, define
    \[
     u_{i,k}^{[t,r]}=\frac{1}{\sqrt{\widehat v_{i,k}^{[t,r]}}+\epsilon},
    \]
    where the square root and the division are understood coordinatewise. Then
    \[
    \|u_{i,k}^{[t,r]}\|_{\infty}\le \frac{1}{\epsilon}.
    \]
    \end{lemma}
    \begin{proof}
    By \textbf{Algorithm}~\ref{alg:deca}, $v_{i,k}^{[t,0]}=0$. Moreover,
    \[
     v_{i,k}^{[t,r+\frac12]}
     =\alpha_2 v_{i,k}^{[t,r]}+(1-\alpha_2)g_{i,k}^{[t,r]}\odot g_{i,k}^{[t,r]},
    \]
    and
    \[
     v_{i,k}^{[t,r+1]}
     =\beta_2 v_{i,k}^{[t,r+\frac12]}+(1-\beta_2)h_{i,k}^{[t,r]}\odot h_{i,k}^{[t,r]}.
    \]
    Since both $g_{i,k}^{[t,r]}\odot g_{i,k}^{[t,r]}$ and $h_{i,k}^{[t,r]}\odot h_{i,k}^{[t,r]}$ are coordinatewise nonnegative, induction on $r$ yields that $v_{i,k}^{[t,r+\frac12]}$ is coordinatewise nonnegative. Hence the bias-corrected second moment
    \[
    \widehat v_{i,k}^{[t,r]}=\frac{v_{i,k}^{[t,r+\frac12]}}{1-\alpha_2^{r+1}}
    \]
    is also coordinatewise nonnegative. Therefore each coordinate of $\sqrt{\widehat v_{i,k}^{[t,r]}}+\epsilon$ is at least $\epsilon$, so each coordinate of $u_{i,k}^{[t,r]}$ is at most $1/\epsilon$. Taking the maximum over coordinates proves the claim.
    \end{proof}
    
    \begin{lemma}[Mixing identities and spectral contraction]\label{lem:dsm}
    Let
    \[
    \mathbf J=\frac1N\mathbf 1\mathbf 1^{\top}.
    \]
    Under \textbf{Assumption}~\ref{assum:dsm} and the symmetry of the DECA mixing matrix in the undirected communication model, the following identities hold:
    \[
    \mathbf J\mathbf W=\mathbf W\mathbf J=\mathbf J,
    \qquad
    (\mathbf I-\mathbf J)\mathbf W=\mathbf W(\mathbf I-\mathbf J)=\mathbf W-\mathbf J.
    \]
    Moreover, for every integer $q\ge 1$,
    \[
    \|(\mathbf I-\mathbf J)\mathbf W^q\|_{\mathfrak S}\le (\sqrt\rho)^q.
    \]
    \end{lemma}
    \begin{proof}
    By \textbf{Assumption}~\ref{assum:dsm}, $\mathbf W\mathbf 1=\mathbf 1$ and $\mathbf 1^{\top}\mathbf W=\mathbf 1^{\top}$. Hence
    \[
    \mathbf W\mathbf J
    =\mathbf W\Bigl(\frac1N\mathbf 1\mathbf 1^{\top}\Bigr)
    =\frac1N(\mathbf W\mathbf 1)\mathbf 1^{\top}
    =\frac1N\mathbf 1\mathbf 1^{\top}
    =\mathbf J,
    \]
    and similarly,
    \[
    \mathbf J\mathbf W
    =\frac1N\mathbf 1(\mathbf 1^{\top}\mathbf W)
    =\frac1N\mathbf 1\mathbf 1^{\top}
    =\mathbf J.
    \]
    Therefore,
    \[
    (\mathbf I-\mathbf J)\mathbf W=\mathbf W-\mathbf J=\mathbf W(\mathbf I-\mathbf J).
    \]
    Since the communication graph is undirected, $\mathbf W$ is symmetric. Hence, $\mathbf W$ admits an orthonormal eigenbasis. The vector $\mathbf 1$ is an eigenvector associated with eigenvalue $1$, while $\mathbf J$ is the orthogonal projector onto ${\rm span}\{\mathbf 1\}$. Consequently, $(\mathbf I-\mathbf J)\mathbf W^q$ vanishes on ${\rm span}\{\mathbf 1\}$ and has eigenvalues $\lambda_j(\mathbf W)^q$ on ${\rm span}\{\mathbf 1\}^{\perp}$. By \textbf{Assumption}~\ref{assum:dsm},
    \[
    |\lambda_j(\mathbf W)|\le \sqrt\rho,
    \qquad j=2,\dots,N.
    \]
    Hence
    \[
    \|(\mathbf I-\mathbf J)\mathbf W^q\|_{\mathfrak S}
    =\max_{2\le j\le N}|\lambda_j(\mathbf W)|^q
    \le (\sqrt\rho)^q.
    \]
    \end{proof}
    
    \begin{lemma}[Frobenius contraction under mixing]\label{lem:fnorm}
    Under \textbf{Lemma}~\ref{lem:dsm}, for every matrix $\mathbf X\in\mathbb R^{N\times d}$ and every integer $q\ge 1$,
    \[
    \|(\mathbf I-\mathbf J)\mathbf W^q\mathbf X\|_{\mathfrak F}^2
    \le \rho^q\|\mathbf X\|_{\mathfrak F}^2.
    \]
    \end{lemma}
    \begin{proof}
    For any compatible matrices $\mathbf A$ and $\mathbf X$,
    \[
    \|\mathbf A\mathbf X\|_{\mathfrak F}\le \|\mathbf A\|_{\mathfrak S}\|\mathbf X\|_{\mathfrak F}.
    \]
    Applying this inequality with $\mathbf A=(\mathbf I-\mathbf J)\mathbf W^q$ and using \textbf{Lemma}~\ref{lem:dsm}, we get
    \[
    \|(\mathbf I-\mathbf J)\mathbf W^q\mathbf X\|_{\mathfrak F}
    \le \|(\mathbf I-\mathbf J)\mathbf W^q\|_{\mathfrak S}\|\mathbf X\|_{\mathfrak F}
    \le (\sqrt\rho)^q\|\mathbf X\|_{\mathfrak F}.
    \]
    Squaring both sides yields the claim.
    \end{proof}
    
    \begin{lemma}[Hadamard-product inequality]\label{lem:hadamard}
    For any matrices $\mathbf A,\mathbf B\in\mathbb R^{N\times d}$,
    \[
    \|\mathbf A\odot\mathbf B\|_{\mathfrak F}
    \le \|\mathbf A\|_{\mathfrak F}\,\|\mathbf B\|_{\infty}.
    \]
    \end{lemma}
    \begin{proof}
    By the definitions of the Frobenius norm and the entrywise maximum norm,
    \[
    \|\mathbf A\odot\mathbf B\|_{\mathfrak F}^2
    =\sum_{i=1}^N\sum_{j=1}^d |A_{ij}B_{ij}|^2
    \le \sum_{i=1}^N\sum_{j=1}^d |A_{ij}|^2\,\|\mathbf B\|_{\infty}^2
    =\|\mathbf A\|_{\mathfrak F}^2\|\mathbf B\|_{\infty}^2.
    \]
    Taking square roots completes the proof.
    \end{proof}

  \subsection{Proof Details}
  %
    % \todo[inline]{Split into two subsections for Theorem 4.4 and Corollary 4.5, respectively.}

    % \subsubsection{Proof of \textbf{Theorem}~\ref{thm:convergence}}

    % \subsubsection{Proof of \textbf{Corollary}~\ref{cor:convergence}}
    
    %\subsubsection{Proof Sketch}

    The proof first introduces a virtual block-wise sequence $\hat{x}_{k}^{[t,r]}$ that follows a centralized SGD-style recursion with adaptive step scaling, and uses it as an analyzable proxy for the averaged decentralized iterate $\bar{x}_{k}^{[t,r]}$. For each active block, the proof applies smoothness to establish a one-step descent inequality for $\mathbb{E}[F(\hat{x}_{k}^{[t,r+1]})]-\mathbb{E}[F(\hat{x}_{k}^{[t,r]})]$, where the descent term is proportional to $\|\nabla_k F(\bar{x}_{k}^{[t,r]})\|^2$ and the remaining terms capture stochastic gradient noise, consensus error, and the sequence tracking error $e_{k}^{[t,r]}=\hat{x}_{k}^{[t,r]}-\bar{x}_{k}^{[t,r]}$. The consensus error is then bounded using the spectral contraction of the mixing matrix, together with the uniform bound on the adaptive reciprocal and the Hadamard-product inequality to handle Adam-style scaling. Finally, the sequence tracking error and the objective increase caused by transitions between consecutive blocks are controlled by the error-controlled block-transition assumption. Summing the resulting descent inequalities over all rounds, blocks, and inner steps yields the final average squared-gradient bound and, with $\gamma=O(1/\sqrt{TBR})$, DECA has the convergence rate of $O(1/\sqrt{TBR})$.

    \textbf{Non-convex Single-step Update}. 
    For all $r \geq 0$, we bridge the gap between $\hat{x}_k^{[t,r]}$ and $\bar{x}_k^{[t,r]}$ with consensus error $\nmsq{\bar{x}_{k}^{[t,r]} - x_{i,k}^{[t,r]}}$ and error sequence $e_k^{[t,r]}$:
    \begin{align}\label{eq:one_step_main}
      & \e{\mathcal{F}(\hat{x}_k^{[t,r+1]})} - \e{\mathcal{F}(\hat{x}_k^{[t,r]})} \nonumber \\
      =& \e{\ip{\nabla_k \mathcal{F}(\hat{x}_{k}^{[t,r]}), \hat{x}_k^{[t,r+1]} - \hat{x}_k^{[t,r]}}} + \frac{L}{2} \e{\nmsq{\hat{x}_k^{[t,r+1]} - \hat{x}_k^{[t,r]}}} \nonumber \\
      =& \e{\ip{\nabla_k \mathcal{F}(\hat{x}_{k}^{[t,r]}), \hat{\theta}_k^{[t,r+1]} - \hat{\theta}_k^{[t,r]}}} + \frac{L}{2} \e{\nmsq{\hat{\theta}_k^{[t,r+1]} - \hat{\theta}_k^{[t,r]}}} \nonumber \\
      =& - \frac{\gamma}{N(1 -\alpha_1^{r+1})} \sum_{i=1}^N \e{\ip{\nabla_k \mathcal{F}(\hat{x}_{k}^{[t,r]}), \nabla_k f_i(x_{i,k}^{[t,r]})}} + \frac{L\gamma^2}{2(1 -\alpha_1^{r+1})^2} \e{\nmsq{\bar{g}_k^{[t,r]}}} \nonumber \\
      \overset{\tcircle{1}}{\leq} & - \frac{\gamma}{N(1 -\alpha_1^{r+1})} \sum_{i=1}^N \e{\ip{\nabla_k \mathcal{F}(\hat{x}_{k}^{[t,r]}), \nabla_k f_i(x_{i,k}^{[t,r]})}} \nonumber \\
      & + \frac{L\gamma^2}{2(1 -\alpha_1^{r+1})^2} \e{\nmsq{\frac{1}{N}\sum_{i=1}^N \nabla_k f_i({x}_{i,k}^{[t,r]})}} + \frac{L\gamma^2\xi^2}{2N(1 -\alpha_1^{r+1})^2} \nonumber \\
      \overset{\tcircle{2}}{\leq} & - \frac{\gamma}{(1 -\alpha_1^{r+1})} \e{\ip{\nabla_k \mathcal{F}(\hat{x}_{k}^{[t,r]}), \frac{1}{N}\sum_{i=1}^N \nabla_k f_i(x_{i,k}^{[t,r]})}} \nonumber \\
      & + \frac{L\gamma^2}{2(1 -\alpha_1^{r+1})^2} \e{\nmsq{\frac{1}{N}\sum_{i=1}^N \nabla_k f_i({x}_{i,k}^{[t,r]})}} + \frac{L\gamma^2\xi^2}{2N(1 -\alpha_1^{r+1})^2} \nonumber \\
      \overset{\tcircle{3}}{\leq} & \left( \frac{L\gamma^2}{2(1 -\alpha_1^{r+1})^2} - \frac{\gamma}{2(1 -\alpha_1^{r+1})} \right) \e{\nmsq{\frac{1}{N}\sum_{i=1}^N \nabla_k f_i({x}_{i,k}^{[t,r]})}} \nonumber \\ 
      & + \frac{L\gamma^2\xi^2}{2N(1 -\alpha_1^{r+1})^2} - \frac{\gamma}{2(1 -\alpha_1^{r+1})} \e{\nmsq{\nabla_k \mathcal{F}(\hat{x}_{k}^{[t,r]})}} \nonumber \\
      & + \frac{\gamma}{2(1 -\alpha_1^{r+1})} \e{\nmsq{\nabla_k \mathcal{F}(\hat{x}_{k}^{[t,r]}) - \frac{1}{N}\sum_{i=1}^N \nabla_k f_i(x_{i,k}^{[t,r]})}} \nonumber \\
      \overset{\tcircle{4}}{\leq} & \left( \frac{L\gamma^2}{2(1 -\alpha_1^{r+1})^2} - \frac{\gamma}{2(1 -\alpha_1^{r+1})} \right) \e{\nmsq{\frac{1}{N}\sum_{i=1}^N \nabla_k f_i(x_{i,k}^{[t,r]})}} \nonumber \\
      & + \frac{L\gamma^2\xi^2}{2N(1 -\alpha_1^{r+1})^2} - \frac{\gamma}{4(1 -\alpha_1^{r+1})} \e{\nmsq{\nabla_k \mathcal{F}(\bar{x}_{k}^{[t,r]})}} \nonumber \\ 
      & + \frac{\gamma}{2(1 -\alpha_1^{r+1})} \e{\nmsq{\nabla_k \mathcal{F}(\hat{x}_{k}^{[t,r]}) - \nabla_k \mathcal{F}(\bar{x}_{k}^{[t,r]})}} \nonumber \\
      & + \frac{\gamma}{2(1 -\alpha_1^{r+1})} \e{\nmsq{\nabla_k \mathcal{F}(\hat{x}_{k}^{[t,r]}) - \frac{1}{N}\sum_{i=1}^N \nabla_k f_i(x_{i,k}^{[t,r]})}} \nonumber \\
      \leq & \left( \frac{L\gamma^2}{2(1 -\alpha_1^{r+1})^2} - \frac{\gamma}{2(1 -\alpha_1^{r+1})} \right) \e{\nmsq{\frac{1}{N}\sum_{i=1}^N \nabla_k f_i(x_{i,k}^{[t,r]})}} \nonumber \\
      & + \frac{L\gamma^2\xi^2}{2N(1 -\alpha_1^{r+1})^2} - \frac{\gamma}{4(1 -\alpha_1^{r+1})} \e{\nmsq{\nabla_k \mathcal{F}(\bar{x}_{k}^{[t,r]})}} \nonumber \\ 
      & + \frac{3\gamma}{2(1 -\alpha_1^{r+1})} \e{\nmsq{\nabla_k \mathcal{F}(\hat{x}_{k}^{[t,r]}) - \nabla_k \mathcal{F}(\bar{x}_{k}^{[t,r]})}} \nonumber \\
      & + \frac{\gamma}{(1 -\alpha_1^{r+1})} \e{\nmsq{\nabla_k \mathcal{F}(\bar{x}_{k}^{[t,r]}) - \frac{1}{N}\sum_{i=1}^N \nabla_k f_i(x_{i,k}^{[t,r]})}} \nonumber \\
      \leq & \left( \frac{L\gamma^2}{2(1 -\alpha_1^{r+1})^2} - \frac{\gamma}{2(1 -\alpha_1^{r+1})} \right) \e{\nmsq{\frac{1}{N}\sum_{i=1}^N \nabla_k f_i(x_{i,k}^{[t,r]})}} \nonumber \\
      & + \frac{L\gamma^2\xi^2}{2N(1 -\alpha_1^{r+1})^2} - \frac{\gamma}{4(1 -\alpha_1^{r+1})} \e{\nmsq{\nabla_k \mathcal{F}(\bar{x}_{k}^{[t,r]})}} \nonumber \\ & + \frac{3L^2\gamma}{2(1 -\alpha_1^{r+1})} \e{\nmsq{\hat{x}_{k}^{[t,r]} - \bar{x}_{k}^{[t,r]}}} \nonumber \\
      & + \frac{\gamma}{N(1 -\alpha_1^{r+1})} \sum_{i=1}^N \e{\nmsq{\nabla_k f_i(\bar{x}_{k}^{[t,r]}) - \nabla_k f_i(x_{i,k}^{[t,r]})}} \nonumber \\
      \overset{\tcircle{5}}{\leq} & \left( \frac{L\gamma^2}{2(1 -\alpha_1^{r+1})^2} - \frac{\gamma}{2(1 -\alpha_1^{r+1})} \right) \e{\nmsq{\frac{1}{N}\sum_{i=1}^N \nabla_k f_i(x_{i,k}^{[t,r]})}} \nonumber \\
      & + \frac{L\gamma^2\xi^2}{2N(1 -\alpha_1^{r+1})^2} - \frac{\gamma}{4(1 -\alpha_1^{r+1})} \e{\nmsq{\nabla_k \mathcal{F}(\bar{x}_{k}^{[t,r]})}} \nonumber \\
      & + \frac{3L^2\gamma}{2(1 -\alpha_1^{r+1})} \e{\nmsq{e_k^{[t,r]}}} + \frac{L^2\gamma}{N(1 -\alpha_1^{r+1})} \sum_{i=1}^N \e{\nmsq{\bar{x}_{k}^{[t,r]} - x_{i,k}^{[t,r]}}}
    \end{align}
    $\tcircle{1}$ comes from 
    \begin{align*}
      \mathbb{E}\left\|\frac{1}{N}\sum_{i=1}^{N}\nabla_k F_i(x_{i,k}^{[t,r]}, \xi_{i,b}^{[t,r]}) - \frac{1}{N}\sum_{i=1}^{N}\nabla_k f_i(x_{i,k}^{[t,r]})\right\|^2 \leq \mathbb{E}\left\|\frac{1}{N}\sum_{i=1}^{N}\nabla_k f_i(x_{i,k}^{[t,r]})\right\|^2 + \frac{\xi^2}{N} 
    \end{align*}
    with $f_i(x_{i,k}^{[t,r]})$ represeting an unbiased estimation on $F_i(x_{i,k}^{[t,r]}, \zeta)$; $\tcircle{2}$ comes from characteristic of inner-product; $\tcircle{3}$ refers to $2E\langle A,B\rangle = -A^2 - B^2 + (A-B)^2 $; $\tcircle{4}$ is derived from $\frac{1}{2} \mathbb{E} \|A - B\|^2 - \frac{1}{4} \mathbb{E} \|B\|^2 + \frac{1}{2} \mathbb{E} \|A\|^2 = \mathbb{E} \|A\|^2 - \mathbb{E}\langle A, B\rangle + \frac{1}{4} \mathbb{E} \|B\|^2 \geq 0$ where $A=\nabla_k \mathcal{F}(\bar{x}_{k}^{[t,r]}) + \nabla_k \mathcal{F}(\hat{x}_{k}^{[t,r]})$ and $B=\nabla_k \mathcal{F}(\bar{x}_{k}^{[t,r]})$; $\tcircle{5}$ satisfies when $\left( \frac{L\gamma^2}{2(1 -\alpha_1^{r+1})^2} - \frac{\gamma}{2(1 -\alpha_1^{r+1})} \right) \leq 0$.

    Rewrite $x_{k}^{[t,r+1]}$ and $m_{i}^{[t,r+1]}$ as
    \begin{align}
      x_{k}^{[t,r+1]} =& \sum_{j\in \mathcal{N}_i} w_{ij} x_{j,k}^{t,\tau+\frac{1}{2}} = \sum_{j\in \mathcal{N}_i} w_{ij} x_{j,k}^{[t,r]} - \gamma \sum_{j\in \mathcal{N}_i} w_{ij}  \dfrac{\hat{m}_{j}^{[t,r]}}{\sqrt{\hat{v}_{j}^{[t,r]}} + \epsilon} \nonumber \\
      =& \sum_{j\in \mathcal{N}_i} w_{ij} x_{j,k}^{[t,r]} - \frac{\gamma}{(1-\alpha_1^{r+1})} \sum_{j\in \mathcal{N}_i} w_{ij} \left(\alpha_1 m_{j}^{[t,r]} + (1-\alpha_1) g_{j,k}^{[t,r]} \right) \odot u_{j}^{[t,r]} \nonumber \\
      =& \sum_{j\in \mathcal{N}_i} w_{ij} x_{j,k}^{[t,r]} - \frac{\gamma \alpha_1}{(1-\alpha_1^{r+1})}  \sum_{j\in \mathcal{N}_i} w_{ij} \left(m_{j}^{[t,r]} \odot u_{j}^{[t,r]} \right) \nonumber \\
      & - \frac{\gamma (1-\alpha_1)}{(1-\alpha_1^{r+1})} \sum_{j\in \mathcal{N}_i} w_{ij} \left(g_{j,k}^{[t,r]} \odot u_{j}^{[t,r]}\right)
    \end{align}
    and
    \begin{align}
      m_{i}^{[t,r+1]} =& \beta_1 {m}_{i}^{[t,r+\frac{1}{2}]} + (1-\beta_1) {h}_{i,k}^{[t,r]} \nonumber \\
      =& \beta_1 \alpha_1 {m}_{i}^{[t,r]} + \beta_1(1-\alpha_1) g_{i,k}^{[t,r]}  + \frac{1-\beta_1}{\nm{{d}_{i,k}^{[t,r]}}} \left({x}_{i,k}^{[t,r]} - \sum_{j\in \mathcal{N}_i} w_{ij} {x}_{i,k}^{[t,r+\frac{1}{2}]} \right) \nonumber \\
      =& \beta_1\alpha_1 {m}_{i}^{[t,r]} + \beta_1(1-\alpha_1) g_{i,k}^{[t,r]} + \frac{1-\beta_1}{\nm{{d}_{i,k}^{[t,r]}}} \left( {x}_{i,k}^{[t,r]} - \sum_{j\in \mathcal{N}_i} w_{ij} {x}_{j,k}^{[t,r]}\right) \nonumber \\
      & + \frac{\gamma(1-\beta_1) \alpha_1}{(1-\alpha_1^{r+1})\nm{{d}_{i,k}^{[t,r]}}} \sum_{j\in \mathcal{N}_i} w_{ij} \left( m_{j}^{[t,r]} \odot u_{j}^{[t,r]} \right) \nonumber \\
      & + \frac{\gamma(1-\beta_1)(1-\alpha_1)}{(1-\alpha_1^{r+1}) \nm{{d}_{i,k}^{[t,r]}}} \sum_{j\in \mathcal{N}_i} w_{ij} \left( g_{j,k}^{[t,r]} \odot u_{j}^{[t,r]} \right)
    \end{align}
    where $\nm{{d}_{i,k}^{[t,r]}} = \nm{{x}_{i,k}^{[t,r]} - \sum_{j\in \mathcal{N}_i} w_{ij} {x}_{j,k}^{[t,r+\frac{1}{2}]}}$. Let $\mathbf{H}^{[t,r]} \overset{\Delta}{=} \mathbf{M}^{[t,r]} \odot \mathbf{U}^{[t,r]}$, $\mathbf{Q}_k^{[t,r]} \overset{\Delta}{=} \mathbf{G}_k^{[t,r]} \odot \mathbf{U}^{[t,r]}$, $\Delta_k^{[t,r]} = \mathrm{diag}\left[\nm{{x}_{i,k}^{[t,r]} - \sum_{j\in \mathcal{N}_i} w_{ij} {x}_{j,k}^{[t,r+\frac{1}{2}]}}^{-1}, \forall i\in\mathcal{N} \right]$, we have the matrix forms of ${x}_{k}^{[t,r+1]}$ and $m_{i}^{[t,r+1]}$:
    \begin{align} \label{eq:matrix_form_x}
      \mathbf{X}_{k}^{[t,r+1]} =& \mathbf{W} \mathbf{X}_{k}^{[t,r]} - \frac{\gamma \alpha_1}{(1-\alpha_1^{r+1})} \mathbf{W} \mathbf{H}^{[t,r]} - \frac{\gamma (1-\alpha_1)}{(1-\alpha_1^{r+1})} \mathbf{W} \mathbf{Q}_k^{[t,r]} 
    \end{align}
    and
    \begin{align} \label{eq:matrix_form_m}
      \mathbf{M}^{[t,r+1]} =&  \beta_1 {\mathbf{M}}^{[t,r+\frac{1}{2}]} + (1-\beta_1)\Delta_k^{[t,r]}\left(\mathbf{X}_{k}^{[t,r]} -\mathbf{X}_{k}^{[t,r+\frac{1}{2}]} \right)\nonumber \\ 
      =& \beta_1\alpha_1 {\mathbf{M}}^{[t,r]} + \beta_1(1-\alpha_1) \mathbf{G}_k^{[t,r]} + (1-\beta_1)\Delta_k^{[t,r]} \left(\mathbf{X}_{k}^{[t,r]} -\mathbf{X}_{k}^{[t,r+\frac{1}{2}]} \right)\nonumber \\
      =& \beta_1\alpha_1 {\mathbf{M}}^{[t,r]} + \beta_1(1-\alpha_1) \mathbf{G}_k^{[t,r]} + (1-\beta_1) \left(\mathbf{I} - \mathbf{W} \right) \Delta_k^{[t,r]} \mathbf{X}_{k}^{[t,r]} \nonumber \\
      & + \frac{\gamma\alpha_1 (1-\beta_1)}{(1-\alpha_1^{r+1})} \mathbf{W} \Delta_k^{[t,r]} \mathbf{H}^{[t,r]} + \frac{\gamma(1-\beta_1)(1-\alpha_1)}{(1-\alpha_1^{r+1})} \mathbf{W} \Delta_k^{[t,r]} \mathbf{Q}_k^{[t,r]} 
    \end{align}

    \textbf{Consensus Error}. Let $\mathbf{J} = \frac{1}{N}\mathbf{1}_N \mathbf{1}_N^\top $. We bound the Frobenius norm of $\mathbf{X}_k^{[t,r+1]} -\bar{\mathbf{X}}_k^{[t,r+1]}$:
    \begin{align}\label{eq:x_iter}
      & \frac{1}{N} \mathbb{E} \left\| \mathbf{X}_k^{[t,r+1]} -\bar{\mathbf{X}}_k^{[t,r+1]} \right\|_\mathfrak{F}^{2}  \nonumber \\
      = & \frac{1}{N}\mathbb{E} \left\|
        \begin{aligned}
        &\mathbf{W} \left(\mathbf{X}_k^{[t,r]} -\bar{\mathbf{X}}_k^{[t,r]} \right) - \left(\mathbf{W} - \mathbf{W}\mathbf{J}\right) \frac{\gamma \alpha_1}{(1-\alpha_1^{r+1})} \mathbf{H}^{[t,r]} \\
        & - \left(\mathbf{W} - \mathbf{W}\mathbf{J}\right) \frac{\gamma (1-\alpha_1)}{(1-\alpha_1^{r+1})} \mathbf{Q}_k^{[t,r]}
        \end{aligned}
      \right\|_\mathfrak{F}^{2} \nonumber \\
      \overset{\tcircle{2}}{\leq} & \frac{1+\rho}{2N} \mathbb{E} \left\| \mathbf{X}_k^{[t,r]} -\bar{\mathbf{X}}_k^{[t,r]} \right\|_\mathfrak{F}^{2} + \frac{2(1+\rho)}{N(1-\rho)} \mathbb{E} \left\| \left(\mathbf{I} - \mathbf{J}\right) \mathbf{W} \frac{\gamma \alpha_1}{(1-\alpha_1^{r+1})} \mathbf{H}^{[t,r]} \right\|_\mathfrak{F}^{2} \nonumber \\
      & + \frac{2(1+\rho)}{N(1-\rho)} \mathbb{E} \left\| \left(\mathbf{I} - \mathbf{J}\right) \mathbf{W} \frac{\gamma (1-\alpha_1)}{(1-\alpha_1^{r+1})} \mathbf{Q}_k^{[t,r]} \right\|_\mathfrak{F}^{2} \nonumber \\
      \overset{\tcircle{3}}{\leq} & \frac{1+\rho}{2N} \mathbb{E} \left\| \mathbf{X}_k^{[t,r]} -\bar{\mathbf{X}}_k^{[t,r]} \right\|_\mathfrak{F}^{2} + \frac{2(1+\rho)\rho\gamma^2 \alpha_1^2}{N(1-\rho)(1-\alpha_1^{r+1})^2} \mathbb{E} \left\| \mathbf{H}^{[t,r]} \right\|_\mathfrak{F}^{2} \nonumber \\
      & + \frac{2(1+\rho)\rho\gamma^2 (1-\alpha_1)^2}{N(1-\rho)(1-\alpha_1^{r+1})^2} \mathbb{E} \left\| \mathbf{Q}_k^{[t,r]} \right\|_\mathfrak{F}^{2} \nonumber \\
      \overset{\tcircle{4}}{\leq} & \frac{1+\rho}{2N} \mathbb{E} \left\| \mathbf{X}_k^{[t,r]} -\bar{\mathbf{X}}_k^{[t,r]} \right\|_\mathfrak{F}^{2} + \frac{2(1+\rho)\rho\gamma^2 \alpha_1^2}{N(1-\rho)(1-\alpha_1^{r+1})^2\epsilon^2} \mathbb{E} \left\| \mathbf{M}^{[t,r]} \right\|_\mathfrak{F}^{2}  \nonumber \\ 
      & + \frac{2(1+\rho)\delta^2\rho\gamma^2 (1-\alpha_1)^2}{(1-\rho)(1-\alpha_1^{r+1})^2 \epsilon^2} \nonumber \\ 
      \overset{\tcircle{5}}{\leq} & \frac{1+\rho}{2N} \mathbb{E} \left\| \mathbf{X}_k^{[t,r]} -\bar{\mathbf{X}}_k^{[t,r]} \right\|_\mathfrak{F}^{2} + \frac{4(1+\rho)\rho\gamma^2 \alpha_1^2}{N(1-\rho)(1-\alpha_1^{r+1})^2\epsilon^2} \mathbb{E} \left\| \mathbf{M}^{[t,r]} - \bar{\mathbf{M}}^{[t,r]} \right\|_\mathfrak{F}^{2} \nonumber \\
      & + \frac{4(1+\rho)\rho\gamma^2 \alpha_1^2 + 2(1+\rho)\delta^2\rho\gamma^2 (1-\alpha_1)^2}{(1-\rho)(1-\alpha_1^{r+1})^2\epsilon^2} 
    \end{align}
    where we have $\tcircle{1}$ using Young’s inequality, $\tcircle{2}$ and $\tcircle{3}$ according to \textbf{Lemma}~\ref{lem:fnorm}, $\tcircle{4}$ from \textbf{Lemma}~\ref{lem:s_sup} and \textbf{Lemma}~\ref{lem:hadamard}, and $\tcircle{5}$ since $\nm{h_{i,k}^{[t,r]}} \leq 1$ and thus $\nm{m_{i,k}^{[t,r+1]}}=\nm{\beta_1m_{i,k}^{[t,r+\frac{1}{2}]} + (1-\beta_1)h_{i,k}^{[t,r]}} \leq \beta_1\nm{m_{i,k}^{[t,r]}} + \beta_1(1-\beta_1) \nm{g_{i,k}^{[t,r]}} + (1-\beta_1)\nm{h_{i,k}^{[t,r]}} \leq \frac{\beta_1\delta+1}{1+\beta_1}, \; \forall i,k,t,r$.

    \textbf{Momentum Error}. Using the same way to bound the consensus error, we have:
    \begin{align}\label{eq:m_iter}
      & \frac{1}{N}\mathbb{E} \left\| \mathbf{M}_k^{[t,r+1]} -\bar{\mathbf{M}}_k^{[t,r+1]} \right\|_{F}^{2} \nonumber \\
      = & \frac{1}{N}\mathbb{E} \left\|
        \begin{aligned}
        &\beta_1\alpha_1 \left({\mathbf{M}}^{[t,r]} - \bar{\mathbf{M}}^{[t,r]} \right) 
        + \beta_1(1-\alpha_1) \left( \mathbf{G}_k^{[t,r]} - \bar{\mathbf{G}}_k^{[t,r]} \right) \\
        & + (1-\beta_1) \left(\mathbf{I} - \mathbf{J}\right) \mathbf{W} 
        \left[ \Delta_k^{[t,r]}\left({\mathbf{X}}_{k}^{[t,r]} -{\mathbf{X}}_{k}^{[t,r+\frac{1}{2}]} \right)\right]
        \end{aligned}
        \right\|_{F}^{2} \nonumber \\
      \overset{\tcircle{1}}{\leq} & \frac{1+\lambda}{N} \beta_1^2\alpha_1^2 \mathbb{E} \left\| {\mathbf{M}}^{[t,r]} - \bar{\mathbf{M}}^{[t,r]}\right\|_{F}^{2} + \frac{2(1+\frac{1}{\lambda})\beta_1^2(1-\alpha_1)^2}{N} \mathbb{E} \left\| {\mathbf{G}_k}^{[t,r]} - \bar{\mathbf{G}}_k^{[t,r]} \right\|_{F}^{2} \nonumber \\
      & + \frac{2(1+\frac{1}{\lambda})\rho(1-\beta_1)^2 }{N} \mathbb{E} \left\| \Delta_k^{[t,r]}\left({\mathbf{X}}_{k}^{[t,r]} -{\mathbf{X}}_{k}^{[t,r+\frac{1}{2}]} \right) \right\|_{F}^{2} \nonumber \\
      \leq & \frac{1+\lambda}{N} \beta_1^2 \alpha_1^2 \mathbb{E} \left\| {\mathbf{M}}^{[t,r]} - \bar{\mathbf{M}}^{[t,r]} \right\|_{F}^{2} + 4\left(1+\frac{1}{\lambda}\right) \beta_1^2(1-\alpha_1)^2\delta^2 + 2\left(1+\frac{1}{\lambda}\right)\rho(1-\beta_1)^2 \nonumber \\
      \overset{\tcircle{2}}{\leq} & \frac{1}{1-(1+\lambda)\beta_1^2\alpha_1^2} \left(4\left(1+\frac{1}{\lambda}\right) \beta_1^2(1-\alpha_1)^2\delta^2 + 2\left(1+\frac{1}{\lambda}\right)\rho(1-\beta_1)^2 \right)
    \end{align}
    where $\tcircle{1}$ comes from Young's inequality $\|A+B\|^2 \leq (1+\lambda)\|A\|^2 + (1+\frac{1}{\lambda}) \|B\|^2$ where $0< \lambda < \frac{1-\alpha_1^2\beta_1^2}{\alpha_1^2\beta_1^2}$; $\tcircle{2}$ holds for $\sum_{z} q^z < \frac{1}{1-q}, \forall q < 1$ and ${\mathbf{M}}^{[t,0]} - \bar{\mathbf{M}}^{[t,0]} = 0$. Substituting Eq.~\eqref{eq:m_iter} into Eq.~\eqref{eq:x_iter} and letting $C_2 = 4(1+\rho)\rho\alpha_1^2 + 2(1+\rho)\delta^2\rho (1-\alpha_1)^2$ and $C_3 = 4(1+\frac{1}{\lambda})\beta_1^2(1-\alpha_1)^2\delta^2 + 2\left(1+\frac{1}{\lambda}\right)\rho(1-\beta_1)^2 $, we have:
    \begin{align}\label{eq:x_iter_final}
      \frac{1}{N} \mathbb{E} \left\| \mathbf{X}_k^{[t,r+1]} -\bar{\mathbf{X}}_k^{[t,r+1]} \right\|_\mathfrak{F}^{2} 
      \leq & \frac{1+\rho}{2N} \mathbb{E} \left\| \mathbf{X}_k^{[t,r]} -\bar{\mathbf{X}}_k^{[t,r]} \right\|_\mathfrak{F}^{2}  + \frac{C_2\gamma^2}{(1-\rho)(1-\alpha_1^{r+1})^2\epsilon^2}  \nonumber \\
      & + \frac{4(1+\rho)\rho\gamma^2 \alpha_1^2}{N(1-\rho)(1-\alpha_1^{r+1})^2 \epsilon^2} \mathbb{E} \left\| \mathbf{M}^{[t,r]} - \bar{\mathbf{M}}^{[t,r]} \right\|_\mathfrak{F}^{2} \nonumber \\
      \leq & \frac{1+\rho}{2N} \mathbb{E} \left\| \mathbf{X}_k^{[t,r]} -\bar{\mathbf{X}}_k^{[t,r]} \right\|_\mathfrak{F}^{2}  + \frac{C_2\gamma^2}{(1-\rho)(1-\alpha_1^{r+1})^2\epsilon^2} \nonumber \\
      & + \frac{4C_3(1+\rho)\rho \gamma^2 \alpha_1^2}{(1-\rho)(1-\alpha_1^{r+1})^2 \epsilon^2 \left(1-(1+\lambda)\beta_1^2\alpha_1^2\right)} \nonumber \\
      \leq & \frac{1+\rho}{2N} \mathbb{E} \left\| \mathbf{X}_k^{[t,r]} -\bar{\mathbf{X}}_k^{[t,r]} \right\|_\mathfrak{F}^{2} + C_4 \gamma^2 \leq \frac{2C_4\gamma^2}{1-\rho}
    \end{align}
    where $C_4 = \frac{4C_3(1+\rho)\rho \alpha_1^2 + C_2\left(1-(1+\lambda)\beta_1^2\alpha_1^2 \right)}{(1-\rho)(1-\alpha_1)^2 \epsilon^2 \left(1-(1+\lambda) \beta_1^2 \alpha_1^2\right)}$  and the last inequality comes from $\mathbf{X}_k^{[t,0]} -\bar{\mathbf{X}}_k^{[t,0]} = \mathbf{X}_{k-1}^{[t,R]} -\bar{\mathbf{X}}_{k-1}^{[t,R]} $ and $\mathbf{X}_1^{[0,0]} = \bar{\mathbf{X}}_1^{[0,0]}$.

    \textbf{Sequence Tracking Error}. Next we show the error boundary of $e_k^{[t,r]}$. According to the definition of the error term (see Eq.~\eqref{eq:error_df}), we have
    \begin{align}
      e_k^{[t,r+1]} =& \hat{x}_k^{[t,r+1]} -  \bar{x}_k^{[t,r+1]} \nonumber \\
      =& e_k^{[t,r]} + \gamma \frac{1}{1-\alpha_1^{r+1}} \left[\frac{1}{N}\sum_{i=1}^N \sum_{j\in \mathcal{N}_i} w_{ij}\left(\alpha_1 m_{j}^{[t,r]} + (1-\alpha_1) g_{j,k}^{[t,r]} \right)\odot u_{j}^{[t,r]}-\bar{g}_k^{[t,r]}\right] \nonumber \\
      =& e_k^{[t,0]} +\sum_{\ell=0}^{r}\frac{\gamma}{1-\alpha_1^{\ell+1}} \left[\frac{1}{N}\sum_{i=1}^N \sum_{j\in \mathcal{N}_i} w_{ij}\left(\alpha_1 m_{j}^{[t,\ell]} + (1-\alpha_1) g_{j,k}^{[t,\ell]} \right)\odot u_{j}^{[t,\ell]}-\bar{g}_k^{[t,\ell]}\right]
    \end{align}
    Under \textbf{Assumption~\ref{assum:error_control}}, we have $e_k^{[t,0]} = \hat{x}_k^{[t,0]} - \bar{x}_k^{[t,0]} = 0, \forall {t, k} $. Hence, by taking the expectation at both sides, we obtain
    \begin{align}\label{eq:ee}
      & \e{\nmsq{e_k^{[t,r+1]}}} \nonumber \\
      \leq & r\gamma^2 \sum_{\ell=0}^{r}\frac{1}{(1-\alpha_1^{\ell+1})^2} \e{\nmsq{\frac{1}{N}\sum_{i=1}^N \sum_{j\in \mathcal{N}_i} w_{ij}\left(\alpha_1 m_{j}^{[t,\ell]} + (1-\alpha_1) g_{j,k}^{[t,\ell]} \right)\odot u_{j}^{[t,\ell]}-\bar{g}_k^{[t,\ell]}}}
      \nonumber \\
      \leq & 2r\gamma^2\sum_{\ell=0}^{r}\frac{1}{(1-\alpha_1^{\ell+1})^2} \frac{\alpha_1^2 + \delta^2(1-\alpha_1)^2}{\epsilon^2} + 2r\gamma^2 \sum_{\ell=0}^{r}\frac{1}{(1-\alpha_1^{\ell+1})^2} \e{\nmsq{\bar{g}_k^{[t,\ell]}}} \nonumber \\
      \leq & 2r\frac{\gamma^2\alpha_1^2 + \gamma^2\delta^2(1-\alpha_1)^2}{(1-\alpha_1)^2 \epsilon^2} + \frac{3r\gamma^2}{(1-\alpha_1)^2} \sum_{\ell=0}^{r} \e{\nmsq{\frac{1}{N}\sum_{i=1}^N \nabla_k f_i\left(x_{i,k}^{[t,\ell]} \right)}}
    \end{align}
    Finally, we substitute Eqs.~\eqref{eq:x_iter}, \eqref{eq:m_iter}, and \eqref{eq:ee} into Eq.~\eqref{eq:one_step_main} and obtain
    \begin{align}\label{eq:main_eq_single}
      & \e{\mathcal{F}(\hat{x}_k^{[t,r+1]})} - \e{\mathcal{F}(\hat{x}_k^{[t,r]})} \nonumber \\
      \leq & \left( \frac{L\gamma^2}{2(1 -\alpha_1^{r+1})^2} - \frac{\gamma}{2(1 -\alpha_1^{r+1})} \right) \e{\nmsq{\frac{1}{N}\sum_{i=1}^N \nabla_k f_i(x_{i,k}^{[t,r]})}} \nonumber \\
      & - \frac{\gamma}{4(1 -\alpha_1^{r+1})} \e{\nmsq{\nabla_k \mathcal{F}(\bar{x}_{k}^{[t,r]})}} + \frac{L\gamma^2\xi^2}{2N(1 -\alpha_1^{r+1})^2}  \nonumber \\
      & + \frac{3L^2\gamma}{2(1 -\alpha_1^{r+1})} \e{\nmsq{e_k^{[t,r]}}} + \frac{L^2\gamma}{N(1 -\alpha_1^{r+1})} \sum_{i=1}^N \e{\nmsq{\bar{x}_{k}^{[t,r]} - x_{i,k}^{[t,r]}}} \nonumber \\
      \leq & \left( \frac{L\gamma^2}{2(1 -\alpha_1)^2} - \frac{\gamma}{2(1 -\alpha_1)} \right) \e{\nmsq{\frac{1}{N}\sum_{i=1}^N \nabla_k f_i(x_{i,k}^{[t,r]})}} - \frac{\gamma}{4(1 -\alpha_1^{r+1})} \e{\nmsq{\nabla_k \mathcal{F}(\bar{x}_{k}^{[t,r]})}} \nonumber \\
      & + \frac{L\gamma^2\xi^2}{2N(1 -\alpha_1)^2} + \frac{3L^2\gamma}{2(1 -\alpha_1)} \e{\nmsq{e_k^{[t,r]}}} + \frac{L^2\gamma}{(1 -\alpha_1)}\e{\nmsq{\bar{\mathbf{X}}_{k}^{[t,r]} - \mathbf{X}_{i,k}^{[t,r]}}} \nonumber \\
      \leq & \left( \frac{L\gamma^2}{2(1 -\alpha_1)^2} - \frac{\gamma}{2(1 -\alpha_1)} \right) \e{\nmsq{\frac{1}{N}\sum_{i=1}^N \nabla_k f_i(x_{i,k}^{[t,r]})}} - \frac{\gamma}{4(1 -\alpha_1^{r+1})} \e{\nmsq{\nabla_k \mathcal{F}(\bar{x}_{k}^{[t,r]})}} \nonumber \\
      & + \frac{L\gamma^2\xi^2}{2N(1 -\alpha_1)^2} + \frac{3L^2\gamma}{2(1 -\alpha_1)} \e{\nmsq{e_k^{[t,r]}}} + \frac{2L^2C_4\gamma^3}{(1-\rho)(1-\alpha_1)}
    \end{align}

    \textbf{Descent Iteration}. By summing over $r=0\to R-1$, we get
    \begin{align}\label{eq:main_eq_r}
      & \e{\mathcal{F}(\hat{x}_{k+1}^{[t,0]})} - \e{\mathcal{F}(\hat{x}_k^{[t,0]})} \nonumber \\
      = & \e{\mathcal{F}(\hat{x}_{k+1}^{[t,0]})} - \e{\mathcal{F}(\hat{x}_k^{[t,R]})} + \e{\mathcal{F}(\hat{x}_k^{[t,R]})} - \e{\mathcal{F}(\hat{x}_k^{[t,R-1]})} \nonumber \\ 
      & + \ldots + \e{\mathcal{F}(\hat{x}_k^{[t,2]})} - \e{\mathcal{F}(\hat{x}_k^{[t,1]})} + \e{\mathcal{F}(\hat{x}_k^{[t,1]})} - \e{\mathcal{F}(\hat{x}_k^{[t,0]})} \nonumber \\
      \leq & C_{err}\e{\nmsq{e_k^{[t,R]}}} + \frac{2RL^2C_4\gamma^3}{(1-\rho)(1-\alpha_1)} + \frac{RL\gamma^2\xi^2}{2N(1 -\alpha_1)^2} \nonumber\\
      & + \sum_{r=0}^{R-1} \left( \frac{L\gamma^2}{2(1 -\alpha_1)^2} - \frac{\gamma}{2(1 -\alpha_1)} \right) \e{\nmsq{\frac{1}{N}\sum_{i=1}^N \nabla_k f_i(x_{i,k}^{[t,r]})}}  \nonumber \\
      & - \sum_{r=0}^{R-1}\frac{\gamma}{4(1 -\alpha_1^{r+1})} \e{\nmsq{\nabla_k \mathcal{F}(\bar{x}_{k}^{[t,r]})}} + \sum_{r=0}^{R-1} \frac{3L^2\gamma}{2(1 -\alpha_1)} \e{\nmsq{e_k^{[t,r]}}} 
    \end{align}
    Furthermore, by summing over a full round block switch, we have
    \begin{align}\label{eq:main_eq_rb}
      & \e{\mathcal{F}(\hat{x}_1^{[t+1,0]})} - \e{\mathcal{F}(\hat{x}_1^{[t,0]})} \nonumber \\
      =& \e{\mathcal{F}(\hat{x}_1^{[t+1,0]})} - \e{\mathcal{F}(\hat{x}_{B}^{[t,0]})} + \ldots + \e{\mathcal{F}(\hat{x}_2^{[t,0]})} - \e{\mathcal{F}(\hat{x}_1^{[t,0]})} \nonumber \\
      \leq & \left( \frac{L\gamma^2}{2(1 -\alpha_1)^2} - \frac{\gamma}{2(1 -\alpha_1)} \right) \sum_{k=1}^{B} \sum_{r=0}^{R-1}  \e{\nmsq{\frac{1}{N}\sum_{i=1}^N \nabla_k f_i(x_{i,k}^{[t,r]})}} + \frac{RBL\gamma^2\xi^2}{2N(1 -\alpha_1)^2} \nonumber \\
      & + \sum_{k=1}^{B} \sum_{r=0}^{R-1} \frac{3L^2\gamma}{2(1 -\alpha_1)} \e{\nmsq{e_k^{[t,r]}}} +  C_{err}\sum_{k=1}^{B} \e{\nmsq{e_k^{[t,R]}}}  \nonumber \\
      & - \sum_{k=1}^{B} \sum_{r=0}^{R-1}  \frac{\gamma}{4(1 -\alpha_1^{r+1})} \e{\nmsq{\nabla_k \mathcal{F}(\bar{x}_{k}^{[t,r]})}} + \frac{2RBL^2C_4\gamma^3}{(1-\rho)(1-\alpha_1)}
    \end{align}
    Moreover, summing from $t=0 \to T-1$ and letting $\mathcal{F}^{\star}=\e{\mathcal{F}(\hat{x}_1^{[T,0]})}$ and $\e{\mathcal{F}(\hat{x}^{0})} = \e{\mathcal{F}(\hat{x}_1^{[0,0]})}$, we obtain
    \begin{align}\label{eq:main_eq_rbt}
       \mathcal{F}^{\star} - \e{\mathcal{F}(\hat{x}^{0})}
      = & \e{\mathcal{F}(\hat{x}_1^{[T,0]})} - \e{\mathcal{F}(\hat{x}_1^{[T-1,0]})} + \ldots +  \e{\mathcal{F}(\hat{x}_1^{[1,0]})} - \e{\mathcal{F}(\hat{x}_1^{[0,0]})} \nonumber \\
      % 
      % = & \sum_{k=1}^{B} \left(\e{\mathcal{F}(\hat{x}_k^{[T-1,R]})} - \e{\mathcal{F}(\hat{x}_k^{[0,0]})} \right) \nonumber \\
       \leq & \left( \frac{L\gamma^2}{2(1 -\alpha_1)^2} - \frac{\gamma}{2(1 -\alpha_1)} \right) \sum_{t=0}^{T-1} \sum_{k=1}^{B} \sum_{r=0}^{R-1}  \e{\nmsq{\frac{1}{N}\sum_{i=1}^N \nabla_k f_i(x_{i,k}^{[t,r]})}}  \nonumber \\
       & + \frac{TBRL\gamma^2\xi^2}{2N(1 -\alpha_1)^2} + \sum_{t=0}^{T-1} \sum_{k=1}^{B} \sum_{r=0}^{R-1} \frac{3L^2\gamma}{2(1 -\alpha_1)} \e{\nmsq{e_k^{[t,r]}}} \nonumber \\
       & + \frac{2TBRL^2C_4\gamma^3}{(1-\rho)(1-\alpha_1)}  + C_{err} \sum_{t=0}^{T-1} \sum_{k=1}^{B} \e{\nmsq{e_k^{[t,R]}}} \nonumber \\
       & - \sum_{t=0}^{T-1} \sum_{k=1}^{B} \sum_{r=0}^{R-1} \frac{\gamma}{4(1 -\alpha_1^{R})} \e{\nmsq{\nabla_k \mathcal{F}(\bar{x}_{k}^{[t,r]})}}
    \end{align}
    % 
    % Let $C_5=\frac{\alpha_1^2 + \delta^2(1-\alpha_1)^2}{(1-\alpha_1)^2 \epsilon^2}$. 
    By substituting Eq.~\eqref{eq:ee} into Eq.~\eqref{eq:main_eq_rbt},  we obtain:
    \begin{align}
       & \mathcal{F}^{\star} - \e{\mathcal{F}(\hat{x}^{0})} \nonumber \\
       \leq & \left( \frac{L\gamma^2}{2(1 -\alpha_1)^2} - \frac{\gamma}{2(1 -\alpha_1)} \right) \sum_{t=0}^{T-1} \sum_{k=1}^{B} \sum_{r=0}^{R-1}  \e{\nmsq{\frac{1}{N}\sum_{i=1}^N \nabla_k f_i(x_{i,k}^{[t,r]})}}  \nonumber \\
       & + \frac{3R\gamma^2}{(1-\alpha_1)^2} \left(\frac{3L^2\gamma R}{2(1-\alpha_1)}+ C_{err} \right) \sum_{t=0}^{T-1} \sum_{k=1}^{B} \sum_{r=0}^{R-1}  \e{\nmsq{\frac{1}{N}\sum_{i=1}^N \nabla_k f_i(x_{i,k}^{[t,r]})}} \nonumber \\
       & + \frac{TBRL\gamma^2\xi^2}{2N(1 -\alpha_1)^2} + + \frac{2C_4TBRL^2\gamma^3}{(1-\rho)(1-\alpha_1)} + \frac{3C_1TBR^2 L^2\gamma^3}{2(1-\alpha_1)} + C_1C_{err}TBR\gamma^2  \nonumber \\
       & - \sum_{t=0}^{T-1} \sum_{k=1}^{B} \sum_{r=0}^{R-1} \frac{\gamma}{4(1 -\alpha_1^{R})}  \e{\nmsq{\nabla_k \mathcal{F}(\bar{x}_{k}^{[t,r]})}}
    \end{align}

    To eliminate the term $\mathbb{E}\!\left[\left\|\frac{1}{N}\sum_{i=1}^N \nabla_k f_i\!\left(x_{i,k}^{[t,r]}\right)\right\|^2\right]$, it suffices to choose the learning rate $\gamma$ such that its coefficient is non-positive. This requires
    \begin{align}
      \frac{L\gamma^2}{2(1 -\alpha_1)^2} - \frac{\gamma}{2(1 -\alpha_1)} + \frac{3R\gamma^2}{(1-\alpha_1)^2} \left(\frac{3L^2\gamma R}{2(1-\alpha_1)}+ C_{err} \right) \leq 0
    \end{align}
    Hence, we have:
    \begin{align}
      \gamma \leq \frac{2(1-\alpha_1)}{L+6RC_{err}+\sqrt{(L+6RC_{err})^2 + 36L^2R^2}}
    \end{align}
    Dividing by $T$, $B$ and $R$ and reordering the inequality yields:
    \begin{align}
       & \frac{1}{TBR} \sum_{t=0}^{T-1} \sum_{k=1}^{B} \sum_{r=0}^{R-1} \e{\nmsq{\nabla_k \mathcal{F}(\bar{x}_{k}^{t,r})}} \nonumber \\
       \leq & \frac{8(1 -\alpha_1^{R})}{TBR\gamma}\left( \e{\mathcal{F}(\hat{x}^{0}) - \mathcal{F}^{\star}} \right) + \frac{4(1 -\alpha_1^{R})L \gamma\xi^2}{N(1 -\alpha_1)^2} \nonumber \\
       & + \frac{16(1 -\alpha_1^{R})L^2C_4\gamma^2}{(1-\rho)(1-\alpha_1)} + \frac{12(1 -\alpha_1^{R})RC_1 L^2\gamma^2}{(1-\alpha_1)}  + 8C_1C_{err}(1 -\alpha_1^{R})\gamma
    \end{align}
    Let $\gamma = \mathcal{O}\left( \frac{1}{\sqrt{TBR}} \right)$, and $\gamma$ satisfies
    \begin{align}
      0< \gamma \leq \min\left\{ \frac{(1-\beta_1)}{3L\sqrt{TBR}}, \frac{2(1-\beta_1)}{L+6RC_{err} +\sqrt{(L+6RC_{err})^2 + 36L^2R^2}} \right\}
    \end{align}
    When
    %
    % \begin{align}
    %     \frac{1}{\sqrt{TBR}} & \geq \frac{(1-\beta_1)}{3L\sqrt{TBR}} \nonumber \\
    %     \frac{1}{\sqrt{TBR}} & \geq \frac{2(1-\beta_1)}{L+6RC_{err} +\sqrt{(L+6RC_{err})^2 + 36L^2R^2}}
    % \end{align}
    %
    \begin{align}
    %\begin{cases}
      \frac{1}{\sqrt{TBR}} &\geq \max \left\{ \frac{(1-\beta_1)}{3L\sqrt{TBR}}, \frac{2(1-\beta_1)}{L+6RC_{err} +\sqrt{(L+6RC_{err})^2 + 36L^2R^2}} \right\}
    %\end{cases}
    \end{align}
    such that 
    \begin{align}
      T \geq \max \left\{ \frac{(1-\beta_1)^2}{9L^2 B R},~\frac{ \mathcal{O}\left(L^2 + R^2C_{err}^2 + L^2R^2\right)}{L^2BR} \right\}
    \end{align}
    %
    % \begin{align}
    % &\left\{
    % \begin{aligned}
    % \frac{1}{\sqrt{TBR}} &\geq \frac{(1-\beta_1)}{3L\sqrt{TBR}} \\
    % \frac{1}{\sqrt{TBR}} &\geq \frac{2(1-\beta_1)}{L+6RC_{err} +\sqrt{(L+6RC_{err})^2 + 36L^2R^2}}
    % \end{aligned}
    % \right.
    % \\
    % & \Longrightarrow\ 
    % T \geq \max\left\{
    % \frac{(1-\beta_1)^2}{9L^2 B R},\ 
    % \frac{ \mathcal{O}\left(L^2 + R^2C_{err}^2 + L^2R^2\right)}{L^2BR}
    % \right\}
    % \end{align}
    %fa
    we obtain
    \begin{align}
       \frac{1}{TBR} \sum_{t=0}^{T-1} \sum_{k=1}^{B} \sum_{r=0}^{R-1} \e{\nmsq{\nabla_k \mathcal{F}(\bar{x}_{k}^{t,r})}} = \mathcal{O}\left( \frac{1}{\sqrt{TBR}} \right)
    \end{align}

    \begin{remark}
      Compared with QGM~\cite{LinKSJ-ICML21}, which proves convergence for decentralized SGD with first-order momentum, and BAdam~\cite{LuoYLi-NIPS24}, which analyzes centralized block-wise Adam only in the deterministic case, DECA establishes convergence for decentralized block-wise Adam under stochastic gradients and non-IID data. The proof simultaneously controls stochastic gradient noise, peer-to-peer consensus error, client drift, stale block-wise moment estimates, and the objective perturbation induced by transitions between consecutive blocks. By introducing a virtual averaged block sequence and bounding its deviation from decentralized client iterates, we prove that DECA achieves an average squared-gradient convergence rate of $O(1/\sqrt{TBR})$, thereby preserving the standard stochastic non-convex convergence behavior under block-wise optimization.
      
    \end{remark}

\section{Additional Theoretical Analysis} \label{sec:add_theory}

\subsection{Communication Complexity} \label{ssec:commcomplexity}

    We analyze the per-round communication complexity of each client, where $d$ denotes the number of model parameters and $N_{\max}$ denotes the maximum node degree of the communication graph. Standard decentralized algorithms, such as D-SGD~\cite{LianZZHZL-NIPS17}, synchronize full-model parameters with neighboring clients at each iteration, incurring a per-client communication cost of $\mathcal{O}(dN_{\max})$. For LLMs with billions of parameters, such full-parameter exchange can easily saturate communication links and lead to severe latency. In contrast, DECA exchanges only the active block of size $d/B$ during each inner step, resulting in a per-round complexity of $\mathcal{O}\!\left(dRN_{\max}/B\right)$. Unlike decentralized PEFT methods~\cite{GhiasvandAP-REALM25,Saadati-NeuralNet26}, which improve communication efficiency by restricting updates to low-rank parameter subspaces, DECA reduces bandwidth cost through block-wise communication while preserving full-parameter fine-tuning capacity.

    %
    %The $\tau$-DECA variant further lowers the communication frequency, yielding an amortized cost of $\mathcal{O}\left(\frac{dRN_{\rm max}}{\tau B}\right)$. 
    %

  \subsection{Memory Consumption} \label{ssec:memconsumption}

    We analyze the GPU memory consumption of DECA. Let $M$ denote the FP16 memory required to store the model parameters, i.e., $M \approx 2d$ bytes. In standard decentralized FPFT with Adam~\cite{KingmaB-ICLR2015}, each client stores FP16 parameters for forward and backward propagation, together with FP32 master parameters, gradients, and first- and second-order moment estimates. This results in a total memory cost of approximately $9M$. When the consensus-derived local discrepancy is further incorporated, an additional FP32 auxiliary vector $h$ is required (see Line~12 in \textbf{Algorithm}~\ref{alg:deca}), increasing the memory cost to approximately $11M$. In contrast, DECA maintains FP32 optimizer states only for the active block, reducing the peak memory usage to $M + 10M/B$. This analysis excludes the memory required for intermediate activations, as activation storage is inherent to backpropagation rather than optimizer-state maintenance. Nevertheless, DECA is fully compatible with gradient checkpointing~\cite{ChenXZG-arXiv16}, which trades additional computation for reduced activation memory and can further improve feasibility on memory-constrained devices.

  \subsection{Computation Complexity} \label{ssec:compcomplexity}

    We analyze the computational efficiency of DECA under a layer-wise partitioning strategy that groups consecutive Transformer layers into blocks. The computational cost is measured using an \emph{atomic backward operation}, defined as the backward pass of one batch through a single Transformer module. In standard decentralized SGD or LoRA, trainable parameters span all layers, requiring full-depth backpropagation for every batch. Let $R$ denote the number of batches processed per block (see \textbf{Algorithm}~\ref{alg:deca}). A global round processes $RB$ batches, each traversing all $B$ blocks, resulting in $RB^2$ atomic operations. In contrast, DECA updates only the active block. Backpropagation proceeds only from the loss to the active block, traversing $k$ blocks when updating the $k$-th block. Summing over all blocks yields a total cost of $R \sum_{k=1}^{B} k = \tfrac{1}{2} R B (B+1)$. Moreover, DECA reduces backward computation by skipping parameter gradient calculations for frozen intermediate layers. Instead, it performs only the vector Jacobian products required to propagate gradients to the active block, effectively halving the arithmetic cost for these layers.

\section{More Information about Our Experiment Settings} \label{sec:app_expsetting}
  We select the following models to demonstrate that our algorithm scales effectively across different parameter sizes ($1.5\text{B}- 8\text{B}$) and architectures:
  \begin{itemize}
    \item \textbf{Qwen2-1.5B} is a lightweight yet capable model from the Qwen series, designed for efficient deployment under limited computational and memory budgets. Its inclusion allows us to assess whether DECA remains effective for extremely compact models, further validating its applicability to mobile, edge, and resource-constrained decentralized scenarios.
    \item \textbf{Qwen2.5-3B-Instruct} is a compact architecture tailored for mobile and edge applications. We include this model to explicitly validate the feasibility of our approach within resource-constrained decentralized environments.
    \item \textbf{Llama-2-7B} is a widely adopted open-weights model from Meta, representing a strong previous-generation baseline in the 7B parameter class. We include this model to evaluate the effectiveness of DECA on a mature and extensively benchmarked architecture, enabling direct comparison with prior decentralized and instruction-following studies.
    \item \textbf{Llama-3.1-8B-Instruct} is the latest open-weights offering from Meta. This model serves as the contemporary industry standard for the 8B parameter class. Its instruction-tuning allows it to act as a robust baseline for assessing chat and instruction-following capabilities.
    % \item \textbf{Qwen2.5-14B-Instruct} is a high-performance medium-scale model that frequently surpasses significantly larger baselines (e.g., Llama-2-70B). Its inclusion demonstrates the scalability of DECA to larger, more capable architectures.
  \end{itemize}

  We utilize the following datasets to evaluate the model's understanding capabilities across diverse domains. NWGI, AGNEWS, TFNS, and MNLI are datasets for classification tasks, while Alpaca is a dataset for generation tasks. 

  \begin{itemize}
    \item \textbf{NWGI (News Writer Genre Identification)}~\footnote{\url{https://huggingface.co/datasets/dohonba/nwgi}} is a literary analysis task where the model classifies the genre or stylistic identity of a news writer based on article text. In this work, it serves as a benchmark dataset for evaluating the model's performance on classification tasks.
    \item \textbf{AGNEWS}~\footnote{\url{https://huggingface.co/datasets/fancyzhx/ag_news}} is a classic benchmark dataset for topic classification that challenges models to categorize news snippets into four distinct classes: World, Sports, Business, and Sci/Tech. It is widely used to evaluate a model's fundamental content understanding and its ability to discern the primary subject matter of short text segments.
    \item \textbf{TFNS (Twitter Financial News Sentiment)}~\footnote{\url{https://huggingface.co/datasets/zeroshot/twitter-financial-news-sentiment}} is a domain-specific sentiment analysis benchmark that tasks models with classifying financial tweets as ``Bullish'' (positive) or ``Bearish'' (negative). This dataset is specifically designed to evaluate a model's grasp of financial domain knowledge and its ability to interpret sentiment within the specialized context of market-related social media discourse.
    \item \textbf{MNLI (Multi-Genre Natural Language Inference)}~\footnote{\url{https://huggingface.co/datasets/SetFit/mnli}} is a challenging benchmark for evaluating a model's logic and reasoning capabilities. In this task, the model is presented with a premise and a hypothesis and must determine whether the premise entails, contradicts, or remains neutral toward the hypothesis.
    %
    % \item \textbf{EMOTION}~\footnote{\url{https://huggingface.co/datasets/dohonba/many_emotions}} is a fine-grained text classification benchmark designed to evaluate a model's sensitivity to affective nuance beyond simple binary sentiment. Sourced primarily from social media platforms like Twitter, it challenges models to categorize informal text into a diverse spectrum of emotional states, such as joy, sadness, anger, fear, love, and surprise. This dataset specifically tests the ability to decode real-world communication complexities, including colloquialisms, sarcasm, and varying emotional intensity.
    %
    \item \textbf{Alpaca}~\footnote{\url{https://huggingface.co/datasets/tatsu-lab/alpaca}} is a dataset consisting of 52,000 instruction-response pairs generated by text-davinci-003. It serves as a standard benchmark for evaluating a large language model's ability to follow diverse user instructions. In this paper, the dataset is utilized to assess the generation capabilities of the fine-tuned models.
  \end{itemize}

  We employ accuracy (Acc) and $F_1$ score to assess model performance for classification tasks. As for the generation tasks, we adopt the following evaluation metrics:
  \begin{itemize}
    % \item \textbf{Rouge-L (ROU-L)}~\cite{Lin-TSBO04} measures the overlap between the generated summary and the reference using the longest common subsequence. It is the standard metric for summarization tasks.
    % %
    % \item \textbf{Meteor (MET.)}~\cite{BanerjeeL-ACL05} is an advanced metric that correlates well with human judgment by accounting for synonyms (e.g., matching ``fast'' with ``quick'') and stemming (e.g., matching ``run'' with ``running'').
    % %
    % \item \textbf{CIDEr}~\cite{VedantamLP-CVPR15} is designed for image captioning to evaluate how ``consensus-like'' the generated text is. It utilizes TF-IDF weighting to prioritize rare, informative words found in human references over common stopwords.
    %
    \item \textbf{Vicuna (VIC.)}~\cite{ChiangLLSWZZZZG-blog2023} is a GPT-4-based framework for evaluating chatbot quality. Instead of relying on static NLP benchmarks, it uses 80 diverse, challenging questions spanning across 8 tasks to compare model responses side-by-side, with GPT-4 serving as an automatic judge, offering an efficient proxy for human assessment of open-ended chatbot responses.
    \item \textbf{MT-Bench (MT.)}~\cite{ZhengCSZWZLLLX-NIPS24} is a modern benchmark for evaluating multi-turn conversation quality across categories like coding, reasoning, and roleplay. It employs a ``LLM-as-a-Judge'' approach, using a stronger model (e.g., GPT-4) to score responses on a scale of 1--10, providing a robust proxy for human evaluation.
  \end{itemize}

    We select three state-of-the-art decentralized LLM fine-tuning methods as baselines. Specifically, \textbf{Dec-Adapter}~\cite{HoulsbyGJMDGAG-ICML19} extends adapter-based PEFT to decentralized settings. \textbf{Dec-LoRA}~\cite{GhiasvandAP-REALM25} and \textbf{DeCAF}~\cite{Saadati-NeuralNet26} propose decentralized LoRA-based fine-tuning strategies that eliminate the central server and rely on peer-to-peer client collaboration. By benchmarking against these methods, we systematically compare our decentralized FPFT paradigm with representative decentralized PEFT approaches across different system settings.
      \begin{itemize}
      %
        % \item \textcolor{red}{\textbf{FFA-LoRA}~\cite{SunLLD-ICLR24}, LoRA -based federated LLM fine-tuning approach which freezes the A matrices in LoRA and only trains and communicates and aggregates the B matrices.}
        %
        % \item \textbf{FedSA-LoRA}~\cite{GuoZWFQ-ICLR25} is a federated LoRA method. It utilizes a data-driven strategy to selectively aggregate only the most informative LoRA updates from participating clients. By evaluating the quality or relevance of local adapters before aggregation, this method filters out redundant or detrimental parameter updates, thereby enhancing convergence stability while reducing communication overhead.
        % %
        % \item \textbf{FedPrompt}~\cite{ZhaoDLLL-ICASSP23} applies prompt tuning within a Federated Learning setting. By keeping the entire pre-trained backbone frozen and aggregating only these lightweight ``soft prompts'' via a central server. This method entials light-weight communications and reduces the computational burden on local clients.
        %
        % \item \textbf{FedCyBGD}, a federated block-wise full-parameter fine-tuning method in which each client trains a fixed block of the model, and the server aggregates updates across blocks.
        %
        \item \textbf{Dec-Adapter}~\cite{HoulsbyGJMDGAG-ICML19} is a decentralized extension of adapter-based PEFT. It follows the original adapter design, where lightweight trainable modules are inserted into a frozen backbone model, and adapts this paradigm to decentralized learning by allowing clients to exchange and aggregate only adapter parameters with their neighbors.
        \item \textbf{Dec-LoRA}~\cite{GhiasvandAP-REALM25} proposes a decentralized LoRa method. Specifically, each client maintains a local copy of the frozen backbone and a set of trainable low-rank adapters, exchanging and aggregating only the adapter parameters with their neighbors. This approach combines the communication efficiency of LoRA with the robustness of decentralized learning.
        \item \textbf{DeCAF}~\cite{Saadati-NeuralNet26}, a decentralized LoRA-based framework that replaces traditional aggregation with consensus-aware updates, using TSVD decomposition to aggregate the local LoRA matrices.
      \end{itemize}

      We conduct our experiments on a GPU cluster, where each GPU corresponds to one client. Specifically, we use NVIDIA RTX 5090 GPUs for the Qwen2-1.5B and Qwen2.5-3B models, and NVIDIA RTX PRO 6000 GPUs for the larger Llama-2-7B and Llama-3.1-8B models. 
      %
      %Our code~\footnote{\url{https://anonymous.4open.science/r/DECA-3F73}} is avaliable. 

\section{Additional Experiment Results} \label{sec:app_expresults}
  We evaluate the different algorithms under different communication topologies and varying numbers of clients in Sec.~\ref{ssec:topology} and Sec.~\ref{ssec:clients}, respectively. We then investigate the impact of different partitioning granularities on the performance of our algorithm in Sec.~\ref{ssec:granularity} and report the resource consumption of the different algorithms in Sec.~\ref{ssec:app_exp_resource}. We finally perform hyperparameter sensitivity analysis in Sec.~\ref{ssec:sensitivity}.

  \subsection{Different Communication Topologies} \label{ssec:topology}
    We employ the Qwen2.5-3B model to evaluate the classification performance of the proposed methods against baselines across ring, ER, and bipartite communication topologies, each comprising eight clients. The experimental results in Table~\ref{tab:app_classexp_topo} clearly demonstrate the effectiveness of DECA under the different communication topologies. Overall, DECA achieves the best average accuracy and $F_1$ scores across all three topologies, showing strong robustness to variations in network connectivity. In the ring topology, DECA obtains an average accuracy of 75.58\% and an average $F_1$ score of 69.64\%, outperforming Dec-Adapter by 0.79\% and 2.87\%, respectively. Compared with Dec-LoRA, DECA further improves the average accuracy and $F_1$ score by 4.38\% and 7.99\%, and the gains over DeCAF are even larger. Under the bipartite topology, DECA again consistently delivers the best overall performance, improving average accuracy and $F_1$ score by 2.79\% and 3.18\% over Dec-Adapter, 4.27\% and 8.60\% over Dec-LoRA, and 12.82\% and 15.16\% over DeCAF.
    Notably, DECA maintains highly stable performance across different topologies, with average accuracy ranging only from 75.22\% to 76.14\% and average $F_1$ score ranging from 69.22\% to 69.64\%. This indicates that DECA can effectively mitigate the negative impact of constrained or heterogeneous communication structures. Overall, these results confirm that DECA consistently achieves lower classification loss and stronger predictive performance than Dec-Adapter, while also surpassing decentralized PEFT baselines such as Dec-LoRA and DeCAF across different network topologies.
      \begin{table*}[t!]
      \centering
      \begin{small}
      \caption{Performance comparison of different algorithms under different communication topologies. }
      %For clarity, the results on the ER graph are reproduced from Table~\ref{tab:exp_classification}}
      \label{tab:app_classexp_topo}
      \begin{tabular}{l l *{6}{c} *{6}{c}}
      \toprule
      \multirow{2}{*}{\textbf{Topology}}
      & \multirow{2}{*}{\textbf{Method}}
      & \multicolumn{2}{c}{\textbf{NWGI}}
      & \multicolumn{2}{c}{\textbf{TFNS}}
      & \multicolumn{2}{c}{\textbf{MNLI}}
      & \multicolumn{2}{c}{\textbf{Average}} \\
      \cmidrule(lr){3-4} \cmidrule(lr){5-6} \cmidrule(lr){7-8} \cmidrule(lr){9-10}
      & & \textit{Acc.} & $F_1$
      & \textit{Acc.} & $F_1$
      & \textit{Acc.} & $F_1$
      & \textit{Acc.} & $F_1$ \\
      \midrule
      \multirow{4}{*}{Ring}
      & Dec-Adapter
      & \textbf{56.11} & \underline{36.43} & \underline{88.73} & \textbf{85.84} & \underline{79.54} & \underline{78.04} &\underline{74.79} &\underline{66.77} \\
      & Dec-LoRA
      & 53.52 & 32.91 & 86.31 & 81.86 &73.77 & 70.17 &71.20 &61.65 \\
      & DeCAF
      & 43.91 & 28.17 &67.37 & 56.92 & 72.02 & 67.70 &61.10 &50.93\\
      & \textbf{DECA}
      & \underline{54.09} & \textbf{40.11} &\textbf{89.35} &\underline{85.72} &\textbf{ 83.31} & \textbf{83.10} & \textbf{75.58} &\textbf{69.64} \\
      % & $\tau$-DECA
      % &\underline{52.64} & \underline{34.03} &\underline{81.92} &\underline{80.53} & \underline{82.59} & \underline{77.61}  & 68.25 & 65.19 & \underline{81.08} & \underline{77.17}  & \underline{73.30} & \underline{66.91} \\
      \midrule
      \multirow{4}{*}{ER} 
      & Dec-Adapter
      &\underline{52.72} &\underline{35.55} &\underline{88.36} &\underline{85.47} &\underline{80.57} &\underline{78.65} &\underline{73.88} &\underline{66.56} \\
      & Dec-LoRA
      &52.44 &35.23 &84.69 & 78.31 &77.63 &76.01 &71.59 &63.18 \\
      & DeCAF
      &42.93 &27.96 &66.12 &55.52 &72.09 &67.62 &60.38 &50.37 \\
      & \textbf{DECA}
      &\textbf{53.51} &\textbf{36.09} &\textbf{89.25} &\textbf{86.24} &\textbf{85.67} &\textbf{85.34} &\textbf{76.14} &\textbf{69.22} \\
      \midrule
      \multirow{4}{*}{Bipartite}
      & Dec-Adapter
      &49.47 &\underline{35.70} & \underline{88.34} & \underline{84.74} &\underline{79.48} &\underline{78.64} &\underline{72.43} &\underline{66.36} \\
      & Dec-LoRA
      & \underline{52.30} &32.53 &83.47 &75.97 & 77.09 &74.31 &70.95 &60.94 \\
      & DeCAF
      & 43.49 &27.56 & 72.08 &68.36 &71.64 &67.22 &62.40 &54.38 \\
      & \textbf{DECA}
      &\textbf{53.07} &\textbf{39.32} &\textbf{88.65} &\textbf{85.68} &\textbf{83.95} &\textbf{83.63} &\textbf{75.22} &\textbf{69.54} \\
      % & $\tau$-DECA
      % &\underline{55.01} &\underline{36.70} &\underline{90.49} &\underline{90.51} &\underline{87.31} &\underline{84.55} &\underline{82.81} &\underline{82.18} &\underline{80.27} &\underline{75.38} \\
      \bottomrule
      \end{tabular}
      \end{small}
      \end{table*}

      Furthermore, Fig.~\ref{fig:loss_ring} presents the training loss trajectories of the different algorithms under a ring topology. DECA achieves a convergence speed and final training loss comparable to those of decentralized PEFT baselines, indicating that our block-wise full-parameter adaptation can remain stable even under sparse peer-to-peer communication. Fig.~\ref{fig:loss_topo} further shows the training loss of DECA under different network topologies. Across all topologies, DECA exhibits steady and consistent convergence behavior, demonstrating its robustness to diverse communication structures.

      %
      % \begin{figure*}[t!]
      % \centering
      % \parbox{.32\textwidth}{\center\includegraphics[width=.32\textwidth]{result/loss_plot/loss_qwen3_ring.pdf}}
      % \parbox{.32\textwidth}{\center\includegraphics[width=.32\textwidth]{result/loss_plot/loss_qwen3.pdf}}
      % \parbox{.32\textwidth}{\center\includegraphics[width=.32\textwidth]{result/loss_plot/loss_qwen3_bipar.pdf}}
      % \parbox{.32\textwidth}{\center\scriptsize(a) Ring}
      % \parbox{.32\textwidth}{\center\scriptsize(b) ER}
      % \parbox{.32\textwidth}{\center\scriptsize(b) Bipartite}
      % \caption{Trainig loss of different decentralized methods on EMOTION dataset under various communication toploties.}
      % \label{fig:app_classexp_loss_topo}
      % \end{figure*}
      % %
      % %
      % \begin{figure*}[t!]
      % \centering
      % \parbox{.32\textwidth}{\center\includegraphics[width=.32\textwidth]{result/loss_plot/loss_qwen3_gen_ring.pdf}}
      % \parbox{.32\textwidth}{\center\includegraphics[width=.32\textwidth]{result/loss_plot/loss_qwen3_gen.pdf}}
      % \parbox{.32\textwidth}{\center\includegraphics[width=.32\textwidth]{result/loss_plot/loss_qwen3_gen_bipar.pdf}}
      % \parbox{.32\textwidth}{\center\scriptsize(a) Ring}
      % \parbox{.32\textwidth}{\center\scriptsize(b) ER}
      % \parbox{.32\textwidth}{\center\scriptsize(b) Bipartite}
      % \caption{Trainig loss of different decentralized methods on Alpaca dataset under various communication toploties.}
      % \label{fig:app_genexp_loss_topo}
      % \end{figure*}

      \begin{figure*}[t]
          \centering
          \begin{minipage}[t]{0.48\textwidth}
             \centering
              \includegraphics[width=\textwidth]{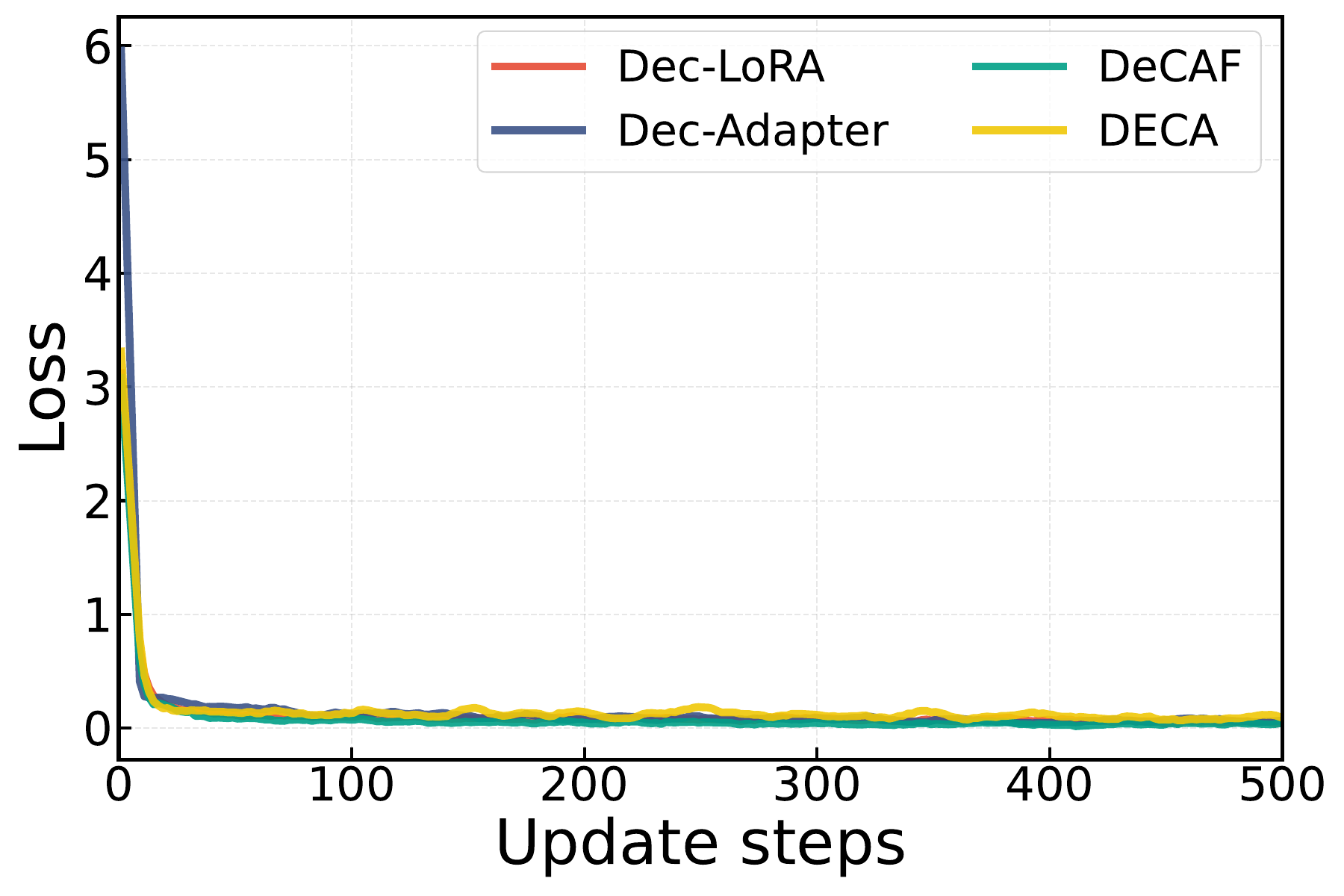}
              \captionof{figure}{Training loss of different algorithms on NWGI dataset under ring topology.}
              \label{fig:loss_ring}
          \end{minipage}
          \hfill
          \begin{minipage}[t]{0.48\textwidth}
              \centering
              \includegraphics[width=\textwidth]{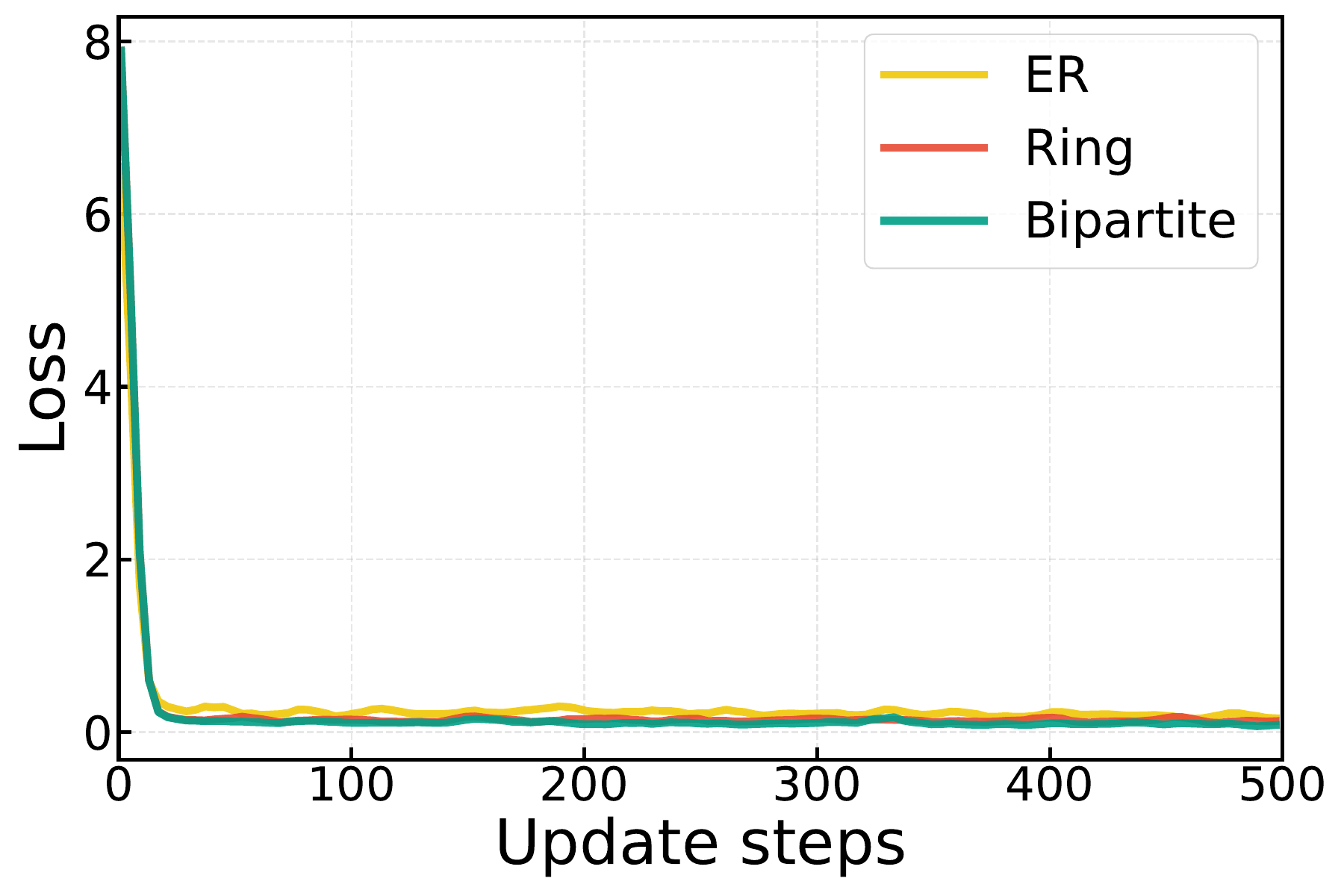}
              \captionof{figure}{Training loss of DECA on MNLI dataset under different topologies.}
              \label{fig:loss_topo}
          \end{minipage}
      \end{figure*}

      \begin{figure*}[t]
          \centering
          \begin{minipage}[t]{0.48\textwidth}
              \centering
              \includegraphics[width=\textwidth]{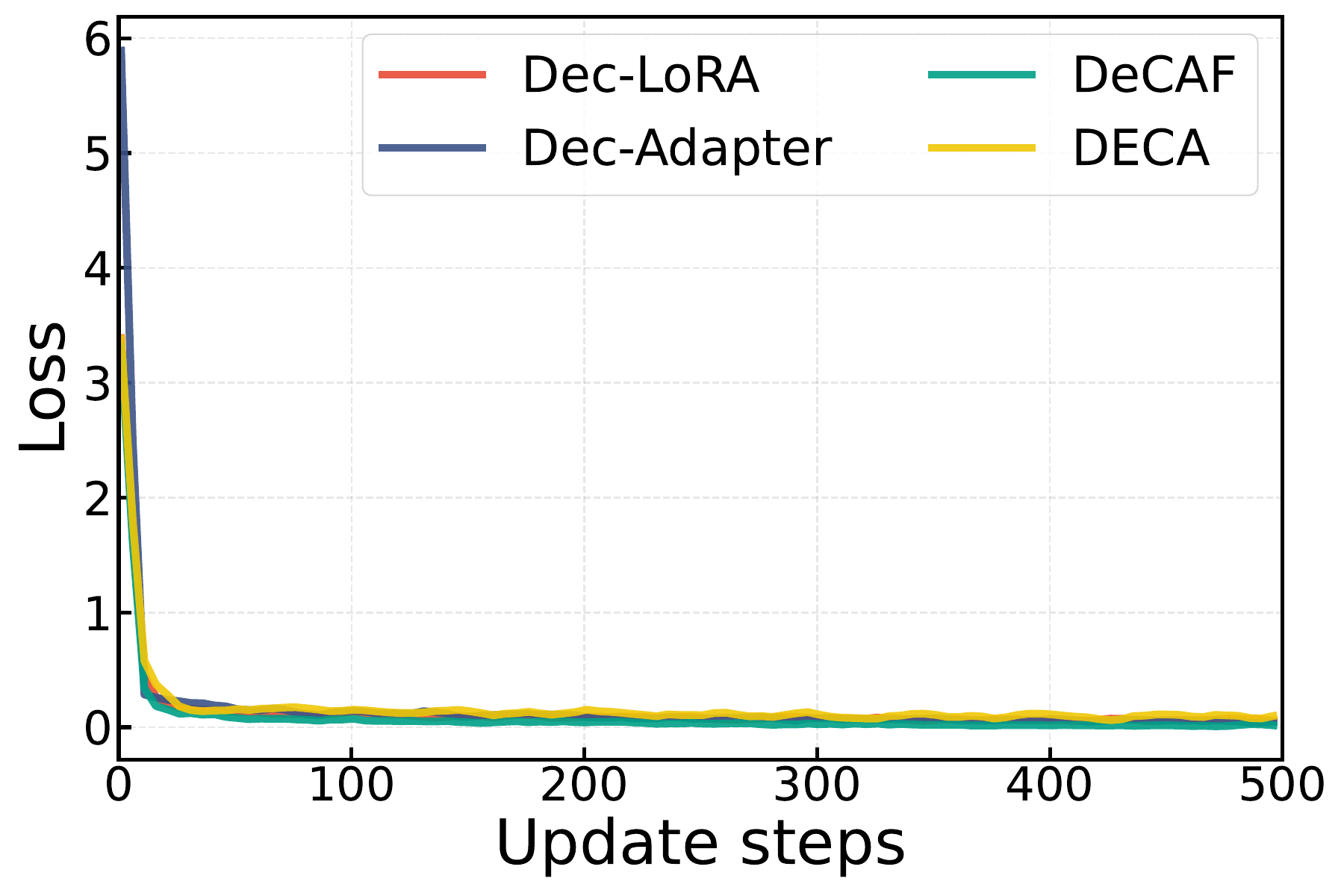}
              \captionof{figure}{Training loss of different algorithms with 12 clients on TFNS dataset.}
              \label{fig:loss_nwgi_12}
          \end{minipage}
          \hfill
          \begin{minipage}[t]{0.48\textwidth}
              \centering
              \includegraphics[width=\textwidth]{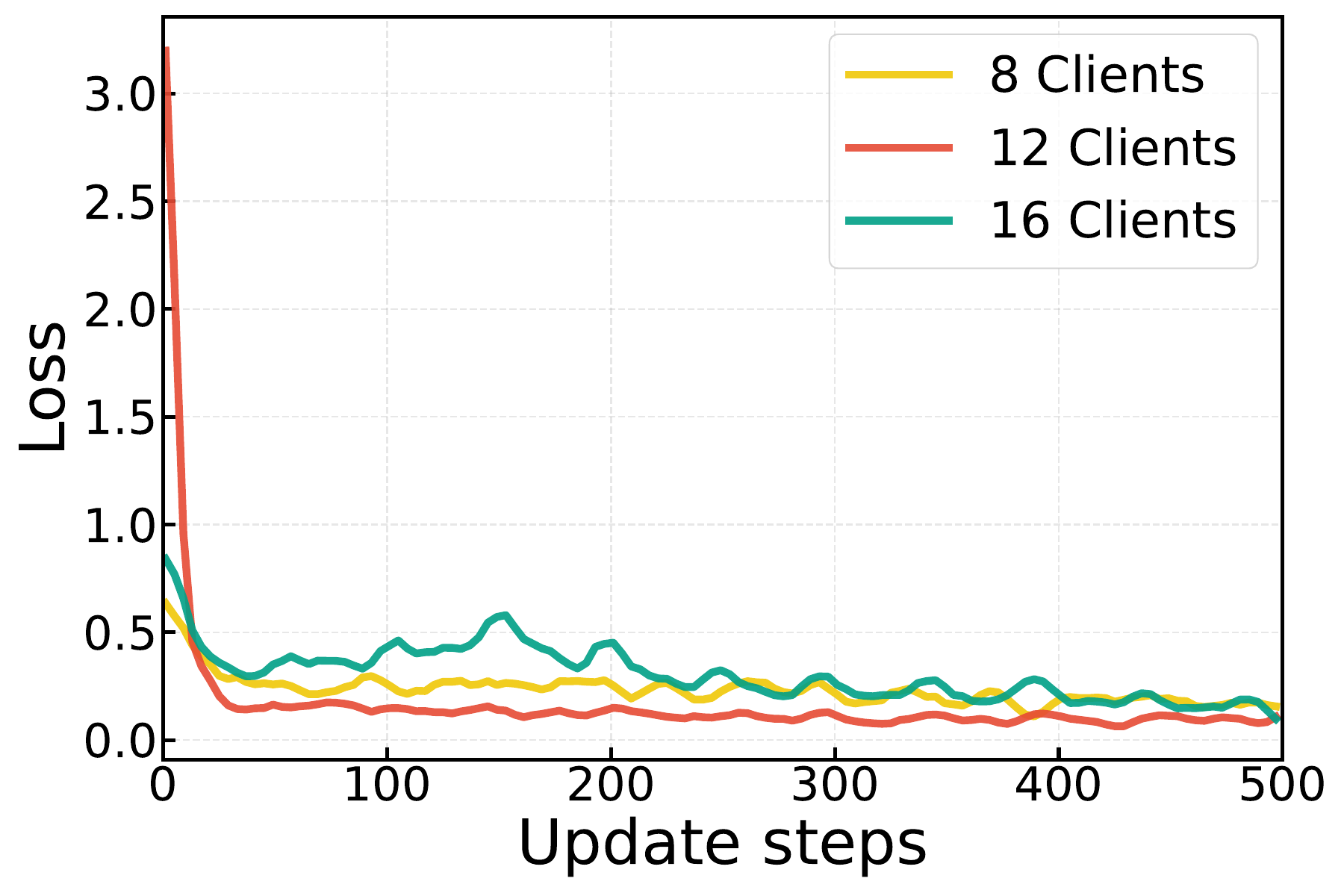}
              \captionof{figure}{Training loss of DECA with different numbers of clients on NWGI dataset.}
              \label{fig:loss_clients}
          \end{minipage}
      \end{figure*}

    \subsection{Different Numbers of Clients} \label{ssec:clients}
      We study the effect of network scale by varying the number of clients, $N\in\{8,12,16\}$, and applying different decentralized algorithms to fine-tune a Qwen2.5-3B model on the five classification datasets under the representative ER topology. As shown in Table~\ref{tab:app_classexp_client}, DECA consistently outperforms decentralized PEFT baselines across all network sizes, demonstrating its robustness to client-scale variations. In the largest setting with $N=16$, DECA achieves an average accuracy of 70.32\%, exceeding Dec-LoRA by nearly 2.25\%. When $N=12$, DECA also maintains a clear advantage over the baselines in most settings. These results confirm that DECA enables effective FPFT even as the communication network expands.
      \begin{table*}[t!]
      \centering
      \begin{small}
      \caption{Performance comparison of different algorithms across varying numbers of clients.}  
      \label{tab:app_classexp_client}
      \begin{tabular}{l l *{6}{c} *{6}{c}}
      \toprule
      \multirow{2}{*}{\textbf{$N$}}
      & \multirow{2}{*}{\textbf{Algorithm}}
      & \multicolumn{2}{c}{\textbf{NWGI}}
      & \multicolumn{2}{c}{\textbf{TFNS}}
      & \multicolumn{2}{c}{\textbf{MNLI}}
      & \multicolumn{2}{c}{\textbf{Average}} \\
      \cmidrule(lr){3-4} \cmidrule(lr){5-6} \cmidrule(lr){7-8}
      \cmidrule(lr){9-10} 
      & & \textit{Acc.} & $F_1$
      & \textit{Acc.} & $F_1$
      & \textit{Acc.} & $F_1$
      & \textit{Acc.} & $F_1$ \\
      \midrule
      \multirow{4}{*}{8}
      & Dec-Adapter
      &\underline{52.72} &\underline{35.55} &\underline{88.36} &\underline{85.47} &\underline{80.57} &\underline{78.65} &\underline{73.88} &\underline{66.56} \\
      & Dec-LoRA
      &52.44 &35.23 &84.69 & 78.31 &77.63 &76.01 &71.59 &63.18 \\
      & DeCAF
      &42.93 &27.96 &66.12 &55.52 &72.09 &67.62 &60.38 &50.37 \\
      & \textbf{DECA}
      &\textbf{53.51} &\textbf{36.09} &\textbf{89.25} &\textbf{86.24} &\textbf{85.67} &\textbf{85.34} &\textbf{76.14} &\textbf{69.22} \\
      \midrule
      \multirow{4}{*}{12}
      & Dec-Adapter
      & \underline{44.51} & \underline{32.92} & 88.84 & 85.84 & 85.15 & 85.18 &72.83 &\underline{67.98} \\
      & Dec-LoRA
      & 43.64 & 31.02 &\textbf{89.80} & \underline{86.35} & \underline{85.52} & \underline{85.34} &\underline{72.99} &67.57 \\
      & DeCAF
      &36.66 &21.14 &65.95 &51.18 &65.44 &59.61 &56.02 &43.98 \\
      & \textbf{DECA}
      &\textbf{46.40} &\textbf{34.04} &\underline{89.44} &\textbf{86.59} &\textbf{86.85} &\textbf{86.28} &\textbf{74.23} &\textbf{68.97} \\
      % & $\tau$-DECA
      % &\underline{41.10} & 28.02  &\textbf{86.16} &\textbf{81.75} & 71.85 & 69.30 & \underline{73.10} & \underline{68.11} \\
      \midrule
      \multirow{4}{*}{16}
      & Dec-Adapter
      & \underline{46.89} & \underline{30.12} & \underline{84.35} & \underline{81.99} & \underline{78.65} & \textbf{77.03} &\underline{69.96} &\underline{63.05} \\
      & Dec-LoRA
      & 46.23 & 29.52 & 82.45 & 79.16 & 75.54 & 73.77 &68.07 &60.82 \\
      & DeCAF
      & 42.05 &26.56 &65.22 &56.38 &63.73 &60.44 &57.00 &47.79 \\
      & \textbf{DECA}
      & \textbf{47.30} & \textbf{30.70} & \textbf{84.83} & \textbf{82.33} & \textbf{78.82} & \underline{76.69} & \textbf{70.32} &\textbf{63.24} \\
      % & $\tau$-DECA
      % & 46.86 &\underline{29.80} & \underline{85.76} & \underline{81.39} & \underline{74.16} & \underline{71.48} & \underline{73.39} & \underline{67.91} \\
      \bottomrule
      \end{tabular}
      \end{small}
      \end{table*}

      Fig.~\ref{fig:loss_nwgi_12} illustrates the training loss trajectories for classification (TFNS) tasks when $N=12$. It is shown that, DECA achieves similar convergence rate and loss value compared to decentralized PEFT baselines. This persists regardless of the system scale, highlighting the scalability of the proposed FPFT framework. Furthermore, Fig.~\ref{fig:loss_clients} demonstrates the training loss trajectories of DECA on NWGI task across the varying numbers of clients. As the number of clients grows, DECA continues to exhibit nearly unchanged convergence speed and loss levels, suggesting that it scales well and remains effective even in large-scale network settings. When more clients are involved, the heterogeneity induced by non-IID data becomes more pronounced, often causing noticeable fluctuations in the initial training stages (e.g. 16 clients). Nevertheless, DECA consistently converges to comparable loss levels across different client configurations.

  \subsection{Different Partitioning Granularities} 
  \label{ssec:granularity}
    We further investigate the impact of model partitioning granularity by setting the granularity to $\{1,2,4\}$, where the parameter controls the number of consecutive Transformer layers grouped into a single trainable block. Based on the results in Table~\ref{tab:app_classexp_granularity}, where DECA is evaluated with the Llama-3.1-8B model on the NWGI, TFNS, and MNLI classification tasks, we observe that finer-grained partitioning consistently yields better performance. Specifically, when granularity $=1$, DECA achieves the best results across all datasets and metrics, obtaining an average accuracy of $78.81\%$ and an average $F_1$ score of $71.98\%$. As the granularity becomes coarser, the performance gradually decreases: increasing granularity from $1$ to $2$ reduces the average accuracy from $78.81\%$ to $78.10\%$, while further increasing it to $4$ lowers the accuracy to $76.71\%$. A similar monotonic degradation can also be observed in the average $F_1$ score, which drops from $71.98\%$ to $70.33\%$ and then to $69.43\%$. Among the three datasets, NWGI is the most sensitive to coarser partitioning, with its $F_1$ score decreasing from $39.82\%$ at granularity $=1$ to $36.49\%$ at granularity $=4$. These results suggest that, for DECA, a finer model partitioning strategy provides more flexible layer-wise optimization and better preserves task-specific adaptation capability, whereas overly coarse blocks may reduce the effectiveness of localized training and lead to degraded classification performance.

    Fig~\ref{fig:loss_gran_mnli} and Fig~\ref{fig:loss_gran_tfns} illustrate the loss curves of DECA under different block granularities across two datasets. While all settings exhibit consistent overall convergence trends, the granularity affects the optimization dynamics. Specifically, a finer granularity (e.g., granularity $=1$) generally achieves slightly lower loss values throughout most of the training process and produces a smoother trajectory with reduced fluctuations, especially in the later stages. In contrast, a coarser granularity (e.g., granularity $=4$) tends to exhibit more pronounced oscillations and occasional spikes, suggesting higher gradient variance and less stable updates. Overall, finer granularity improves training stability and smoothness by reducing gradient variance, while having minimal impact on convergence.
    \begin{table*}[t]
    \centering
    \begin{small}
    \caption{Performance of DECA under different model partitioning granularities using Llama3-8B model.}
    \label{tab:app_classexp_granularity}
      \begin{tabular}{c *{6}{c} *{6}{c}}
      \toprule
      \multirow{3}{*}{\textbf{Granularity}}
      & \multicolumn{2}{c}{\textbf{NWGI}}
      & \multicolumn{2}{c}{\textbf{TFNS}}
      & \multicolumn{2}{c}{\textbf{MNLI}}
      & \multicolumn{2}{c}{\textbf{Average}} \\
      \cmidrule(lr){2-3} \cmidrule(lr){4-5} \cmidrule(lr){6-7}
      \cmidrule(lr){8-9} \cmidrule(lr){10-11}
      & \textit{Acc.} & $F_1$
      & \textit{Acc.} & $F_1$
      & \textit{Acc.} & $F_1$
      & \textit{Acc.} & $F_1$ \\
      \midrule
      \multirow{1}{*}{1}
      & {57.36} & {39.82} & {90.42} & {87.66} & {88.64} &  {88.46} & {78.81} & {71.98} \\
      \midrule
      \multirow{1}{*}{2}
      &56.68 &36.73 &89.94 &86.78 &87.68 &87.49 &78.10 &70.33 \\
      \midrule
      \multirow{1}{*}{4}
      &54.73 &36.49 & 88.81 &85.26 &86.58	&86.55 &76.71 &69.43 \\
      \bottomrule
    \end{tabular}
    \end{small}
    \end{table*}

      \begin{figure*}[t]
          \centering
          \begin{minipage}[t]{0.48\textwidth}
              \centering
              \includegraphics[width=\textwidth]{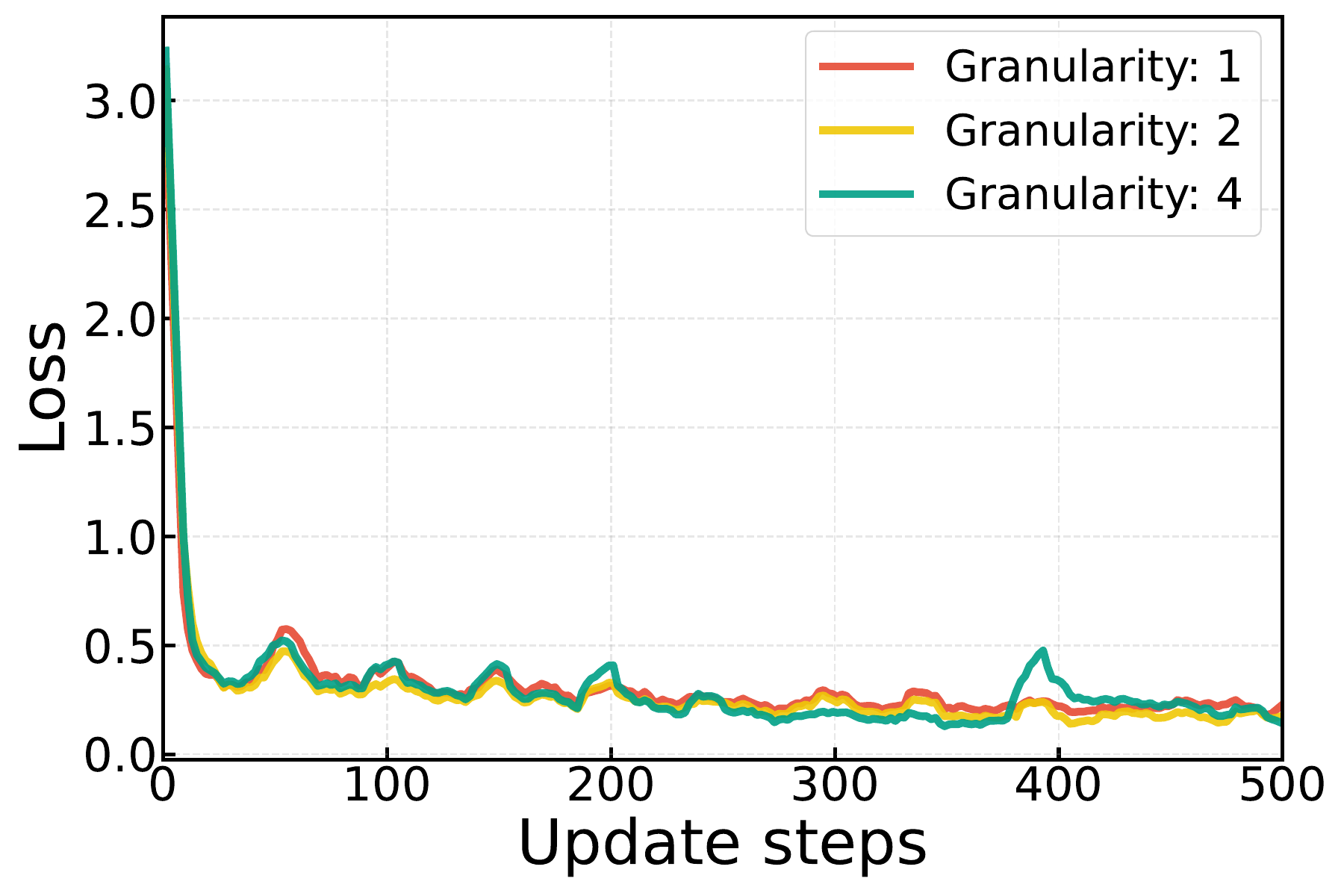}
              \captionof{figure}{Training loss of DECA under different partitioning granularities on MNLI dataset.}
              \label{fig:loss_gran_mnli}
          \end{minipage}
          \hfill
          \begin{minipage}[t]{0.48\textwidth}
              \centering
              \includegraphics[width=\textwidth]{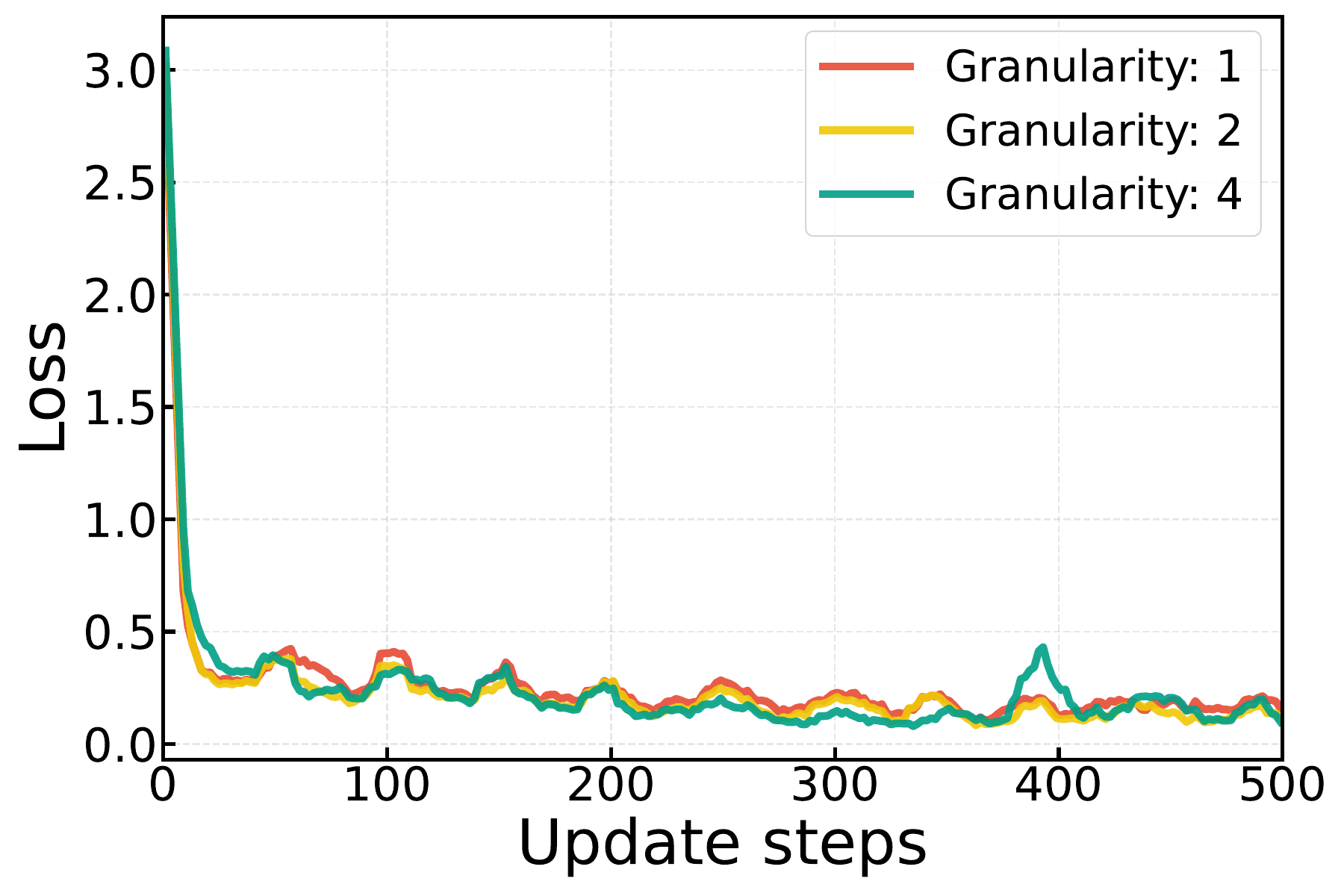}
              \captionof{figure}{Training loss of DECA with different partitioning granularities on TFNS dataset.}
              \label{fig:loss_gran_tfns}
          \end{minipage}
      \end{figure*}

  \subsection{Resource Consumption} \label{ssec:app_exp_resource}

    \begin{table*}[t!]
    \centering
    \caption{Computation and communication overheads across different methods and models. This table summarizes the average forward-propagation latency (\textbf{Fwd}), backward-propagation latency (\textbf{Bkwd}), aggregation-and-update latency (\textbf{Agg \& Up}), and the average number of parameters transmitted per communication step by each client (\textbf{Comm}).}
    \label{tab:comm_comp_model}
    \begin{tabular}{l l *{4}{c}}
    \toprule
    \textbf{Models} & \textbf{Methods} & \textbf{Fwd} (ms) & \textbf{Bkwd} (ms) & \textbf{Agg \& Up} (ms) & \textbf{Comm} (MB) \\
    \midrule
    % \rowcolor{lightgreen} \multicolumn{5}{c}{\textit{Qwen2-1.5B}} \\
    % \midrule
    % 
    \multirow{4}{*}{{Qwen2-1.5B}}
    & Dec-Adapter & 174.21 & 115.42 & 7.11 & 294.67 \\ \cmidrule(lr){2-6}
    & DeCAF & 164.79 & 162.02 & 1141.70 & \multirow{2}{*}{367.50} \\
    & Dec-LoRA & 179.30 & 162.04 & 19.31 & \\ \cmidrule(lr){2-6}
    & DECA & 98.15 & 123.63 & 3.53 & 441.25 \\ 
    \midrule
    % \rowcolor{lightgreen} \multicolumn{5}{c}{\textit{Qwen2.5-3B-Instruct}} \\
    % \midrule
    % 
    \multirow{4}{*}{\makecell{{Qwen2.5-3B}}}
    & Dec-Adapter & 230.12 & 175.06 & 9.38 & 677.49 \\ \cmidrule(lr){2-6}
    & DeCAF & 232.44 & 236.66 & 1637.48 & \multirow{2}{*}{598.50} \\
    & Dec-LoRA & 264.59 & 237.51 & 25.09 & \\ \cmidrule(lr){2-6}
    & DECA & 149.01 & 187.20 & 5.74 & 720.00 \\ 
    \midrule
    % \rowcolor{lightgreen} \multicolumn{5}{c}{\textit{Llama-2-7B}} \\
    % \midrule
    % 
    \multirow{4}{*}{{Llama-2-7B}}
    & Dec-Adapter & 274.96 & 178.60 & 10.13 &2389.04\\ \cmidrule(lr){2-6}
    & DeCAF & 286.62 & 250.14 &  2922.09  & \multirow{2}{*}{1722.00} \\
    & Dec-LoRA & 323.60 & 241.53 &  21.86 & \\ \cmidrule(lr){2-6}
    & DECA & 140.84 & 184.85 & 12.98 & 2100.00 \\
    \midrule
    % \rowcolor{lightgreen} \multicolumn{5}{c}{\textit{Llama-3.1-8B-Instruct}} \\
    % \midrule
    %
    \multirow{4}{*}{\makecell{{Llama-3.1-8B}}}
    & Dec-Adapter & 258.34 & 197.21 & 10.09 & 2389.04 \\ \cmidrule(lr){2-6}
    & DeCAF & 249.03 & 274.29 & 3054.13 & \multirow{2}{*}{1736.00} \\
    & Dec-LoRA & 287.17 & 263.99 & 21.95 & \\ \cmidrule(lr){2-6}
    & DECA & 145.51 & 208.97 & 13.30 & 2128.00 \\
    \bottomrule
    \end{tabular}
    \end{table*}

    Table~\ref{tab:comm_comp_model} presents a comparative analysis of computation and communication overheads across different methods and model scales. The results show that the proposed DECA framework achieves a clear efficiency advantage in end-to-end training latency by jointly reducing forward computation, backward computation, and aggregation/update overhead. Across all evaluated models, DECA consistently yields the fastest forward propagation among all methods. For example, on Llama-2-7B, DECA reduces the forward time from 323.60 ms in Dec-LoRA to 140.84 ms, achieving a reduction of approximately 56.5\%. Similar gains are observed on Qwen2-1.5B, Qwen2.5-3B, and Llama-3.1-8B, where DECA reduces the forward latency by 45.3\%, 43.7\%, and 49.3\% over Dec-LoRA, respectively. 

    In terms of backward propagation, aggregation and update, DECA also significantly outperforms DeCAF and Dec-LoRA across all model scales. For instance, on Llama-3.1-8B, DECA reduces the backward time from 263.99 ms in Dec-LoRA to 208.97 ms, corresponding to a reduction of about 20.8\%. Although Dec-Adapter shows slightly lower backward latency, DECA compensates for this through substantially faster forward propagation, leading to a lower overall computation time for all models. Specifically, on Llama-2-7B, DECA requires only 338.67 ms in a full step, compared with 463.69 ms for Dec-Adapter, 586.99 ms for Dec-LoRA and 3458.85 ms for DeCAF, reducing the end-to-end latency by approximately 27.0\%, 42.3\% and 90.21\%, respectively.
    % 
    % The advantage of DECA becomes even more evident in the aggregation and update phase. DeCAF suffers from extremely high aggregation/update latency, reaching 1141.70 ms on Qwen2-1.5B and over 3000 ms on Llama-3.1-8B, which dominates its overall runtime. In contrast, DECA keeps this overhead within only 3.53--13.30 ms across all models. Compared with Dec-LoRA, DECA reduces the aggregation/update time by 81.7\% on Qwen2-1.5B, 77.1\% on Qwen2.5-3B, 40.6\% on Llama-2-7B, and 39.4\% on Llama-3.1-8B. This demonstrates that the proposed block-wise decentralized update mechanism not only reduces gradient computation but also avoids the costly aggregation bottleneck observed in existing decentralized fine-tuning methods.
    % 
    
    Regarding communication overhead, DECA introduces a moderate increase compared with LoRA-based baselines, since it transmits block-wise full-parameter updates rather than only lightweight adapter parameters. Nevertheless, the communication cost remains manageable across all model scales. More importantly, for larger models such as Llama-2-7B and Llama-3.1-8B, DECA requires less communication than Dec-Adapter, reducing the transmitted volume from 2389.04 MB to 2100.00 MB and 2128.00 MB, respectively. Therefore, DECA achieves a favorable trade-off between communication and computation: it slightly increases communication relative to PEFT-style methods, but in return substantially accelerates forward, backward, and aggregation/update stages, making decentralized full-parameter fine-tuning practical and efficient for large language models.

    %
    % \begin{table*}[t!]
    % \centering
    % \caption{Computation and communication overheads across varying methods and models. This table summarizes the average forward and backward propagation times, along with the average number of parameters transmitted per communication step for each client.}
    % \label{tab:comm_comp_model}
    % \begin{tabular}{l *{5}{c}}
    % \toprule
    % \textbf{Models} & \textbf{Methods} & \textbf{Forward} (ms) & \textbf{Backward} (ms) & \textbf{Comm.} (GB) \\ \midrule
    % \multirow{4}{*}{Qwen2.5-3B}
    % & DeCAF & 413.29 & 398.85 & \multirow{2}{*}{1.09} \\
    % & Dec-LoRA & 418.12 & 403.28 &  \\ \cmidrule(lr){2-5}
    % & DECA & 238.19 & 134.55 & \multirow{2}{*}{0.93} \\
    % & $\tau$-DECA & 237.65 & 135.62 &  \\ \midrule
    % %
    % \multirow{4}{*}{Llama3.1-8B}
    % & DeCAF & 631.70 & 644.27 & \multirow{2}{*}{1.56} \\
    % & Dec-LoRA & 631.12 & 657.05 &  \\ \cmidrule(lr){2-5}
    % & DECA & 471.85 & 295.45 & \multirow{2}{*}{2.64} \\
    % & $\tau$-DECA & 472.94 & 296.87 \\ \midrule
    % %
    % \multirow{4}{*}{Qwen2.5-14B}
    % & DeCAF & 1167.35 & 1093.50 & \multirow{2}{*}{2.63} \\
    % & Dec-LoRA & 1145.01 & 1088.83 &  \\ \cmidrule(lr){2-5}
    % & DECA & 822.46 & 491.26 & \multirow{2}{*}{3.33} \\
    % & $\tau$-DECA & 820.04 & 492.69 &  \\
    % \bottomrule
    % \end{tabular}
    % \end{table*}

    Fig~\ref{fig:gran_time} presents the computational overheads associated with different model partitioning granularities. While increasing the granularity from $1$ to $4$ results in only a moderate increase in backward propagation latency (e.g., rising from {$208.97$ ms} to {$256.61$ ms} for DECA).     Fig.~\ref{fig:gran_gpu} compares the memory consumption of Llama-3.1-8B under different partitioning granularities. The results show that finer granularity is significantly more memory-efficient, with {DECA} requiring only roughly {$46$--$47$ GB} of peak memory and {$29$--$30$ GB} of average memory, which is well within the capacity of commodity GPUs (e.g., PRO 6000 with 96 GB of memory). In contrast, increasing the granularity to $4$ drives the peak memory usage up to approximately {$60$ GB}, a substantial increase that reflects the cost of storing optimizer states for larger active blocks. 
    
    % Specifically, the per-step transmission volume increases fourfold, from $2.08$ GB at granularity $1$ to $8.32$ GB at granularity $4$. 
    %
    %These results indicate that granularity $1$ is the most bandwidth-efficient configuration, as it minimizes the communication load while maintaining slightly lower computation latency.
    %
    \begin{figure*}[t]
    \centering
    \begin{minipage}[t]{0.48\textwidth}
      \centering
      \includegraphics[width=\textwidth]{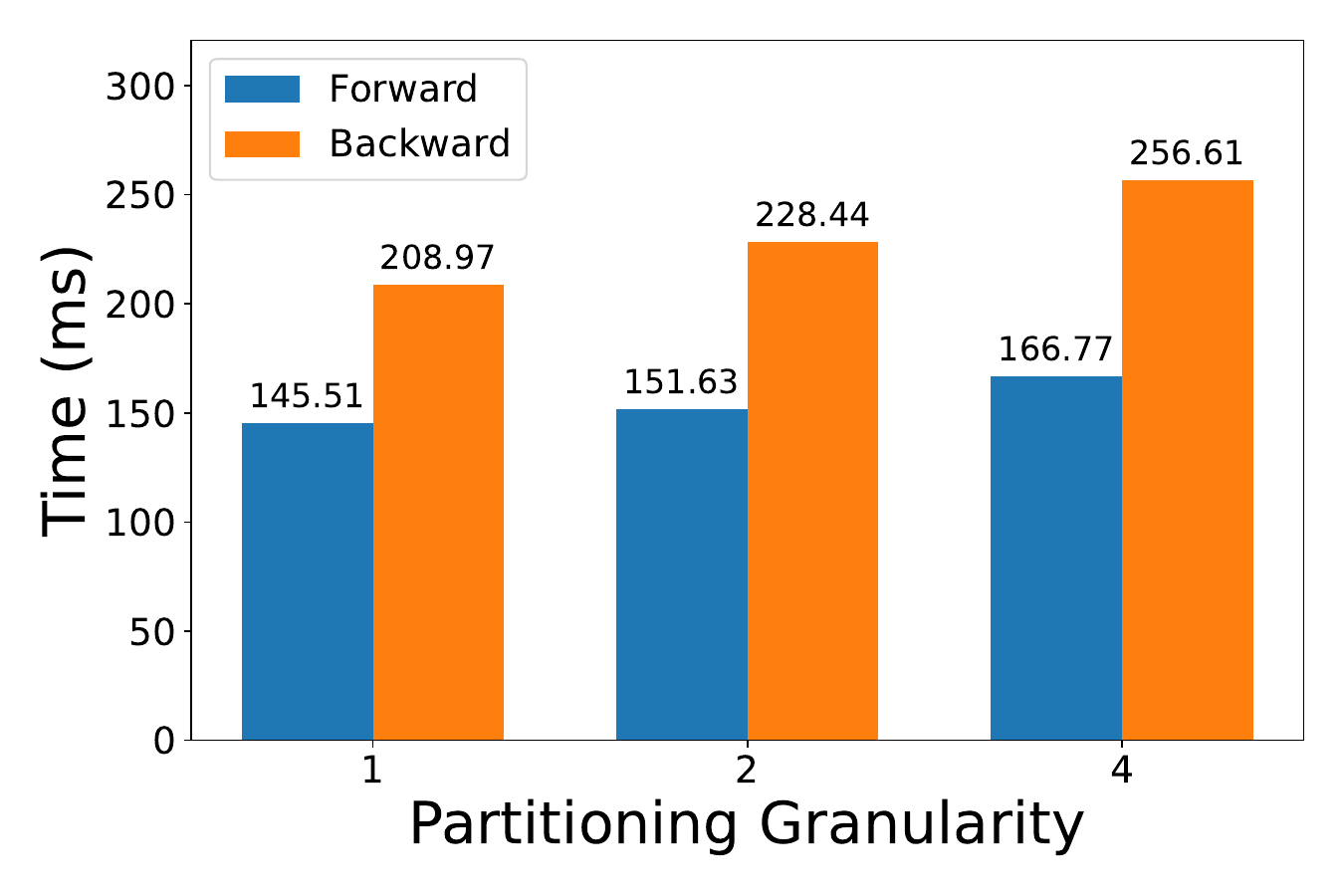}
      \captionof{figure}{Forward and backward latency under different granularities using Llama-3.1-8B. We measure the average latency per single forward and backward pass over the entire training phase.}
    \label{fig:gran_time}
    \end{minipage}
    \hfill
    \begin{minipage}[t]{0.48\textwidth}
      \centering
      \includegraphics[width=\textwidth]{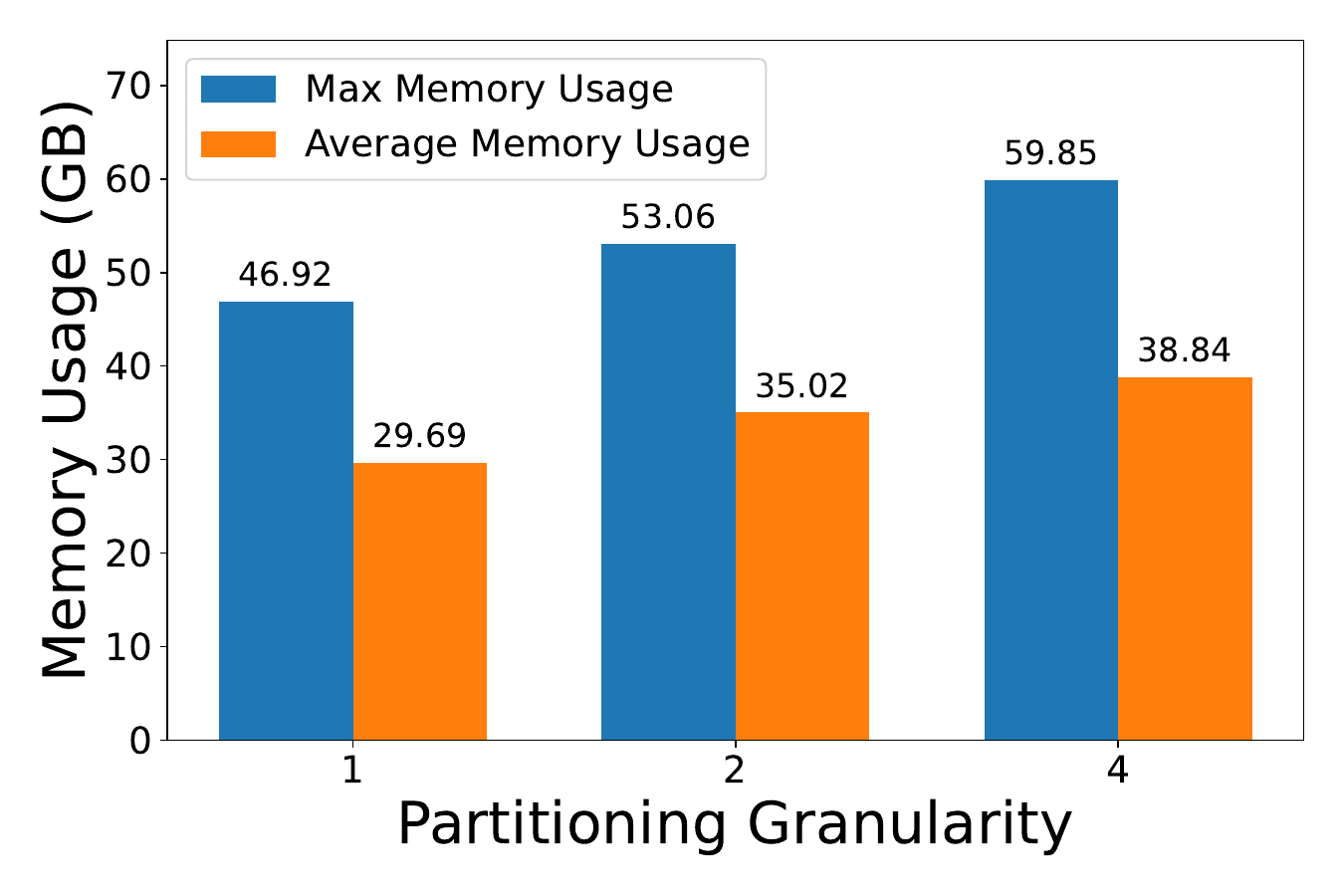}
      \captionof{figure}{Maximum and average GPU memory usage under different granularities. We measure the maximum and average memory usage during the whole fine-tuning process. }
    \label{fig:gran_gpu}
    \end{minipage}
    \end{figure*}

    % Interestingly, GPU utilization improves slightly with coarser granularity (rising from around $75\%$ to $81\%$), suggesting that larger blocks allow for better saturation of compute resources, albeit at the expense of higher memory footprint. This trade-off highlights that while coarser partitioning can slightly enhance computational throughput, finer partitioning is crucial for deploying FPFT on memory-constrained devices.    

\subsection{Hyper-parameter Sensitivity} \label{ssec:sensitivity}
  In this section, we conduct a sensitivity analysis of the hyperparameters in DECA. The Adam hyperparameters $\alpha_1$ and $\alpha_2$ control the exponential moving averages of the first- and second-order gradient moments, respectively, and are widely used with standard default settings~\cite{KingmaB-ICLR2015}. Since the same Adam hyperparameters are adopted for all baselines, we do not further tune them. Instead, we focus on the BMA-specific hyperparameters $\beta_1$ and $\beta_2$, which control how strongly the consensus-derived discrepancy signal is incorporated into the first- and second-order BMAs. As these parameters determine the balance between local gradient-driven optimization and consensus-guided correction, they may affect the stability and performance of DECA. Therefore, we evaluate DECA under different combinations of $\beta_1$ and $\beta_2$.

  \begin{table*}[t!]
  \centering
  \caption{Hyperparameter sensitivity analysis. We vary different BMA hyperparameters for generation tasks using the Qwen2-1.5B model.}
  \label{tab:gen_hyper}
  \begin{tabular}{c|c|cccc}
    \toprule
    $\beta_1$ & $\beta_2$ & \textbf{VIC.} & \textbf{MT-1} & \textbf{MT-2} & \textbf{MT.} \\
    \midrule
    \multirow{3}{*}{0.3} 
    & 0.222 & 5.66 &4.74 &3.25 &3.99 \\
    & 0.333 & 5.68 &4.88 & 3.30 &4.10 \\
    & 0.444 &5.63 &5.23 &3.39 &4.31 \\
    \midrule
    0.2 & \multirow{3}{*}{0.333} 
    & 5.70 &4.86 &3.31 &4.09 \\
    0.3 & & 5.68 &4.88 & 3.30 &4.10 \\
    0.4 & &5.75 &4.91 &3.26 &4.09 \\
    \bottomrule
  \end{tabular}
  \end{table*}

  As shown in Table~\ref{tab:gen_hyper}, the performance of DECA remains stable across all different configurations, demonstrating the robustness of DECA to BMA hyperparameters. First, when fixing $\beta_1=0.3$ and varying $\beta_2$ from $0.222$ to $0.444$, we observe a consistent improvement in performance. In particular, the MT. score increases from $3.99$ to $4.31$, indicating that assigning a larger weight to the local second-order momentum in BMA can effectively enhance generation quality. This suggests that a moderate increase in $\beta_2$ helps better exploit cross-client information without harming local adaptation. Second, when fixing $\beta_2=0.333$ and varying $\beta_1$ from $0.2$ to $0.4$, the performance exhibits only minor fluctuations (MT. ranges from $4.09$ to $4.10$). This indicates that DECA is relatively insensitive to $\beta_1$ within a reasonable range. Since $\beta_1$ controls the contribution of local updates, it implies that DECA maintains a good balance between local optimization and global aggregation even when $\beta_1$ varies.
  
  Combining these observations, we find that $\beta_2$ has a slightly more noticeable impact on performance compared to $\beta_1$, while the overall variation remains limited. The best performance is achieved at $(\beta_1, \beta_2) = (0.3, 0.444)$, suggesting that placing relatively more emphasis on global information is beneficial. Nevertheless, the performance differences across configurations are small, confirming that DECA is robust and does not require careful hyperparameter tuning in practice.

    % \begin{table*}[t!]
    % \centering
    % \caption{Performance comparison in generation tasks.}
    % \label{tab:gen_bsl}
    
    % \begin{small}
    % \begin{tabular}{lcccc}
    % \toprule
    
    % \multirow{1}{*}{\textbf{Methods}}
    % & \textbf{VIC.}
    % & \textbf{MT-1}
    % & \textbf{MT-2}
    % & \textbf{MT.} \\
    
    % \midrule
    
    % \rowcolor{lightgreen}
    % \multicolumn{5}{c}{\textit{Qwen2-1.5B}} \\
    % \midrule
    
    % Dec-Adapter
    % & 5.31 $\pm$ 0.06
    % & 4.23 $\pm$ 0.07
    % & 2.77 $\pm$ 0.16
    % & 3.50 $\pm$ 0.12 \\
    
    % Dec-LoRA
    % & 5.52 $\pm$ 0.05
    % & 4.43 $\pm$ 0.06
    % & 2.96 $\pm$ 0.09
    % & 3.70 $\pm$ 0.06 \\
    
    % DeCAF
    % & \textbf{5.69 $\pm$ 0.03}
    % & \underline{4.85 $\pm$ 0.04}
    % & \underline{3.31 $\pm$ 0.07}
    % & \underline{4.08 $\pm$ 0.02} \\
    
    % \textbf{DECA}
    % & \underline{5.63 $\pm$ 0.02}
    % & \textbf{4.91 $\pm$ 0.10}
    % & \textbf{3.32 $\pm$ 0.04}
    % & \textbf{4.12 $\pm$ 0.04} \\
    
    % \midrule
    
    % \rowcolor{lightgreen}
    % \multicolumn{5}{c}{\textit{Llama2-7B}} \\
    % \midrule
    
    % Dec-Adapter
    % & 4.93 $\pm$ 0.10
    % & 3.84 $\pm$ 0.05
    % & 2.57 $\pm$ 0.19
    % & 3.21 $\pm$ 0.05 \\
    
    % Dec-LoRA
    % & 5.58 $\pm$ 0.04
    % & 4.35 $\pm$ 0.06
    % & 2.82 $\pm$ 0.09
    % & 3.59 $\pm$ 0.01 \\
    
    % DeCAF
    % & \textbf{5.98 $\pm$ 0.09}
    % & \underline{4.38 $\pm$ 0.07}
    % & \underline{2.89 $\pm$ 0.03}
    % & \underline{3.63 $\pm$ 0.02} \\
    
    % \textbf{DECA}
    % & \underline{5.85 $\pm$ 0.02}
    % & \textbf{4.49 $\pm$ 0.05}
    % & \textbf{2.97 $\pm$ 0.06}
    % & \textbf{3.73 $\pm$ 0.00} \\
    
    % \bottomrule
    % \end{tabular}
    % \end{small}
    % \end{table*}

\section{Comprehensive Literature Survey} \label{sec:app_relwork}
  \subsection{PEFT and FPFT} 
    An effective strategy for LLM fine-tuning is to train a small number of model parameters, while keeping the majority of the pretrained parameters frozen~\cite{HoulsbyGJMDGAG-ICML19}. In particular, adapter tuning only finetunes the inserted small modules between layers called adapters~\cite{HoulsbyGJMDGAG-ICML19}. Prompt-tuning / prefix-tuning~\cite{LesterAC-EMNLP21,LiL-ARXIV221} attaches additional trainable prefix tokens to the input and/or hidden layers, while remaining the base model unchanged. Another method is the Low-Rank Adaptation (LoRA)~\cite{HuSSAWLWC-ICLR22}, which models the increment to the base model as a product of two significantly lower dimensional trainable low-rank matrices. Subsequent related research has aimed at extending its rank constraints~\cite{LialinSNA-ICLR24,XiaCH-arXiv24}, further reducing the number of trainable parameters~\cite{KoohpayeganiNNKP-arXiv23,KopiczkoBA-ICLR24}, decreasing memory usage through quantization~\cite{DettmersPHZ-NIPS23}. Presently, LoRA-based methods are commonly employed for LLM fine-tuning with limited memory resources.

    To conduct full parameter fine-tuning of LLMs with limited memory, \cite{LvYLGQ-ACL24} proposes LOMO, which efficiently leverages the BP process to update parameters on the fly in the process of computing stochastic gradients. Consequently, LOMO helps to execute SGD for full parameter finetuning without physically storing the stochastic gradients, significantly reducing memory consumption. However, it is worth emphasizing that SGD generally converges more slowly and is often considered suboptimal compared to Adam. MeZO~\cite{MalladiGNDLCA-NIPS23} is to approximate SGD by using only the forward pass. The idea of MeZO derives from zeroth-order optimization, which utilizes function value difference to approximate the stochastic gradients of the trainable model parameters. Galore~\cite{ZhaoZCWAT-ICML24} uses gradient low-rank projection, which largely reduces memory consumption for full parameter finetuning compared to Adam. Adam-mini~\cite{ZhangCLDWKYLS-ICLR25} proposes to apply block-wise adaptive learning rate, which reduces the memory for storing the full second moment. Another popular approach for finetuning with limited memory is to perform CPU offloads to reduce the memory consumption caused by training data and optimizers; see, e.g., \cite{RenRARYZLH-ATC21,LiuZLLFWZS-EMNLP24}. 
    %
    %
    %\cite{LuoYLi-NIPS24} adopts the well-established block coordinate descent (BCD) paradigm~\cite{Tseng-JOTA01,BengioLPL-NIPS06,Wright-MP15,BelilovskyEO-ICML19} to develop a block-wise Adam optimizer for full-parameter LLM fine-tuning. \textcolor{red}{However, this approach assumes a centralized training regime, requiring all data to be aggregated within a central infrastructure, and its theoretical convergence performance is guaranteed only in a deterministic setting. Enabling this paradigm to a decentralized setting is highly non-trivial, as maintaining convergence guarantees becomes challenging when data is non-IID across clients. In such decentralized environments, the mismatch between local objectives and the global goal often leads to significant model drift, making it unclear whether the theoretical benefits of BCD can be preserved.}
    %
    %
    %
    \cite{LuoYLi-NIPS24} adopts the well-established block coordinate descent (BCD) paradigm~\cite{Tseng-JOTA01,BengioLPL-NIPS06,Wright-MP15,BelilovskyEO-ICML19} to develop a block-wise Adam optimizer for full-parameter LLM fine-tuning. However, this approach is designed for centralized training, where the data are available within a single training infrastructure, and its convergence analysis is established only for the deterministic case. Extending this paradigm to decentralized learning is highly non-trivial: under non-IID data, local block-wise updates may be biased toward heterogeneous client objectives, causing client drift and disrupting consensus. As a result, it remains unclear whether the convergence benefits of BCD can be preserved without central coordination.

  \subsection{Federated LLM Fine-Tuning}
    FL has emerged as a pivotal paradigm for training LLMs while preserving data privacy. However, the massive parameter count of modern LLMs introduces prohibitive communication and computational overheads for standard FL protocols. While recent studies suggest that fine-tuning pre-trained language models offers inherent robustness against non-IID data distributions~\cite{QuZLXWAR-CVPR22,ChenTLSC-ICLR23,NguyenWMSR-ICLR23,WellerMBLD-NAACL22}, the resource constraints of edge devices render full-parameter updates often impractical. Consequently, PEFT has become the de facto standard in this domain~\cite{HuSSAWLWC-ICLR22,WeiSHY-ACL25,LesterAC-EMNLP21}. Early works utilized adapter-based~\cite{HoulsbyGJMDGAG-ICML19,ZhangQX-ICML24} or prompt-tuning methods, such as FedPrompt~\cite{ZhaoDLLL-ICASSP23}, which optimizes soft prompts to reduce transmission costs. Nevertheless, these methods often suffer from accuracy degradation compared to full fine-tuning, particularly under heterogeneous data distributions~\cite{PuJYK-arXiv23}.

    Among PEFT techniques, Low-Rank Adaptation (LoRA) has gained prominence due to the hypothesis that over-parameterized models reside on a low intrinsic dimension~\cite{LiFLY-ICLR18,AghajanyanGZ-IJCNLP21}. LoRA allows for the efficient merging of trainable matrices with frozen weights, introducing no inference latency. Recent research has focused on optimizing LoRA for federated settings. FDLoRA~\cite{QiLHFYQ-arXiv24} introduces a dual-LoRA mechanism to decouple global knowledge sharing from local personalization. Similarly, FedSA-LoRA~\cite{GuoZWFQ-ICLR25} leverages the asymmetric roles of low-rank matrices, selectively aggregating only the shared $A$ matrices to capture global knowledge while maintaining robustness against data heterogeneity. To accommodate devices with varying capacities, ``Heterogeneous LoRA'' strategies have been proposed. For instance, \cite{ChoLXFJ-EMNLP24} manages diverse rank allocations via zero-padding and truncation, while RBLA~\cite{ChenTNZ-ICWS24} and stacking-based methods~\cite{WangSHSWLL-NIPS24} perform weighted aggregation of heterogeneous structures. FFA-LoRA~\cite{SunLLD-ICLR24} further reduces overhead by freezing the randomly initialized $A$ matrices and fine-tuning only the zero-initialized $B$ matrices. However, fixing parameters can impair the model's learning capacity, potentially leading to suboptimal performance~\cite{ZhangZSCL-arXiv23}.

    Despite the efficiency of PEFT, the performance gap between partial and full fine-tuning remains a critical bottleneck. Addressing this, recent works have begun to explore communication-efficient FPFT frameworks. FedKSeed~\cite{QinCQDLD-ICML24} utilizes zeroth-order optimization to enable billion-scale training by exchanging only random seeds and scalar gradients. Similarly, Ferret~\cite{ShuHNLY-ICML25} leverages low-dimensional projection and shared randomness to substantially mitigate the communication overhead of full-parameter updates. These emerging methods signal a shift toward enabling high-performance FPFT within resource-constrained federated environments.

  \subsection{Decentralized Learning/Fine-Tuning}
    Recent years have witnessed significant advancements in the paradigm of decentralized learning. Foundational works such as \cite{LianZZHZL-NIPS17} combined \textit{Stochastic Gradient Descent} (SGD) with gossip-averaging~\cite{BoydGPS-TIT06} to develop \textit{Decentralized Parallel Stochastic Gradient Descent} (DPSGD). Subsequent studies addressed convex optimization with greater theoretical rigor: \cite{ScamanBBLM-ICML17} proposed the \textit{Multi-Step Dual Accelerated} (MSDA) algorithm with a provable linear convergence rate for strongly convex smooth functions, while \cite{ScamanBBLM-NIPS18} introduced the \textit{Multi-Step Primal-Dual} (MSPD) algorithm for the more challenging non-smooth case. Furthermore, inspired by the success of momentum in accelerating convergence and improving generalization, \cite{YuJY-ICML19} proposed \textit{Decentralized Momentum Stochastic Gradient Descent} (DMSGD), extending momentum methods to decentralized topologies.

    While the aforementioned studies typically assume that training data are independent and identically distributed across clients, this assumption rarely holds in real-world applications. Consequently, significant efforts have been dedicated to addressing data heterogeneity. One line of research focuses on robust optimization techniques, such as \cite{ShiSWSYWT-ICML23}, which employs Sharpness Aware Minimization (SAM) to search for flat minima with uniformly low loss values across clients. Others utilize knowledge transfer, such as \cite{LiLV-IOTJ21}, which proposes a mutual knowledge transfer mechanism where subsets of clients exchange local models to mitigate distribution shifts.

    To explicitly correct for drift caused by heterogeneity, various tracking mechanisms have been introduced. Gradient tracking~\cite{PuN-MP21} and momentum tracking~\cite{TakezawaBNSY-TMLR23} are widely adopted to estimate global gradients locally. Diverging from traditional average-gradient tracking, \cite{AketiHR-NeurIPS24} proposes tracking model updates by having clients store copies of their neighbors' parameters. Similarly, NET-FLEET~\cite{ZhangFLYLZ-MobiHoc22} incorporates recursive gradient correction to approximate global stochastic gradients efficiently. More complex aggregation strategies include the \textit{Cross Gradient Aggregation} (CGA) algorithm~\cite{EsfandiariTJBHHS-ICML21}, which uses quadratic programming (QP) to aggregate cross-gradients (derivatives of a local model on neighbors' data). 
    %
    %This concept was further refined in \cite{AketiKR-TMLR23}, which considers bidirectional cross-gradients---derivatives of local data on neighbors' models and vice versa.
    %
    %\textcolor{red}{Most relevant to our work, \cite{LinKSJ-ICML21} introduced Quasi-Global (QG) momentum to accelerate decentralized DNN training on non-IID data by locally approximating the global optimization direction. We build on and extend this idea to block-wise Adam optimization: instead of applying QG momentum only to first-order updates, we design both first- and second-order \emph{block-wise moment approximations} (BMAs) to smooth the block-wise optimization trajectory and stabilize training under non-IID data.}
    %
    \cite{LinKSJ-ICML21} introduces Quasi-Global (QG) momentum for decentralized SGD with momentum, to approximate the global optimization direction at each client. Inspired by this idea, we construct a block-wise discrepancy signal from the gap between the local block parameters and their neighbor-averaged counterparts. Different from QG momentum, DECA incorporates this signal into both the first- and second-order moment estimates of block-wise Adam optimization. This design enables each client to preserve objective-aligned local gradient information while adaptively correcting its update direction according to decentralized consensus, thereby mitigating client drift and stabilizing decentralized block-wise optimization under non-IID data.

    While prior studies primarily focus on decentralized learning for small-scale DNNs, only a few have addressed the unique challenges of decentralized LLM training. \cite{GhiasvandAP-REALM25} proposes Dec-LoRA, a decentralized fine-tuning framework based on LoRA; however, its performance is inherently limited by updating only a small subset of model parameters. \cite{Saadati-NeuralNet26} further combines decentralized LoRA with truncated singular value decomposition (TSVD)-based matrix factorization to mitigate consensus interference, but this factorization introduces substantial overhead. \emph{To the best of our knowledge, no existing work has established a resource-efficient and high-performance decentralized framework for FPFT, leaving decentralized FPFT an unexplored yet promising direction for achieving stronger task-specific adaptation.}

\end{document}